\documentclass[]{bytedance_seed}
\pdfoutput=1

\usepackage{minitoc}
\usepackage{cleveref} 
\usepackage{subcaption}
\usepackage{booktabs}
\usepackage{tabularx}
\usepackage{microtype}
\usepackage{subcaption}
\usepackage{hyperref}
\usepackage{url}
\usepackage{booktabs}
\usepackage{colortbl}
\usepackage{multirow}
\usepackage{enumitem}
\usepackage{lineno}
\usepackage{wrapfig}
\usepackage{algorithm}
\usepackage{xspace}
\usepackage{tcolorbox}
\tcbuselibrary{breakable,listingsutf8,skins}
\usepackage{fvextra}
\usepackage[noEnd,indLines]{algpseudocodex}
\usepackage{tikz}
\usetikzlibrary{arrows.meta}

\usepackage{amsmath,amsfonts,amssymb,amsthm,bm}
\usepackage{accents}
\usepackage{hyperref}

\theoremstyle{plain}
\newtheorem{theorem}{Theorem}[section]
\newtheorem{proposition}[theorem]{Proposition}

\newtheorem{corollary}[theorem]{Corollary}

\newtheorem{remark}[theorem]{Remark}

\numberwithin{equation}{section}

\crefname{equation}{Eqn.}{Eqns.}
\crefname{lemma}{Lem.}{Lems.}
\crefname{section}{Sec.}{Secs.}
\crefname{appendix}{App.}{Apps.}
\crefname{table}{Tab.}{Tabs.}
\crefname{theorem}{Thm.}{Thms.}
\crefname{proposition}{Prop.}{Props.}
\crefname{assumption}{Assump.}{Assumps.}
\crefname{corollary}{Cor.}{Cors.}
\crefname{remark}{Rmk.}{Rmks.}
\crefname{definition}{Def.}{Defs.}
\crefname{algocf}{Alg.}{Algs.}
\crefname{ansatz}{Ansatz}{Ans\"atze}
\crefname{manualassumption}{Assump.}{Assumps.}

\makeatletter
\AddToHook{cmd/appendix/before}{\def\cref@section@alias{appendix}\def\cref@subsection@alias{appendix}}
\makeatother

\renewcommand{\tilde}[1]{\widetilde{#1}}
\renewcommand{\hat}[1]{\widehat{#1}}
\renewcommand{\bar}[1]{\overline{#1}}

\newcommand{\dif}{{\mathrm{d}}}
\def\eps{{\epsilon}}

\def\ve{{\bm{e}}}

\def\vu{{\bm{u}}}

\def\vx{{\bm{x}}}
\def\vy{{\bm{y}}}

\def\mR{{\bm{R}}}

\def\mT{{\bm{T}}}

\DeclareMathAlphabet{\mathsfit}{\encodingdefault}{\sfdefault}{m}{sl}
\SetMathAlphabet{\mathsfit}{bold}{\encodingdefault}{\sfdefault}{bx}{n}

\def\gG{{\mathcal{G}}}

\def\gJ{{\mathcal{J}}}

\def\gL{{\mathcal{L}}}
\def\gM{{\mathcal{M}}}

\def\gS{{\mathcal{S}}}

\def\gV{{\mathcal{V}}}

\def\gY{{\mathcal{Y}}}

\newcommand{\E}{\mathbb{E}}
\renewcommand{\P}{\mathbb{P}}

\newcommand{\softmax}{\mathrm{softmax}}

\DeclareMathOperator*{\argmax}{arg\,max}

\definecolor{darkblue}{rgb}{0, 0, 0.5}
\hypersetup{colorlinks=true, citecolor=darkblue, linkcolor=darkblue, urlcolor=darkblue}

\title{Multi-Mask Diffusion Language Models\\ for Few-Step Generation}

\author[1,2,*,\dagger,\ddagger]{Sijin Chen}
\author[1,3,*,\dagger,\ddagger]{Yinuo Ren}
\author[1,4,\dagger]{Heyang Zhao}
\author[1,5,\dagger]{Ziheng Cheng}
\author[4]{Quanquan Gu}
\author[1,\ddagger]{\mbox{Lexing Ying}}

\affiliation[1]{ByteDance Seed}
\affiliation[2]{Princeton University}
\affiliation[3]{Stanford University}
\affiliation[4]{UCLA}
\affiliation[5]{UC Berkeley}

\authornote{$^*$Equal contribution, listed alphabetically. $^\dagger$Work done during an internship at ByteDance Seed. $^\ddagger$Corresponding authors.}
\correspondence{Lexing Ying \textlangle\email{lexing.ying@bytedance.com}\textrangle, Sijin Chen \textlangle\email{chensj@princeton.edu}\textrangle, Yinuo Ren
\textlangle\email{yinuoren@stanford.edu}\textrangle}

\abstract{%
Masked diffusion models (MDMs) are a promising family of language generators, but achieving high-quality few-step generation remains challenging. In MDMs, all forward trajectories collapse to a single fully masked state, leaving no terminal entropy for consistency-style few-step generation. While recent few-step alternatives based on uniform-state diffusion avoid this degeneracy, it becomes harder to distinguish clean tokens from noise than MDMs, which usually harms modeling quality and training efficiency. In this work, we propose a multi-mask diffusion model (\textbf{MultiMDM}) that preserves the masking structure towards few-step generation. In the forward process, each clean token is first pushed towards a designated mask and then gradually mixes over the mask set. As a result, the backward process has a \textit{drafting} capability by predicting a designated mask before refining to a clean token. We derive a closed-form ELBO training objective for \textbf{MultiMDM} that supports continual training from pretrained MDMs. In addition, we formulate a purely discrete-state consistency distillation scheme, with a shared-Gumbel coupling to reduce pathwise entropy. Experiments on pretraining and distillation show that \textbf{MultiMDM} provides an effective foundation for principled few-step generation.%
}

\newcommand{\mask}{\mathtt{[M]}}
\newcommand{\1}{\mathbf 1}
\newcommand{\uV}{u_{\gV}}

\newcommand{\uM}{u_{\gM}}
\newcommand{\vuM}{\vu_{\gM}}
\newcommand{\ours}{\textbf{MultiMDM}\xspace}
\newcommand{\multillada}{\textbf{MultiLLaDA}\xspace}
\newcommand{\corpusA}{Corpus A\xspace}
\newcommand{\corpusB}{Corpus B\xspace}

\definecolor{CleanColBase}{HTML}{6FE8DF}
\definecolor{CorrColBase}{HTML}{0095FD}
\definecolor{OddJumpColBase}{HTML}{00CBD4}
\colorlet{CleanCol}{CleanColBase!60}
\colorlet{CorrCol}{CorrColBase!60}
\colorlet{OddJumpCol}{OddJumpColBase!80}
\colorlet{BestCellCol}{OddJumpColBase!50}
\newcommand{\bestcell}[1]{\cellcolor{BestCellCol}\textbf{#1}}
\tcbset{
  blockframe/.style={
    enhanced jigsaw,
    breakable,
    boxrule=0pt,
    borderline={.6pt}{0pt}{black},
    arc=8pt,
    outer arc=8pt,
    left=3pt,
    right=3pt,
    top=3pt,
    bottom=3pt,
    before skip=2pt,
    after skip=2pt
  },
  theoremshade/.style={
    blockframe,
    colback=CleanColBase!15
  },
  samplelisting/.style={
    blockframe,
    colback=black!3
  }
}
\newcommand{\wraptheoremenv}[1]{%
  \BeforeBeginEnvironment{#1}{\begin{tcolorbox}[theoremshade]}%
  \AfterEndEnvironment{#1}{\end{tcolorbox}}%
}
\wraptheoremenv{theorem}
\wraptheoremenv{proposition}
\wraptheoremenv{lemma}
\wraptheoremenv{corollary}
\wraptheoremenv{definition}
\wraptheoremenv{assumption}
\wraptheoremenv{ansatz}
\wraptheoremenv{manualassumption}

\ExplSyntaxOn
\tl_new:N \l__wrappedlisting_part_tl
\cs_generate_variant:Nn \tl_replace_all:Nnn { NVn }
\cs_new:Npn \__wrappedlisting_esc:n #1 { \c_backslash_str \tl_to_str:n {#1} }
\tl_const:Nx \c__wrappedlisting_needle_esc_nn_tl   { \__wrappedlisting_esc:n { n } \__wrappedlisting_esc:n { n } }
\tl_const:Nx \c__wrappedlisting_needle_esc_n_tl    { \__wrappedlisting_esc:n { n } }
\tl_const:Nx \c__wrappedlisting_needle_esc_dq_tl   { \__wrappedlisting_esc:n { " } }
\tl_const:Nx \c__wrappedlisting_needle_esc_uI_tl   { \__wrappedlisting_esc:n { u2018 } }
\tl_const:Nx \c__wrappedlisting_needle_esc_uJ_tl   { \__wrappedlisting_esc:n { u2019 } }
\tl_const:Nx \c__wrappedlisting_needle_esc_uK_tl   { \__wrappedlisting_esc:n { u201c } }
\tl_const:Nx \c__wrappedlisting_needle_esc_uKhi_tl { \__wrappedlisting_esc:n { u201C } }
\tl_const:Nx \c__wrappedlisting_needle_esc_uL_tl   { \__wrappedlisting_esc:n { u201d } }
\tl_const:Nx \c__wrappedlisting_needle_esc_uLhi_tl { \__wrappedlisting_esc:n { u201D } }
\tl_const:Nx \c__wrappedlisting_needle_esc_uM_tl   { \__wrappedlisting_esc:n { u2013 } }
\tl_const:Nx \c__wrappedlisting_needle_esc_uN_tl   { \__wrappedlisting_esc:n { u2014 } }
\tl_const:Nx \c__wrappedlisting_needle_esc_uO_tl   { \__wrappedlisting_esc:n { u223c } }
\tl_const:Nx \c__wrappedlisting_needle_esc_uP_tl   { \__wrappedlisting_esc:n { u00f6 } }
\tl_const:Nx \c__wrappedlisting_needle_esc_uQ_tl   { \__wrappedlisting_esc:n { ufffd } }
\cs_new_protected:Npn \__wrappedlisting_textrm_ql:
  { \group_begin: \rmfamily \textquoteleft \group_end: }
\cs_new_protected:Npn \__wrappedlisting_textrm_qr:
  { \group_begin: \rmfamily \textquoteright \group_end: }
\cs_new_protected:Npn \__wrappedlisting_textrm_dql:
  { \group_begin: \rmfamily \textquotedblleft \group_end: }
\cs_new_protected:Npn \__wrappedlisting_textrm_dqr:
  { \group_begin: \rmfamily \textquotedblright \group_end: }
\cs_new_protected:Npn \__wrappedlisting_textrm_endash:
  { \group_begin: \rmfamily \textendash \group_end: }
\cs_new_protected:Npn \__wrappedlisting_textrm_emdash:
  { \group_begin: \rmfamily \textemdash \group_end: }
\newenvironment{wrappedlisting}
  {\VerbatimEnvironment\begin{SaveVerbatim}{wrappedlistingbuffer}}
  {\end{SaveVerbatim}\wrappedlisting_render_buffer:}
\cs_new_protected:Npn \wrappedlisting_render_buffer:
  {
    \begin{tcolorbox}[samplelisting]
    \ttfamily\scriptsize\raggedright\sloppy\setlength{\parindent}{0pt}
    \cs_set_protected:cpn { FV@ProcessLine } ##1
      {
        \tl_set:Nx \l__wrappedlisting_part_tl { \tl_to_str:n { ##1 } }
        \tl_replace_all:NVn \l__wrappedlisting_part_tl \c__wrappedlisting_needle_esc_nn_tl { \newline }
        \tl_replace_all:NVn \l__wrappedlisting_part_tl \c__wrappedlisting_needle_esc_n_tl { \newline }
        \tl_replace_all:NVn \l__wrappedlisting_part_tl \c__wrappedlisting_needle_esc_dq_tl { " }
        \tl_replace_all:NVn \l__wrappedlisting_part_tl \c__wrappedlisting_needle_esc_uI_tl { \__wrappedlisting_textrm_ql: }
        \tl_replace_all:NVn \l__wrappedlisting_part_tl \c__wrappedlisting_needle_esc_uJ_tl { \__wrappedlisting_textrm_qr: }
        \tl_replace_all:NVn \l__wrappedlisting_part_tl \c__wrappedlisting_needle_esc_uK_tl { \__wrappedlisting_textrm_dql: }
        \tl_replace_all:NVn \l__wrappedlisting_part_tl \c__wrappedlisting_needle_esc_uKhi_tl { \__wrappedlisting_textrm_dql: }
        \tl_replace_all:NVn \l__wrappedlisting_part_tl \c__wrappedlisting_needle_esc_uL_tl { \__wrappedlisting_textrm_dqr: }
        \tl_replace_all:NVn \l__wrappedlisting_part_tl \c__wrappedlisting_needle_esc_uLhi_tl { \__wrappedlisting_textrm_dqr: }
        \tl_replace_all:NVn \l__wrappedlisting_part_tl \c__wrappedlisting_needle_esc_uM_tl { \__wrappedlisting_textrm_endash: }
        \tl_replace_all:NVn \l__wrappedlisting_part_tl \c__wrappedlisting_needle_esc_uN_tl { \__wrappedlisting_textrm_emdash: }
        \tl_replace_all:NVn \l__wrappedlisting_part_tl \c__wrappedlisting_needle_esc_uO_tl { \ensuremath{\sim} }
        \tl_replace_all:NVn \l__wrappedlisting_part_tl \c__wrappedlisting_needle_esc_uP_tl { {\"o} }
        \tl_replace_all:NVn \l__wrappedlisting_part_tl \c__wrappedlisting_needle_esc_uQ_tl { ? }
        \tl_replace_all:Nnn \l__wrappedlisting_part_tl { <|endoftext|> } { <|endoftext|>\allowbreak }
        \tl_replace_all:Nnn \l__wrappedlisting_part_tl { [CLS] } { [CLS]\allowbreak }
        \noindent\tl_use:N \l__wrappedlisting_part_tl \par
      }
    \use:c { FV@SV@wrappedlistingbuffer }
    \end{tcolorbox}
  }
\ExplSyntaxOff

\makeatletter
\@ifundefined{corollary}{\newtheorem{corollary}{Corollary}}{}
\makeatother
\crefname{algorithm}{Alg.}{Algs.}
\newcommand{\AlgRequire}[1]{%
    \Statex \hspace*{-\dimexpr\labelwidth+\labelsep\relax}%
    \parbox[t]{\dimexpr\linewidth+\labelwidth+\labelsep\relax}{\textbf{Require:} #1}%
}

\begin{document}

\maketitle

\begin{figure*}[h]
\centering
\DeclareRobustCommand{\inlineTokBox}[2]{%
    \tikz[baseline=(tok.base)]{
        \node[
            draw=black,
            rounded corners=3pt,
            fill=#1,
            inner xsep=2pt,
            inner ysep=1pt,
            font=\footnotesize,
            text height=1.6ex,
            text depth=0.4ex
        ] (tok) {#2};
    }%
}
\newcommand{\trajPanelHeight}{3.25cm}
\newcommand{\trajSubfigWidth}{0.32\textwidth}

\newcommand{\trajpanel}[2]{%
    \resizebox{!}{\trajPanelHeight}{%
    \begin{tikzpicture}[
        x=1cm, y=1cm,
        tok/.style={
            draw=black,
            rounded corners=5pt,
            minimum width=1.42cm,
            minimum height=0.98cm,
            inner xsep=3pt,
            inner ysep=2pt,
            font=\Large,
            text height=2.1ex,
            text depth=0.7ex
        },
        tlabel/.style={font=\Large},
        tick/.style={line width=0.5pt}
    ]
    \def\dx{1.74}
    \def\dy{1.28}
    \path[use as bounding box] (-1.85,-0.70) rectangle (5*\dx+0.95,5*\dy+0.70);
    #1
    \foreach \rowdata [count=\ry from 0] in {#2} {%
        \foreach \cellfill/\celltext [count=\cx from 0] in \rowdata {%
            \node[tok, fill=\cellfill] at (\cx*\dx, \ry*\dy) {\celltext};%
        }%
    }%
    \end{tikzpicture}%
    }%
}

\begin{subfigure}[t]{\trajSubfigWidth}
\centering
\trajpanel{%
        \draw[-{Latex[length=4.2mm,width=3.2mm]}, line width=0.95pt] (-1.22, 5*\dy) -- (-1.22, 0*\dy);
        \node[tlabel] at (-1.49, 2.5*\dy) {$t$};
        \foreach \ty in {1,2,3,4,5} { \draw[tick] (-1.32, \ty*\dy) -- (-1.12, \ty*\dy); }
        \foreach \ty/\tv in {0/1, 1/0.8, 2/0.6, 3/0.4, 4/0.2, 5/0} {
            \node[tlabel, anchor=east] at (-1.37, \ty*\dy) {$\tv$};
        }
    }{%
        {CorrCol/life, CorrCol/man, CorrCol/judge, CorrCol/rich, CorrCol/day, CorrCol/wine},
        {CleanCol/it, CorrCol/love, CleanCol/the, CorrCol/london, CorrCol/storm, CorrCol/day},
        {CleanCol/it, CorrCol/man, CleanCol/the, CorrCol/rage, CorrCol/law, CorrCol/queen},
        {CleanCol/it, CorrCol/child, CleanCol/the, CleanCol/best, CorrCol/judge, CorrCol/day},
        {CleanCol/it, CleanCol/was, CleanCol/the, CleanCol/best, CleanCol/of, CorrCol/man},
        {CleanCol/it, CleanCol/was, CleanCol/the, CleanCol/best, CleanCol/of, CleanCol/times},
    }
\caption{Uniform-state diffusion.}
\end{subfigure}
\begin{subfigure}[t]{\trajSubfigWidth}
\centering
\trajpanel{}{%
        {CorrCol/\texttt{[M]}, CorrCol/\texttt{[M]}, CorrCol/\texttt{[M]}, CorrCol/\texttt{[M]}, CorrCol/\texttt{[M]}, CorrCol/\texttt{[M]}},
        {CorrCol/\texttt{[M]}, CleanCol/was, CorrCol/\texttt{[M]}, CorrCol/\texttt{[M]}, CorrCol/\texttt{[M]}, CorrCol/\texttt{[M]}},
        {CorrCol/\texttt{[M]}, CleanCol/was, CorrCol/\texttt{[M]}, CorrCol/\texttt{[M]}, CorrCol/\texttt{[M]}, CorrCol/\texttt{[M]}},
        {CorrCol/\texttt{[M]}, CleanCol/was, CorrCol/\texttt{[M]}, CorrCol/\texttt{[M]}, CleanCol/of, CorrCol/\texttt{[M]}},
        {CleanCol/it, CleanCol/was, CleanCol/the, CleanCol/best, CleanCol/of, CorrCol/\texttt{[M]}},
        {CleanCol/it, CleanCol/was, CleanCol/the, CleanCol/best, CleanCol/of, CleanCol/times},
    }
\caption{Masked diffusion.}
\end{subfigure}
\begin{subfigure}[t]{\trajSubfigWidth}
\centering
\trajpanel{}{%
        {CorrCol/$\texttt{[M}_{\texttt{2}}\texttt{]}$, CorrCol/$\texttt{[M}_{\texttt{3}}\texttt{]}$, CorrCol/$\texttt{[M}_{\texttt{1}}\texttt{]}$, CorrCol/$\texttt{[M}_{\texttt{3}}\texttt{]}$, CorrCol/$\texttt{[M}_{\texttt{3}}\texttt{]}$, CorrCol/$\texttt{[M}_{\texttt{2}}\texttt{]}$},
        {OddJumpCol/$\texttt{[M}_{\texttt{1}}\texttt{]}$, CorrCol/$\texttt{[M}_{\texttt{2}}\texttt{]}$, OddJumpCol/$\texttt{[M}_{\texttt{2}}\texttt{]}$, OddJumpCol/$\texttt{[M}_{\texttt{1}}\texttt{]}$, OddJumpCol/$\texttt{[M}_{\texttt{2}}\texttt{]}$, CorrCol/$\texttt{[M}_{\texttt{3}}\texttt{]}$},
        {OddJumpCol/$\texttt{[M}_{\texttt{1}}\texttt{]}$, OddJumpCol/$\texttt{[M}_{\texttt{3}}\texttt{]}$, CleanCol/the, OddJumpCol/$\texttt{[M}_{\texttt{1}}\texttt{]}$, OddJumpCol/$\texttt{[M}_{\texttt{2}}\texttt{]}$, CorrCol/$\texttt{[M}_{\texttt{1}}\texttt{]}$},
        {CleanCol/it, OddJumpCol/$\texttt{[M}_{\texttt{3}}\texttt{]}$, CleanCol/the, CleanCol/best, OddJumpCol/$\texttt{[M}_{\texttt{2}}\texttt{]}$, OddJumpCol/$\texttt{[M}_{\texttt{3}}\texttt{]}$},
        {CleanCol/it, CleanCol/was, CleanCol/the, CleanCol/best, CleanCol/of, OddJumpCol/$\texttt{[M}_{\texttt{3}}\texttt{]}$},
        {CleanCol/it, CleanCol/was, CleanCol/the, CleanCol/best, CleanCol/of, CleanCol/times},
    }
\caption{\textbf{\ours (Ours)}.}
\label{fig:multimdm-schematic}
\end{subfigure}
\vskip-.3em
\caption[Schematic illustration of three discrete diffusion variants.]{
Schematic illustration of three discrete diffusion variants. From \inlineTokBox{CleanCol}{clean tokens} to \inlineTokBox{CorrCol}{corrupted tokens}, \ours's trajectory traverses designated \inlineTokBox{OddJumpCol}{intermediate states} to encode more information for generation and consistency distillation. For this example, the designated mapping is \{\inlineTokBox{CleanCol}{it}, \inlineTokBox{CleanCol}{best}\}$\mapsto$\inlineTokBox{OddJumpCol}{$\texttt{[M}_{\texttt{1}}\texttt{]}$}, \{\inlineTokBox{CleanCol}{the}, \inlineTokBox{CleanCol}{of}\}$\mapsto$\inlineTokBox{OddJumpCol}{$\texttt{[M}_{\texttt{2}}\texttt{]}$}, \{\inlineTokBox{CleanCol}{was}, \inlineTokBox{CleanCol}{times}\}$\mapsto$\inlineTokBox{OddJumpCol}{$\texttt{[M}_{\texttt{3}}\texttt{]}$}.
}
\label{fig:forward_trajectories_three_panel}
\end{figure*}
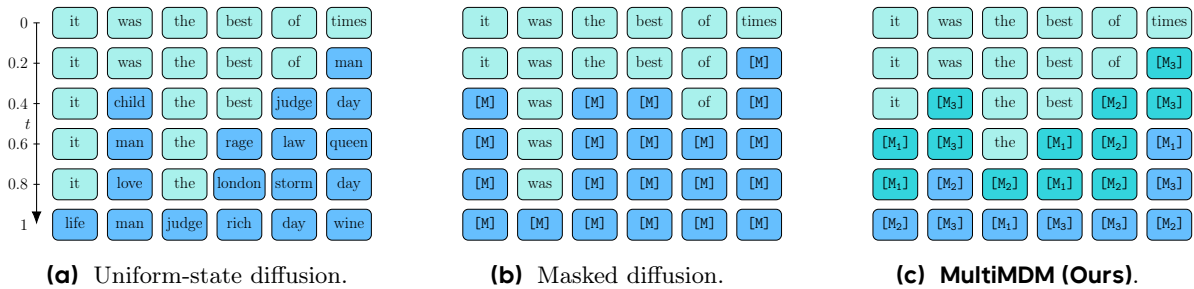

\section{Introduction}

Flow-~\citep{lipman2023flow,liu2023flow,albergo2025stochastic} and diffusion-~\citep{sohl2015deep,ho2020denoising,song2019generative,song2021scorebased} based generative models learn data distributions by transporting simple reference randomness to structured samples along a time-indexed path. 
The appeal comes from their principled training objectives, modular model design, and superior generation quality across image~\citep{rombach2022high}, video~\citep{ho2022video}, and scientific~\citep{xu2022geodiff,watson2023novo} domains.

In discrete domains including text and biological sequences, these ideas are usually instantiated with continuous-time Markov chains on a finite state space~\citep{austin2021structured,hoogeboom2021argmax,chang2022maskgit,lou2023discrete,gat2024discrete,shi2024simplified,sahoo2024simple,arriola2025block,ou2025your,zheng2025masked}. 
Another variant, masked diffusion model (MDM), instead introduces an absorbing state and has emerged as a particularly effective variant due to the simplicity and strong alignment of the reconstruction target with language modeling~\citep{gong2025scaling,nie2025large,song2025seed,ye2025dream}. 
Recent work has further improved this paradigm~\citep{pynadath2025candi,kim2025any,chao2025beyond,havasi2025edit,hersche2026softmasked,ding2026beyond,hoogeboom2026beyond,zhou2026next,chao2026mdm}.

Motivated by the advances in the acceleration and consistency distillation of continuous diffusion~\citep{song2021denoising,lu2022dpm,salimans2022progressive,song2023consistency,sauer2024adversarial,song2024improved,boffi2025how}, interests grow in the few-step generation of discrete diffusion models~\citep{ren2025fast,sahoo2025diffusion,deschenaux2025beyond,ben-hamu2025accelerated,zhu2025di,campbell2026selfspeculative}. 
However, there are two main obstacles to build discrete consistency models: 
(1) \emph{intrinsic stochasticity of jump processes}: unlike in the continuous setting where SDE is paired with a probability-flow ODE that shares the same marginals, the stochasticity of jump processes cannot be fully eliminated;
(2) \emph{terminal degeneracy for masked models}: unlike Gaussian in continuous spaces, the masked diffusion converges to a single fully-masked state of zero entropy, signaling a poor starting point for few-step generation. Existing works only partially address these issues: Uniform-state diffusion model (USDM) corrupts tokens within the vocabulary and recovers stochastic terminal states~\citep{sahoo2025diffusion}, but compared to MDM it weakens the distinction between genuinely noisy and semantically meaningful states. Continuous embedding approaches recover deterministic flows only after leaving the native discrete state space and entirely re-training a continuous ~\citep{jo2025continuous,roos2026categorical,lee2026one} or hybrid~\citep{pynadath2025candi,zhou2025coevolutionary,zheng2025continuously} model from scratch, which may compromise the training efficiency and scalability. 

In this work, we propose \textbf{\ours}, a new discrete diffusion model paradigm to address both obstacles with a multi-mask design. In \ours, the state space contains multiple masks instead of a single one. Each clean token is assigned a designated mask, and the forward marginal is designed so that masked mass remains biased toward that designated mask before mixing across the mask subspace (\emph{cf.}, Fig.~\ref{fig:multimdm-schematic}). As such, each mask token bears information and addresses the entropy-collapse issue.
We formulate this multi-mask process precisely, derive its ELBO training objective in closed form, and then develop a discrete consistency-distillation construction based on low-entropy shared-Gumbel couplings and posterior targets along the coupled path.

\subsection{Our Contributions}

\begin{itemize}[leftmargin=1.25em,itemsep=1pt]
    \item We introduce \ours, a new family of discrete diffusion models rooted in the masked-diffusion paradigm. \ours comes with richer geometry than MDM, clearer data-noise distinction than USDM, and convenience in converting from pretrained MDMs.
    \item We derive the corresponding discrete KL objective/ELBO in closed form consisting of a clean-token reconstruction and an intra-mask identification term. Moreover, it enables continual training recipe initialized from pretrained MDMs. 
    \item We formulate a discrete low-entropy coupling for consistency distillation by sharing Gumbel noise across time, identify the consistency target as the future-filtration posterior over clean tokens, and provide the corresponding few-step generation algorithm.
    \item We demonstrate the superior few-step generation performance of \ours after pretraining and consistency distillation on two English-language corpora, and provide preliminary evidence that continual multi-mask adaptation transfers to an 8B pretrained MDM on code and math benchmarks.
\end{itemize}

\section{Preliminaries}
\label{sec:prelim}

We begin by recalling the standard formulation of discrete diffusion models. 

\subsection{Discrete Diffusion Models}
\label{sec:prelim_discrete}

Let $\gV$ denote the vocabulary of non-mask tokens, with $V:=|\gV|$, and let $p_0$ denote the data distribution over clean sequences $\vx_0=(x_0^1,\dots,x_0^L)\in\gV^L$. A discrete diffusion model defines a noising process $(p_t(\vx_t\mid\vx_0))_{t \in [0,1]}$ on a finite state space and learns a reverse-time model that reconstructs $\vx_0$ from partially corrupted states $\vx_t$. In virtually all discrete diffusion variants, the forward process factorizes across coordinates, as
\begin{equation}
\label{eq:prelim_factorization}
    p_t(\vx_t\mid\vx_0)
    =
    \prod_{\ell=1}^L p_t(x_t^\ell\mid x_0^\ell).
\end{equation}
Two canonical choices are uniform-state and masked diffusion. In uniform-state diffusion, the corrupted token remains in $\gV$ and is driven toward the uniform distribution
\begin{equation}
\label{eq:uniform_forward_prelim}
    p_t^{\mathrm{uni}}(x_t\mid x_0)
    =
    \alpha_t\1\{x_t=x_0\} + (1-\alpha_t)\uV(x_t),
    \qquad x_t\in\gV, \quad \text{ with } \uV(x):=1/V.
\end{equation}
In masked diffusion, one adjoins a single absorbing mask token $\mask$ and uses
\begin{equation}
\label{eq:masked_forward_prelim}
    p_t^{\mathrm{mask}}(x_t\mid x_0)
    =
    \alpha_t\1\{x_t=x_0\} + (1-\alpha_t)\1\{x_t=\mask\},
    \qquad x_t\in\gV\cup\{\mask\}.
\end{equation}
USDM has a stochastic terminal law, but it does not mark which coordinates are noisy. MDM makes this distinction explicit and usually leads to stronger modeling performance together with a cleaner mathematical formulation~\citep{sahoo2024simple,shi2024simplified}. 

\subsection{Toward Discrete Consistency Models}
\label{sec:prelim_consistency}

A consistency model aims to replace a long reverse-time simulation by a small number of NN evaluations (NFEs), ideally approaching a deterministic map from a simple reference to clean sequences, as pioneered in the continuous space by~\cite{song2023consistency,boffi2025how}. However, its extension to the discrete setting faces the following two main obstacles.

The first obstacle is \emph{terminal entropy}, especially for MDM. Converging to a single all-mask state, the terminal distribution possesses no information with $H(\vx_1) = 0$ and thus no signal to start with in consistency models. Although USDM resembles a fix by ending uniformly over vocabulary, noisy tokens are indistinguishable from clean ones, possibly leading to inferior generation quality~\citep{austin2021structured,ni2025training}.

The second obstacle is \emph{path stochasticity}. In continuous diffusion, one can derive from the stochastic backward SDE a deterministic probability-flow ODE while preserving time marginals. No immediate analogy exists on finite-state spaces: when admissible dynamics are jump processes, the path itself remains stochastic. One may embed discrete states into a continuous space and then train a continuous flow~\citep{roos2026categorical,lee2026one}, but at the expense of increasing pretraining and consistency training cost. Moreover, the training and distillation is unable to directly make use of high-quality pretrained MDM, which has already shown scalability in language modeling~\citep{nie2025large,song2025seed,ye2025dream}.

Our goal is therefore to design a discrete diffusion process that keeps the reconstruction advantages of masking while mitigating both issues above. 

\section{Methodology}
\label{sec:methodology}

In this section, we introduce the methodological contributions of \ours, including a designated-entry multi-mask process, its corresponding training and inference algorithm, and consistency distillation on top of pretrained models. Simply put, \ours fixes terminal degeneracy by replacing the delta law with a uniform law over multiple masks, and it helps few-step distillation by yielding lower-entropy coupled paths and sharper posteriors.

\subsection{Multi-Mask Forward Process}
\label{sec:multi_mask_forward}

We augment the single-mask state space by a set $\gM$ of mask tokens, with $M:=|\gM|\ll|\mathcal V|$, and write the time-$t$ sequence as $\vx_t=(x_t^1,\dots,x_t^L)\in(\gV\cup\gM)^L$. Each clean token $x\in\gV$ is assigned a designated mask $m_x\in\gM$, represented by an assignment matrix $\mT\in\{0,1\}^{M\times V}$ with exactly one nonzero entry in each column such that, for one-hot representations $\ve_x\in\{0,1\}^V$ and $\ve_{m_x}\in\{0,1\}^M$, we have $\mT\ve_x = \ve_{m_x}$. We also write $\vuM:=\1/M$ for the uniform distribution on mask states with convention $\uM(x):=0$ for all $x\in\gV$.

Let $\alpha_t$ and $\beta_t$ be differentiable decreasing schedules with $\alpha_0=\beta_0=1$, $\alpha_1=\beta_1=0$, and $0<\alpha_t,\beta_t<1$ for $t\in(0,1)$. For a clean token $x_0\in\gV$, define the mask-side distribution
\begin{equation}
\label{eq:rt_def}
    r_t^{x_0}(k)
    :=
    \beta_t\1\{k=m_{x_0}\} + (1-\beta_t)\uM(k),
    \qquad k\in\gM.
\end{equation}
The single-token forward marginal is
\begin{equation}
\label{eq:forward_marginal_single}
    p_t(x_t\mid x_0)
    =
    \alpha_t\1\{x_t=x_0\} + (1-\alpha_t)r_t^{x_0}(x_t),
    \qquad x_t\in\gV\cup\gM.
\end{equation} 
This parameterization separates two effects: $\alpha_t$ controls how quickly a clean token disappears from the visible vocabulary, while $\beta_t$ controls how concentrated the masked mass remains on the designated mask $m_{x_0}$.

\paragraph{Forward and backward kernels.}

The representation in \eqref{eq:forward_marginal_single} yields closed-form forward and posterior kernels. 
As in \eqref{eq:prelim_factorization}, the sequence-level process factorizes coordinatewise, so all subsequent derivations reduce to the single-token case.

For $0\le s<t\le 1$, define $\alpha_{t\mid s}:=\alpha_t/\alpha_s$ and $\beta_{t\mid s}:=\beta_t/\beta_s$. The forward two-time kernel is
\begin{equation}
\label{eq:two_time_kernel_main}
    p_{t\mid s}(x_t\mid x_s)
    =
    \begin{cases}
        \alpha_{t\mid s}\1\{x_t=x_s\} + (1-\alpha_{t\mid s})r_t^{x_s}(x_t), & x_s\in\gV,\\[1mm]
        \beta_{t\mid s}\1\{x_t=x_s\} + (1-\beta_{t\mid s})\uM(x_t), & x_s\in\gM,
    \end{cases}
\end{equation}
where the second line is supported only on $x_t\in\gM$. It can be shown that the corresponding conditional backward kernel given the clean token $x_0$ is
\begin{equation}
\label{eq:two_time_backward_main}
    p_{s\mid t}(x_s\mid x_t,x_0)
    =
    \begin{cases}
        \1\{x_s=x_0\}, & x_t=x_0,\\[1mm]
        \frac{\alpha_s-\alpha_t}{1-\alpha_t}\1\{x_s=x_0\}
        +
        \frac{1-\alpha_s}{1-\alpha_t}
        \frac{r_s^{x_0}(x_s)\bigl[\beta_{t\mid s}\1\{x_t=x_s\}+(1-\beta_{t\mid s})\uM(x_t)\bigr]}{r_t^{x_0}(x_t)}
        \1\{x_s\in\gM\}, & x_t\in\gM.
    \end{cases}
\end{equation}
We refer to Prop.~\ref{prop:two_time_consistency} and Thm.~\ref{thm:conditional_backward_kernel} for the proof details.

\paragraph{Designated-entry interpretation.}

Under the closed-form kernel \eqref{eq:two_time_kernel_main}, the masked marginal is biased toward the mask class associated with the clean token. Prop.~\ref{prop:entry_mixing_equivalence} shows that the same one-time marginals also admit a more intuitive entrance-and-mixing interpretation: there is another CTMC in which a clean token $x_0$ first jumps to $m_{x_0}$ with rate $\lambda_t$ and then, after entering the mask subspace, refreshes uniformly among masks with total rate $\gamma_t$, with 
\[
    \lambda_t=-\frac{\dot\alpha_t}{\alpha_t},\qquad \text{and} \qquad
    \gamma_t
    =
    \frac{-\dot\alpha_t(1-\beta_t)-(1-\alpha_t)\dot\beta_t}{(1-\alpha_t)\beta_t}.
\]
This is exactly the type of partial information that is useful for later few-step distillation.
While this alternative realization is conceptually useful, its two-time kernel is less explicit because it requires averaging over the unknown entry time into the mask sector. We therefore adopt the simpler kernel~\eqref{eq:two_time_kernel_main} in our implementation, as explained in Remark~\ref{rem:entry_mixing_kernel}.

\paragraph{Number of masks.}

This construction is motivated by the two consistency-model difficulties from \cref{sec:prelim_consistency}. First, because $\alpha_1=\beta_1=0$, the terminal law is exactly $p_1 = \vuM^L$. As a rough rule of thumb, if a dataset contains $N$ sequences of length $L$, then matching the terminal entropy to the crude dataset entropy scale $\log N$ suggests
\begin{equation*}
    H(\vuM^L)=L\log M \gtrsim \log N \approx H(p_0)\quad\text{or equivalently}\quad M\gtrsim N^{1/L}.
\end{equation*}
On practical datasets, this suggests a modest mask number $M$, incurring minimal computational overhead. We empirically observe $M\in [5, 100]$ spans the useful regime (\emph{cf.}, Tab.~\ref{tab:ablation}).

\subsection{Model Parametrization, Training, and Inference}
\label{sec:parameterization_objective}

For each coordinate $\ell$, the network predicts a clean-token distribution $x_{t,\theta}^\ell(\cdot\mid \vx_t)\in\Delta^{V-1}$ conditioned on the full noisy sequence $\vx_t$ at time $t$. We also group this prediction by designated mask classes,
$
    m_{t,\theta}^\ell(\cdot\mid \vx_t)
    :=
    \mT x_{t,\theta}^\ell(\cdot\mid \vx_t)
    \in \Delta^{M-1}.
$

\paragraph{Training objective.} 

The training objective of \ours is summarized below.

\begin{proposition}[Discrete KL objective]
\label{prop:main_kl_objective}
Whenever the current token is masked, $x_t^\ell=k\in\gM$, the per-coordinate loss density $\ell_t^\ell(\theta;\vx_0,\vx_t)$ is defined as
\begin{equation}
\label{eq:per_coordinate_loss_main}
    \underbrace{-\frac{\dot \alpha_t}{1-\alpha_t}
    \log \tfrac{1}{x_{t,\theta}^\ell(x_0^\ell\mid \vx_t)}}_{\text{clean-token reconstruction term}}
    \underbrace{-\frac{\dot \beta_t}{M\beta_t}
    \sum_{j\in\gM\setminus\{k\}}
    \Bigg[
        \psi_t^{x_0^\ell}(j\mid k)
        \log\tfrac{\psi_t^{x_0^\ell}(j\mid k)}{\psi_{t,\theta}^\ell(j\mid \vx_t)}
        +
        \psi_{t,\theta}^\ell(j\mid \vx_t)
        -
        \psi_t^{x_0^\ell}(j\mid k)
    \Bigg]}_{\text{intra-mask identification term}},
\end{equation}
where, for $j\in\gM\setminus\{k\}$, $\psi_t^{a}(j\mid k):=\frac{r_t^{a}(j)}{r_t^{a}(k)}$, and $\psi_{t,\theta}^\ell(j\mid \vx_t):=
    \sum_{a\in\gV} x_{t,\theta}^\ell(a\mid \vx_t)\,\psi_t^{a}(j\mid k)$.
Visible current states contribute zero. Consequently, the full training objective is
\begin{equation}
\label{eq:sequence_objective_main}
    \gL_{\mathrm{disc}}(\theta)
    :=
    \E_{t\sim\operatorname{Unif}(0,1),\,\vx_0\sim p_0,\,\vx_t\sim p_t(\cdot\mid \vx_0)}
    \left[
        \sum_{\ell=1}^L
        \1\{x_t^\ell\in\gM\}
        \,\ell_t^\ell(\theta;\vx_0,\vx_t)
    \right].
\end{equation}
\end{proposition}

The first term in \eqref{eq:per_coordinate_loss_main} is the clean-token reconstruction term weighted by the rate at which visible tokens disappear, which is also the exact loss for the single-mask diffusion model. The second is a new KL-type penalty responsible for mixing within the mask subspace. Thm.~\ref{thm:single_token_closed_form} derives the masked-coordinate density directly from the forward generator and the concrete score, and Cor.~\ref{cor:sequence_objective} lifts it to the full sequence objective.

\begin{theorem}[KL upper bound]
\label{thm:main_kl_upper_bound}
Let $\hat p_0^\theta$ denote the marginal at time $0$ obtained by running the learned backward process from the exact terminal $\vuM^L$. Then $D_{\mathrm{KL}}(p_0\,\|\,\hat p_0^\theta)
\le
\gL_{\mathrm{disc}}(\theta)$.
\end{theorem}

Prop.~\ref{prop:main_kl_objective} yields a straightforward stochastic training procedure: sample a time, draw $\vx_t$ from the exact forward marginal \eqref{eq:forward_marginal_single}, evaluate the clean-token predictor $x_{t,\theta}$, and accumulate the masked-coordinate loss in \eqref{eq:sequence_objective_main}, which is given in Alg.~\ref{alg:training}. Thm.~\ref{thm:main_kl_upper_bound} further identifies the KL objective~\eqref{eq:sequence_objective_main} as a KL upper bound of the generation quality.

\begin{minipage}[t]{0.56\textwidth}
\begin{algorithm}[H]
\caption{\ours Pretraining}
\label{alg:training}
\begin{algorithmic}[1]
\AlgRequire{Initial parameters $\theta$, training data $p_0$, minibatch size $B$, learning rate $\eta$, training steps $K$}
\For{$k=1$ to $K$}
    \State $\{\vx_0^{(i)}\}_{i=1}^B \stackrel{\mathrm{i.i.d.}}{\sim} p_0$, $\{t^{(i)}\}_{i=1}^B \stackrel{\mathrm{i.i.d.}}{\sim}\operatorname{Unif}(0,1)$.
    \State $\vx_t^{(i)}\sim p_{t^{(i)}}(\cdot\mid\vx_0^{(i)})$, $i \in [B]$, using \eqref{eq:forward_marginal_single}.
    \State Evaluate minibatch loss $\widehat{\gL}_{\mathrm{disc}}(\theta)$ using \eqref{eq:sequence_objective_main}.
    \State $\theta \leftarrow \theta - \eta\,\nabla_{\theta}\widehat{\gL}_{\mathrm{disc}}(\theta)$.
\EndFor
\end{algorithmic}
\end{algorithm}
\end{minipage}
\hfill
\begin{minipage}[t]{0.43\textwidth}
\begin{algorithm}[H]
\caption{\ours Inference}
\label{alg:sampling}
\begin{algorithmic}[1]
\AlgRequire{Trained model $\theta$, reverse-time grid $\{t_k\}_{k=0}^K$ with $t_K=1$ and $t_0=0$}
\State $\hat\vx_{t_K}\sim \vuM^L$.
\For{$k=K$ down to $1$}
    \State $\hat\vx_{t_{k-1}} \sim p_{t_{k-1}\mid t_k}^\theta(\cdot\mid \hat\vx_{t_k})$\\ using \eqref{eq:model_reverse_kernel_main}.
\EndFor
\State \Return $\hat\vx_{t_0}$.
\end{algorithmic}
\end{algorithm}
\end{minipage}

\paragraph{Continual adaptation from a pretrained MDM.}

The multi-mask parameterization can be initialized from a pretrained MDM. In the DiT~\citep{peebles2023scalable} backbone, the output head over clean tokens is unchanged, and the only new state-specific parameters are the embeddings associated with the extra mask tokens. A simple warm start is therefore to copy the pretrained single-mask embedding into every mask embedding in $\gM$ at initialization. This makes all masks identical at the training start, after which the model learns to distinguish them under the multi-mask objective \eqref{eq:per_coordinate_loss_main}.

The decomposition in \eqref{eq:per_coordinate_loss_main} also suggests a practical learning recipe: warm-start with the clean-token reconstruction term before gradually introducing the intra-mask identification term. Empirically, we find this curriculum helps improve the model's generation quality significantly (\emph{cf.}, Tab.~\ref{tab:text-gen-pretrain}).

\paragraph{Inference algorithm.}

Reverse-time inference uses the conditional backward kernel \eqref{eq:two_time_backward_main}. If the current state is already visible, the backward kernel is deterministic and keeps that token fixed. If the current state is masked, we replace the unknown clean-token posterior with the model prediction and define the model-induced backward kernel by posterior averaging,
\begin{equation}
\label{eq:model_reverse_kernel_main}
    p_{s\mid t}^\theta(\vx_s\mid \vx_t) \approx \prod_{\ell=1}^L p_{s\mid t}^\theta(x_s^\ell\mid x_t^\ell, \vx_t),
    \ \ \text{with}\ \ 
    p_{s\mid t}^\theta(x_s^\ell\mid x_t^\ell,\vx_t)
    =
    \begin{cases}
        \1\{x_s^\ell=x_t^\ell\}, & x_t^\ell\in\gV,\\[1mm]
        \sum_{x_0^\ell\in\gV}
        x_{t,\theta}^\ell(x_0^\ell\mid \vx_t)
        p_{s\mid t}(x_s^\ell\mid x_t^\ell,x_0^\ell), & x_t^\ell\in\gM,
    \end{cases}
\end{equation}
The only approximation in \eqref{eq:model_reverse_kernel_main} is the product factorization across coordinates. The single-coordinate kernel is obtained by exact posterior averaging once the true clean-token posterior is replaced by the model prediction $x_{t,\theta}^\ell(\cdot\mid\vx_t)$.

\subsection{Consistency Distillation}
\label{sec:distillation}

The multi-mask teacher is particularly well suited to discrete consistency distillation. Relative to the single-mask teacher, it preserves explicit noise markers while providing a non-degenerate terminal distribution and a richer hierarchy of intermediate masked states. The remaining challenge is that the backward path is still a stochastic jump process. Our distillation strategy therefore has two ingredients: a low-entropy shared-Gumbel coupling, and a consistency objective that aligns clean-token posteriors along that coupled path.

\paragraph{Low-entropy shared-Gumbel coupling.}

\begin{wrapfigure}{r}{0.48\textwidth}
    \vspace{-0.8em}
    \centering
    \includegraphics[width=.95\linewidth]{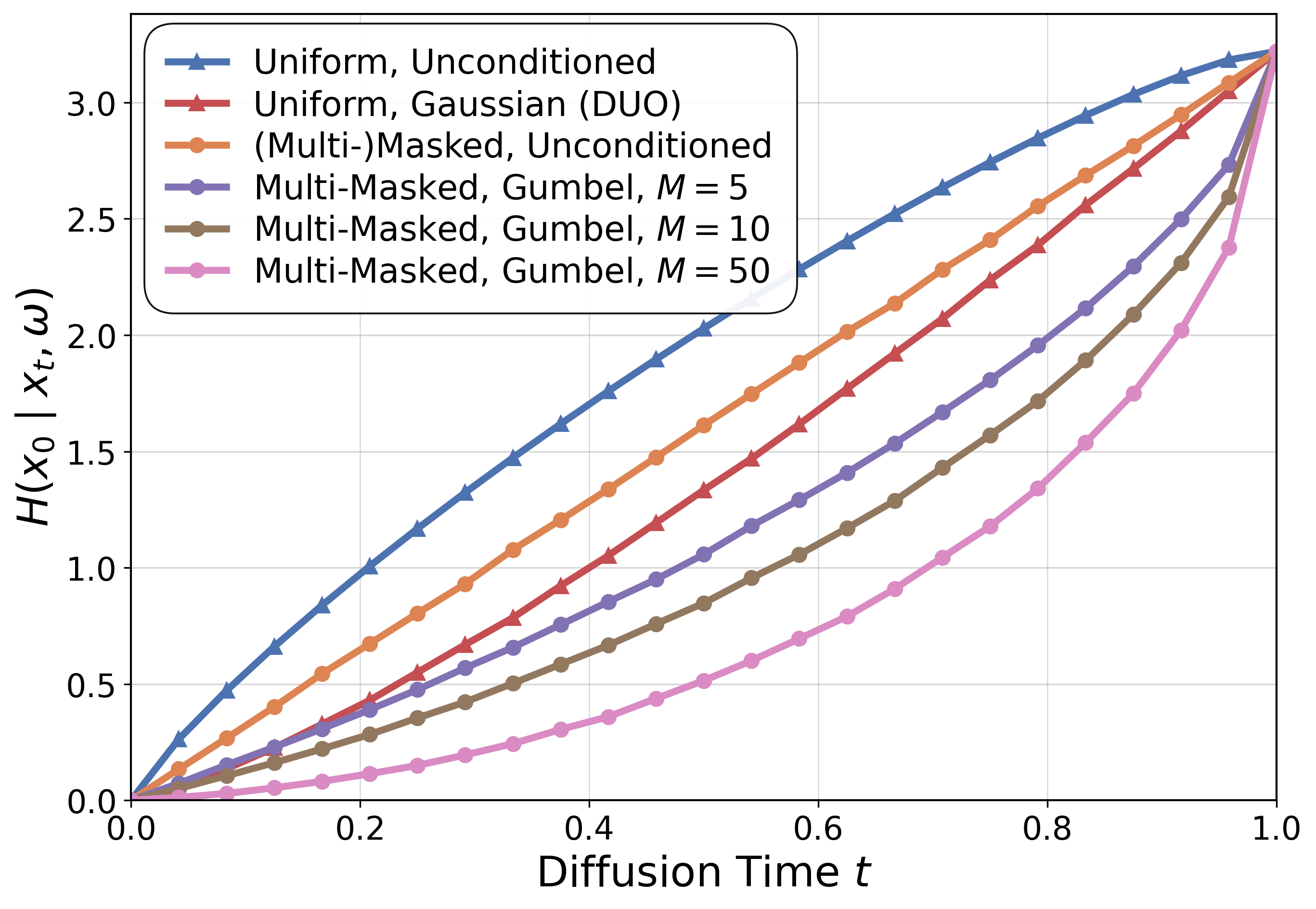}
    \vspace{-1.2em}
    \caption{
        Conditional entropy $H(\vx_0 \mid \vx_t, \omega)$ for different couplings with noise source $\omega$ on a synthetic Zipf dataset (\emph{cf.} App.~\ref{app:conditional_entropy}). Unconditional entropy refers to $H(\vx_0 \mid \vx_t)$. 
    }
    \label{fig:conditional_entropy}
    \vspace{-1.7em}
\end{wrapfigure}

For any time $t$, the conditional entropy $H(\vx_0 \mid \vx_t)$ quantifies the posterior uncertainty of the process, where the ideal is $H(\vx_0 \mid \vx_t) = 0$ under a purely deterministic coupling.
However, once the teacher one-time marginals $p_t(\cdot\mid\vx_0)$ are fixed, $H(\vx_0\mid \vx_t)$ is fixed as well. What can still be reduced is the residual pathwise randomness $H(\vx_0 \mid \vx_t, \omega) \leq H(\vx_0\mid \vx_t)$ after conditioning on an auxiliary noise source $\omega$.

Following the consistency-model perspective of matching two points that lie on the same trajectory~\citep{song2023consistency,boffi2025how}, but staying entirely in the discrete-state space, we construct a shared-noise coupling with Gumbel noises. For each coordinate $\ell$ and state $z\in\gV\cup\gM$, sample an i.i.d. standard Gumbel random variable $g_z^\ell$ and fix it for all $t$. With the convention $\log 0=-\infty$, define
\begin{equation}
\label{eq:gumbel_coupled_path}
    \widetilde x_t^\ell
    :=
    \argmax_{z\in\gV\cup\gM}
    \bigl\{\log p_t(z\mid x_0^\ell) + g_z^\ell\bigr\},
    \qquad t\in[0,1].
\end{equation}
By the Gumbel-max trick, $\widetilde x_t^\ell\sim p_t(\cdot\mid x_0^\ell)$ for every $t$, so $\widetilde\vx_t$ has exactly the same marginals as the teacher model. Conditioned on $\omega:=\{g_z^\ell\}_{\ell,z}$, the map $t\mapsto \widetilde\vx_t$ is deterministic. 

This contrasts with DUO-DCD by~\citet{sahoo2025diffusion}, which constructs coupled discrete trajectories by projecting a shared-Gaussian path in a continuous latent space to discrete states. In fact, they can also be rewritten in a shared-Gumbel form with the same one-time marginals (\emph{cf.} Thm.~\ref{thm:duo_gumbel_path}). Our construction is therefore a natural analogue, extending immediately to both single-mask and multi-mask diffusion families without introducing auxiliary continuous latent processes. In uniform-state diffusion, the shared-Gaussian path jumps at most once. In the present multi-mask setting, an analogous at-most-two-jumps statement holds under mild conditions on $\alpha_t$ and $\beta_t$ (\emph{cf.} Prop.~\ref{prop:mm_gumbel_two_jumps}). Fig.~\ref{fig:conditional_entropy} visualizes the conditional entropy $H(\vx_0 \mid \vx_t, \omega)$ as a function of time $t$ and shows that our coupling reduces the pathwise uncertainty compared to the original Gaussian coupling, with a larger reduction for larger numbers of masks $M$. We refer to App.~\ref{app:conditional_entropy} for its detailed settings.

\paragraph{Consistency objective and target.}
We parameterize the consistency model with the same clean-token head as the diffusion model and write $f_{t,\theta}^\ell(\cdot\mid \widetilde\vx_t)\in\Delta^{V-1}$ for its posterior prediction at coordinate $\ell$ with the boundary condition $f_{0,\theta}^\ell(\cdot\mid \vx_0)=\ve_{x_0^\ell}$.

\begin{theorem}[Future-filtration consistency target]
\label{thm:consistency_target}
Let $\gG_t:=\sigma(\widetilde\vx_u:u\in[t,1])$ denote the future filtration of the coupled path. 
Fix a coordinate $\ell$ and $0\le s<t\le 1$. Let $f_s^\ell\in\Delta^{V-1}$ be any $\gG_s$-measurable predictor. Among all simplex-valued, $\gG_t$-measurable predictors $f_t^\ell$, the unique minimizer of $\E\bigl[D_{\mathrm{KL}}(f_s^\ell\|f_t^\ell)\bigr]$
is
$
    f_t^{\ell,\star}
    =
    \E[f_s^\ell\mid \gG_t].
$
In particular, if
$f_s^\ell
=
\E[\ve_{x_0^\ell}\mid \gG_s],$
then
\begin{equation}
\label{eq:consistency_target_main}
    f_t^{\ell,\star}
    =
    \E[\ve_{x_0^\ell}\mid \gG_t]
    =
    \Pr(x_0^\ell=\cdot\mid \widetilde\vx_{[t,1]}).
\end{equation}
\end{theorem}

Thm.~\ref{thm:consistency_target}, with proof shown in App.~\ref{app:consistency_targets}, shows that the reverse-time posterior martingale family $\E[\ve_{x_0^\ell}\mid \gG_t]$ is the canonical fixed point of the pairwise objective $ \E_{\vx_0\sim p_0,\,0 \leq s < t \leq 1}[
    D_{\mathrm{KL}}(
        f_s^\ell
        \,\|\,
        f_t^\ell
)],$ which is the precise sense in which it provides the natural consistency target along the coupled path.

Towards practical implementation, we impose the following two approximations.
\paragraph{Approximation 1: Backward Markov compression.}
The ideal target in \eqref{eq:consistency_target_main} depends on the path $\widetilde\vx_{[t,1]}$. In practice, we compress the future-filtration posterior~\eqref{eq:consistency_target_main} to a backward Markov predictor, \emph{i.e.}, a tractable model observing only the current sequence $\widetilde\vx_t$ so that 
\begin{equation}
\label{eq:posterior_approximation}
    f_t^\ell(\cdot\mid \widetilde\vx_t) \approx f_t^{\ell,\star} = \Pr(x_0^\ell=\cdot\mid \widetilde\vx_{[t,1]}).
\end{equation}
We train it with an EMA target network, exactly as in consistency distillation~\citep{song2023consistency,sahoo2025diffusion}. A practical adjacent-pair specialization for a step size $\delta \in (0,1]$ is
\begin{equation}
\label{eq:consistency_distillation_loss_main}
    \gL_{\mathrm{CD}}^\delta(\theta,\bar\theta)
    :=
    \E_{\vx_0\sim p_0,\,\omega,\, t\sim\operatorname{Unif}(\delta,1)}
    \left[
        \sum_{\ell=1}^L
        D_{\mathrm{KL}}\left(
            \operatorname{sg}\bigl[f_{t-\delta,\bar\theta}^\ell(\cdot\mid \widetilde\vx_{t-\delta})\bigr]
            \,\middle\|\,
            f_{t,\theta}^\ell(\cdot\mid \widetilde\vx_t)
        \right)
    \right].
\end{equation}
where $\operatorname{sg}[\cdot]$ denotes stop-gradient. In practice, distillation starts from a model pretrained by Alg.~\ref{alg:training}; we initialize both the online network and EMA target with these weights.

\begin{algorithm}[t]
\caption{\ours Consistency Distillation}
\label{alg:consistency_training}
\begin{algorithmic}[1]
\AlgRequire{Initial parameters $\theta$ from Alg.~\ref{alg:training}, data distribution $p_0$, minibatch size $B$, training steps $K$, step size schedule $\{\delta_k\}_{k=1}^K$, EMA parameter $\mu$, learning rate $\eta$}
\State Set target parameters $\bar\theta\leftarrow\theta$
\For{$k=1$ to $K$}
    \State $\{\vx_0^{(i)}\}_{i=1}^B \stackrel{\mathrm{i.i.d.}}{\sim} p_0$, $\{t^{(i)}\}_{i=1}^B \stackrel{\mathrm{i.i.d.}}{\sim}\operatorname{Unif}(\delta_k,1)$, and set $s^{(i)}\leftarrow t^{(i)}-\delta_k$, for $i \in [B]$.
    \State $g_z^{\ell,(i)}\stackrel{\mathrm{i.i.d.}}{\sim}\operatorname{Gumbel}(0,1)$ for all $i\in[B]$, $\ell\in[L]$, and $z\in\gV\cup\gM$.
    \State 
    $
        \widetilde x_{\bullet}^{\ell,(i)}
        =
        \displaystyle\argmax_{z\in\gV\cup\gM}
        \bigl\{\log p_{\bullet}(z\mid x_0^{\ell,(i)}) + g_z^{\ell,(i)}\bigr\},
    $
    for $\bullet \in \{s^{(i)}, t^{(i)}\}$, $i \in [B]$, and $\ell \in [L]$.
    \State Evaluate minibatch loss $\widehat{\gL}_{\mathrm{CD}}^{\delta_k}(\theta,\bar\theta)$ using \eqref{eq:consistency_distillation_loss_main}.
    \State $\theta \leftarrow \theta - \eta\,\nabla_{\theta}\widehat{\gL}_{\mathrm{CD}}^{\delta_k}(\theta,\bar\theta)$, $\bar\theta \leftarrow \mu \bar\theta + (1-\mu) \theta$.
\EndFor
\end{algorithmic}

\end{algorithm}

\begin{algorithm}[H]
\caption{\ours Few-Step Generation}
\label{alg:few_step_generation}
\begin{algorithmic}[1]
\AlgRequire{$\theta,\ \{t_k\}_{k=0}^K,\ t_K=1,\ t_0=0$}
\State $\hat\vx_{t_K}\sim \vuM^L$.
\For{$k=K$ down to $1$}
    \State $\hat\vx_{t_{k-1}} \sim \widetilde p_{t_{k-1}\mid t_k}^\theta(\cdot\mid \hat\vx_{t_k})$ using \eqref{eq:consistency_reverse_ansatz}.
\EndFor
\State \Return $\hat\vx_{t_0}$.
\end{algorithmic}
\end{algorithm}

\paragraph{Approximation 2: Posterior-backward ansatz.}
To obtain a few-step sampler, we treat the posterior predictor $f_{t,\theta}$ as a proxy for the clean-token posterior in light of the approximation~\eqref{eq:posterior_approximation} and plug it into the exact conditional backward kernel~\eqref{eq:two_time_backward_main}. The induced backward kernel is thus defined as
\begin{equation}
\label{eq:consistency_reverse_ansatz}
    \widetilde p_{s\mid t}^\theta(\hat\vx_s \mid \hat\vx_t) \approx \prod_{\ell=1}^L \widetilde p_{s\mid t}^\theta(\hat x_s^\ell\mid \hat x_t^\ell,\hat\vx_t),
    \ \ \text{with}\ \ 
    \widetilde p_{s\mid t}^\theta(\hat x_s^\ell\mid \hat x_t^\ell,\hat\vx_t)
    =
    \begin{cases}
        \1\{\hat x_s^\ell=\hat x_t^\ell\}, & \hat x_t^\ell\in\gV,\\[1mm]
        \sum_{a\in\gV}
        f_{t,\theta}^\ell(a\mid \hat\vx_t)
        p_{s\mid t}(\hat x_s^\ell\mid \hat x_t^\ell,a), & \hat x_t^\ell\in\gM.
    \end{cases}
\end{equation}
Equivalently, on masked coordinates one first samples a clean-token hypothesis $\hat a^\ell\sim f_{t,\theta}^\ell(\cdot\mid \hat\vx_t)$ and then samples $\hat x_s^\ell\sim p_{s\mid t}(\cdot\mid \hat x_t^\ell,\hat a^\ell)$. The resulting few-step algorithm is given in Alg.~\ref{alg:few_step_generation}.

The rationale behind this few-step distillation strategy is that given the low-entropy shared-Gumbel paths~\eqref{eq:gumbel_coupled_path}, the consistency loss~\eqref{eq:consistency_distillation_loss_main} learns an approximation of the posterior likelihood~\eqref{eq:posterior_approximation}. Given more information observed from $\vx_t$ about its own latent history, the approximate posterior becomes more concentrated and thus less correlated among coordinates, yielding a better approximation in~\eqref{eq:consistency_reverse_ansatz} with larger step size than the original~\eqref{eq:model_reverse_kernel_main} that permits fewer-step generation.

\section{Experiments}
\label{sec:experiments}

In this section, we present experimental results of \ours and showcase its superior performance on both pretraining and few-step generation after consistency distillation. Baselines include DUO (uniform-state)~\citep{sahoo2025diffusion}, MDLM (single-mask)~\citep{sahoo2024simple}, and CANDI (hybrid discrete-continuous)~\citep{pynadath2025candi}.
We train \ours on two English-language corpora, denoted \corpusA and \corpusB. The main evaluation metrics are (1) \emph{average generative perplexity} (GenPPL), evaluated by GPT-2, and (2) \emph{sample entropy}, following common practice in the literature~\citep{arriola2025block,sahoo2025diffusion}. Unless otherwise specified, these controlled experiments use a 170M-parameter diffusion transformer (DiT)~\citep{peebles2023scalable}. We additionally study continual multi-mask adaptation of LLaDA-8B-Base in Sec.~\ref{sec:exp_llada}. We refer to App.~\ref{app:experimental_details} for details.

\subsection{Pretraining}
\label{sec:exp_pretraining}

We first compare the pretraining performance of our \ours model, \emph{i.e.}, training with the loss~\eqref{eq:sequence_objective_main} as described in Alg.~\ref{alg:training}, with several baselines. In Tab.~\ref{tab:text-gen-pretrain}, we present pretraining results after 200K steps. Specifically, we consider two training recipes: (1) \ours-rand, training with 200K steps from random initialization; (2) \ours-cont, the continual training of \ours for 50K steps initialized from an MDLM trained for 150K steps. Unless otherwise noted, all baselines use the same 170M DiT backbone, tokenizer, and training budget.

{\renewcommand{\thefootnote}{\fnsymbol{footnote}}
\begin{table}[h]
\centering
\small
\setlength{\tabcolsep}{3pt}
\captionsetup[sub]{skip=6pt}
\renewcommand{\thesubtable}{\arabic{subtable}}

\begin{subtable}{\textwidth}
\centering
\begin{tabular}{ccccccccccc}
\toprule
\multirow{2}{*}{\textbf{Method}} & \multicolumn{5}{c}{\textbf{Entropy-aligned}} & \multicolumn{5}{c}{\textbf{Temperature-aligned}} \\
 & $K=\text{4}$ & 8 & 16 & 32 & 64 & $K=\text{4}$ & 8 & 16 & 32 & 64 \\
\midrule

DUO \scriptsize{\citep{sahoo2025diffusion}}
& 995.3
& 509.4
& {\color{gray}193.6\footnotemark[1]}
& {\color{gray}110.1\footnotemark[1]}
& {\color{gray}70.9\footnotemark[1]}
& 1926.5 & 595.7 & 193.6 & \bestcell{126.8} & 117.8 \\

MDLM \scriptsize{\citep{sahoo2024simple}}
& 721.1
& 280.1
& 117.0
& 76.5
& 65.9
& 1741.9
& 599.9
& 285.4
& 153.1
& 100.5 \\

CANDI \scriptsize{\citep{pynadath2025candi}} 
& 666.5
& 237.5
& 126.8
& 89.0
& 69.6
& 2251.3
& 773.1
& 375.9
& 230.7
& 189.8
\\

\midrule

\textbf{\ours-rand}
& 669.6
& 244.2
& 114.6
& 85.0
&71.4
& \bestcell{1603.4}
& 606.8
& 258.8
& 142.7
& 102.1 \\

\textbf{\ours-cont}
& \bestcell{558.8}
& \bestcell{219.1}
& \bestcell{102.8}
& \bestcell{68.4}
& \bestcell{59.7 }
& 1652.9 
& \bestcell{527.5}
& \bestcell{183.6} 
& 128.9 
& \bestcell{86.2} \\
\bottomrule
\end{tabular}
\vskip -.1em
\caption{\corpusA. Entropy-aligned results match the corpus entropy ($H = 5.44$); temperature-aligned results use $\tau=1$.}
\label{tab:text-gen-corpus-a}
\end{subtable}

\vspace{0.6em}

\begin{subtable}{\textwidth}
\centering
\begin{tabular}{ccccccccccc}
\toprule
\multirow{2}{*}{\textbf{Method}} & \multicolumn{5}{c}{\textbf{Entropy-aligned}} & \multicolumn{5}{c}{\textbf{Temperature-aligned}} \\
  & $K=\text{2}$ & 4 & 8 & 16 & 32 & $K=\text{2}$ & 4 & 8 & 16 & 32 \\
\midrule

DUO \scriptsize{\citep{sahoo2025diffusion}}
& 679.1
& 318.5
& 178.6
& \bestcell{122.8}
& 99.3
& 555.8
& 233.2
& \bestcell{144.9}
& \bestcell{109.9}
& 104.5 \\

MDLM \scriptsize{\citep{sahoo2024simple}}
& 730.3
& 359.8
& 203.4
& 142.3
& 112.5
& 840.1
& 406.7
& 212.6
& 149.6
& 121.8 \\

\midrule

\textbf{\ours-rand}
& 727.6
& 365.8
& 213.6
& 140.8
& 114.8
& 885.4
& 424.4
& 242.7
& 153.3 
& 143.0 \\

\textbf{\ours-cont}
& \bestcell{400.0}
& \bestcell{289.6}
& \bestcell{174.8}
& 137.1
& \bestcell{98.8}
& \bestcell{151.8}
& \bestcell{182.6}
& 171.7
& 131.2 
& \bestcell{100.4}\\
\bottomrule
\end{tabular}
\vskip -.1em
\caption{\corpusB. Entropy-aligned results match the corpus entropy ($H = 4.31$); temperature-aligned results use $\tau=1$. CANDI is omitted.\footnotemark[2]}
\label{tab:text-gen-corpus-b}
\end{subtable}

\vskip -.4em
\caption{Comparisons of unconditional GenPPL ($\downarrow$) after pretraining (Alg.~\ref{alg:training}) with varying numbers of steps $K$. The number of masks $M$ is 50.}
\label{tab:text-gen-pretrain}
\end{table}

\textbf{Perplexity-entropy trade-off.} As observed by prior works such as~\cite{pynadath2025candi}, GenPPL is highly correlated with the generated sample entropy. For a comprehensive evaluation, we vary $\tau$ to draw the Pareto front of the perplexity-entropy trade-off, as depicted later in Figs.~\ref{fig:corpus-a-pretrain} and \ref{fig:corpus-b-pretrain}. The farther a curve lies toward the lower right, the better the perplexity-entropy trade-off. In Table~\ref{tab:text-gen-pretrain}, we report GenPPL in two scenarios: (1) \emph{entropy-aligned}, where the sample entropy is aligned by adjusting the sampling temperature $\tau$ (\emph{cf.} App.~\ref{app:tau}); (2) \emph{temperature-aligned}, the classical setting with $\tau = 1$. In both scenarios, \ours-rand achieves competitive results on its own, while \ours-cont improves further, with especially large gains on \corpusA. When entropy is matched to that of \corpusA, \ours-cont lowers GenPPL relative to MDLM from $721.1$ to $558.8$ at $K=4$.

\subsection{Consistency Distillation}
\label{sec:exp_consistency}

We next present the few-step generation results obtained by our proposed consistency distillation algorithm. In Fig.~\ref{fig:distill-exp}, we compare the Pareto front of \ours distilled with shared-Gumbel paths~\eqref{eq:gumbel_coupled_path} with its counterparts, DUO-DCD with shared-Gaussian paths~\citep{sahoo2025diffusion}, and MDLM with shared-Gumbel paths, another new algorithm enabled by our Gumbel-coupling design. \ours also works synergistically with other consistency distillation methods, \emph{e.g.}, SDTT~\citep{deschenaux2025beyond}, which forms a more informative teacher by simulating the teacher model and gathering logits for multiple steps, at the cost of increased NFEs in training. Across both settings, distilled \ours models achieve favorable final Pareto fronts. These results reflect the combined benefits of a stronger pretrained teacher and distillation, rather than uniformly larger relative gains from distillation.

\footnotetext[1]{We observe that DUO cannot attain the prescribed entropy by varying temperature $\tau$. Thus we report GenPPL at the closest attainable entropy instead.}
\footnotetext[2]{CANDI is omitted because its official implementation does not provide a setup for \corpusB.}}

\begin{figure*}[h]
    \centering
    \begin{minipage}[t]{0.65\textwidth}
        \vspace{0pt}
        \centering
        \begin{subfigure}{0.49\linewidth}
            \centering
            \includegraphics[width=\linewidth]{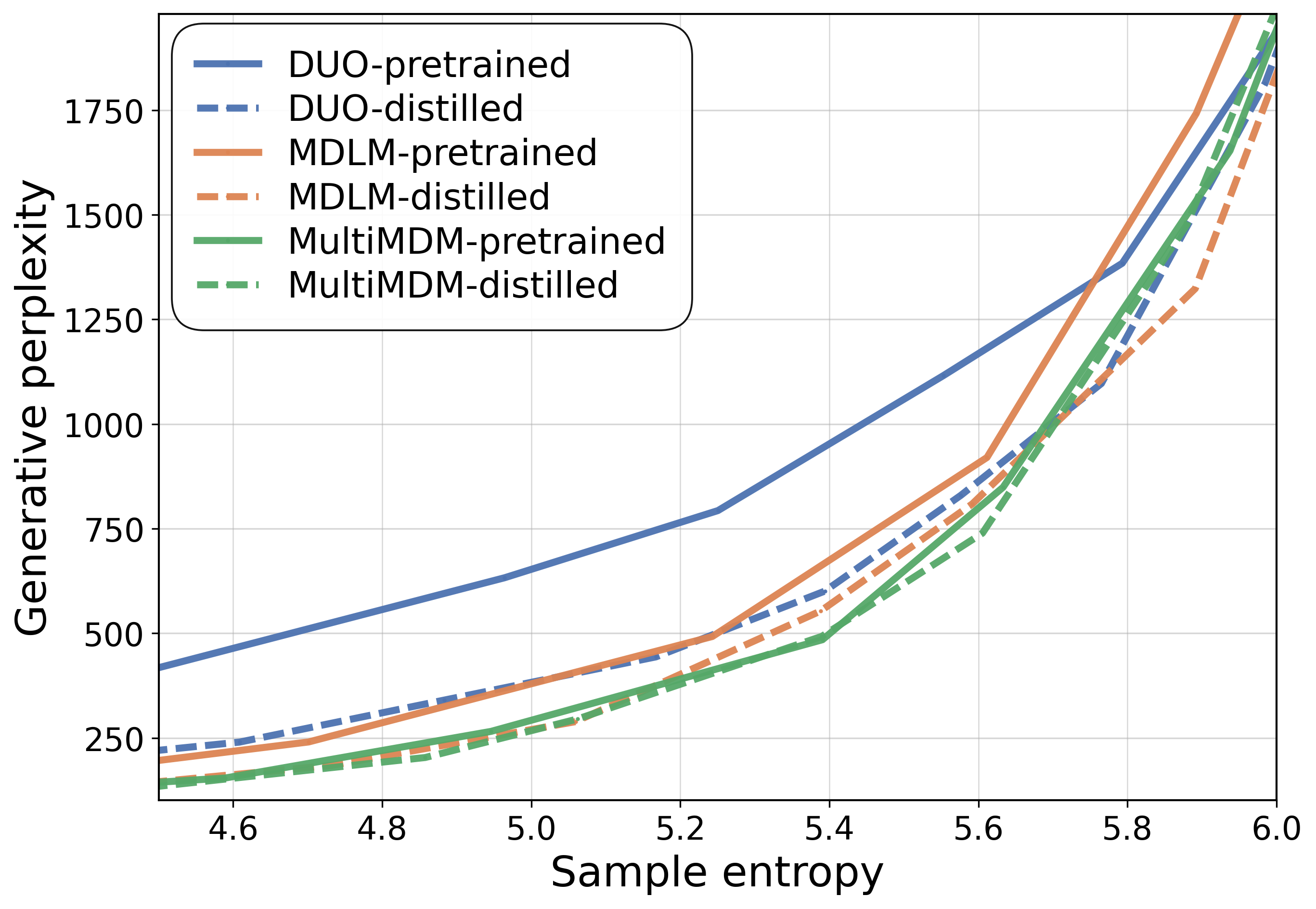}
            \caption{DCD/Gumbel}
            \label{fig:gumbel-4-main}
        \end{subfigure}
        \begin{subfigure}{0.49\linewidth}
            \centering
            \includegraphics[width=\linewidth]{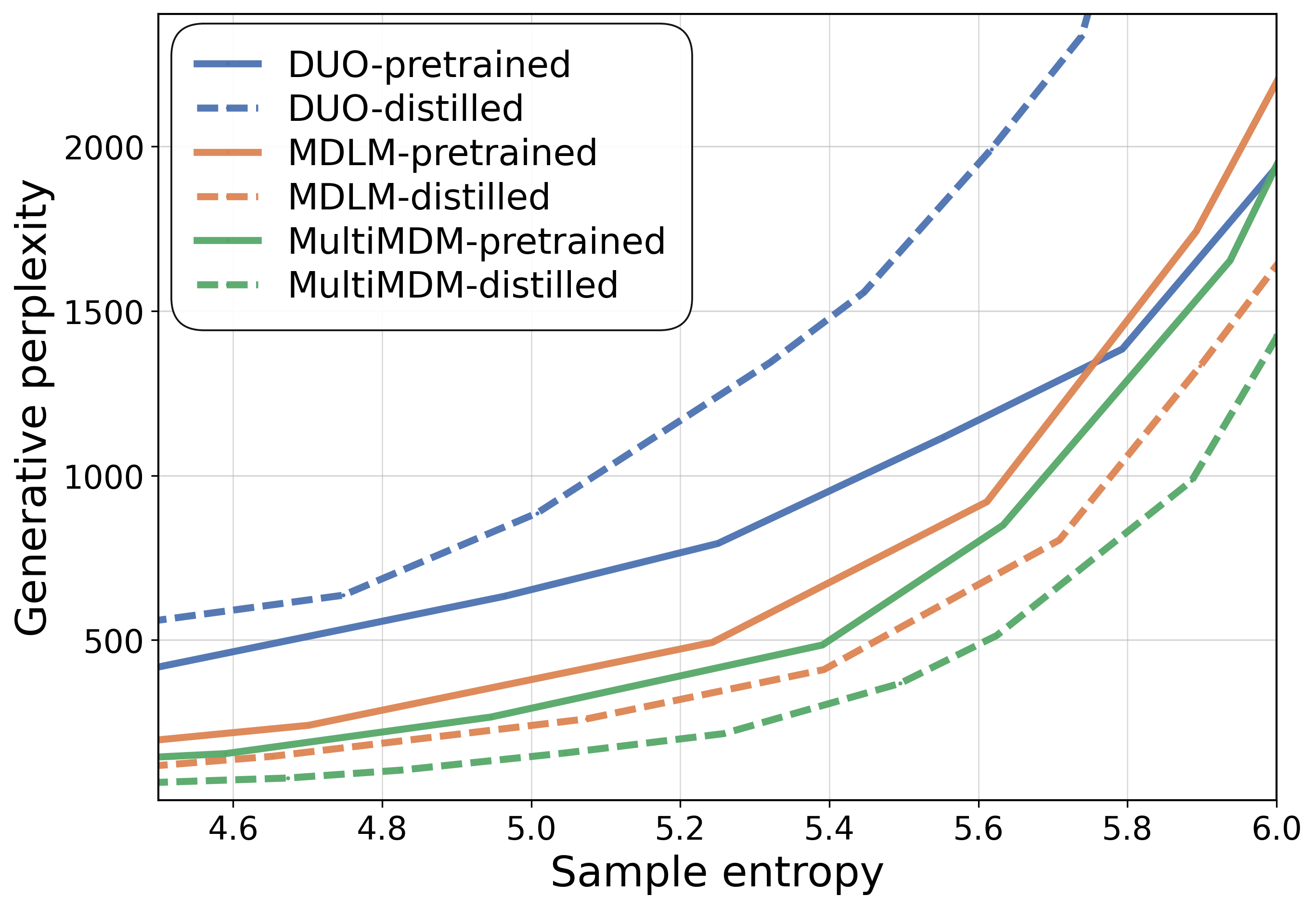}
            \caption{SDTT}
            \label{fig:sdtt-4-main}
        \end{subfigure}
        \vskip -.5em
        \captionof{figure}{Perplexity-entropy Pareto fronts of distilled models sampled with 4 steps by SDTT and DCD/Gumbel.}
        \label{fig:distill-exp}
    \end{minipage}
    \hfill
    \begin{minipage}[t]{0.32\textwidth}
        \vspace{0pt}
        \centering
        \includegraphics[width=\linewidth]{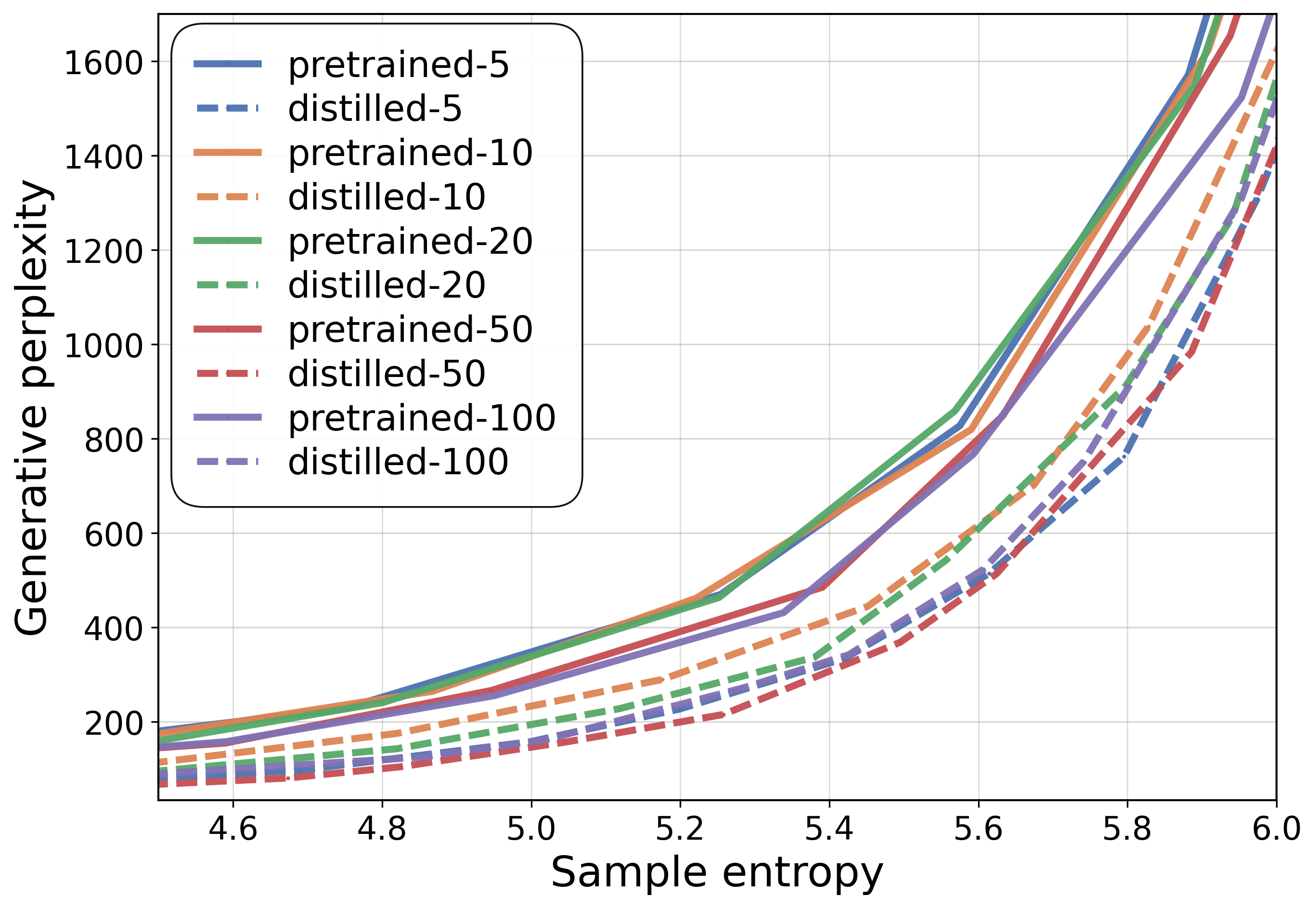}
        \captionof{figure}{Perplexity-entropy Pareto fronts of \ours at 4 steps under different $M$.}
        \label{fig:ablation-4-main}
    \end{minipage}
\end{figure*}

\subsection{Ablation Studies}
\label{sec:exp_ablation}

We conduct ablation studies on the number of masks $M$ to study the impact of the multi-mask design. Tab.~\ref{tab:ablation} reflects that introducing more masks generally improves performance with gains saturating around $M=50$. Fig.~\ref{fig:ablation-4-main} provides a closer look into how multi-mask evidently helps amplify the pretraining performance gain further in consistency distillation. This clearly demonstrates the effective improvement of \ours over MDM. Further experiment results and ablation studies are presented in App.~\ref{app:additional}.

\begin{table}[H]
    \centering
    \small
    \setlength{\tabcolsep}{6pt}
    \begin{tabular}{lccccc}
    \toprule
    \textbf{Method} & $K=4$ & 8 & 16 & 32 & 64 \\
    \midrule
    MDLM ($M=1$) & 721.1 & 280.1 & 117.0 & 76.5 & 65.9 \\
    \midrule
    \textbf{\ours}, $M=5$ & 676.7 & 280.7 & 118.1 & 80.2 & 62.6 \\
    \textbf{\ours}, $M=10$ & 673.6 & 268.3 & 103.5 & 74.3 & 61.4 \\
    \textbf{\ours}, $M=20$ & 698.0 & 247.1 & 114.0 & 75.6 & \bestcell{57.3} \\
    \textbf{\ours}, $M=50$ & \bestcell{558.8} & \bestcell{219.1} & \bestcell{102.8} & \bestcell{68.4} & 59.7 \\
    \textbf{\ours}, $M=100$ & 564.8 & 239.0 & 125.3 & 112.0 & 96.7 \\
    \bottomrule
    \end{tabular}
    \caption{GenPPL ($\downarrow$) of pretrained \ours on \corpusA for different values of $M$, with sample entropy matched to the corpus entropy ($H=5.44$).}
    \label{tab:ablation}
\end{table}

\subsection{Multi-Mask Adaptation of Larger Models}
\label{sec:exp_llada}

To test whether the multi-mask construction transfers beyond the controlled 170M unconditional language-modeling setting, we continually adapt LLaDA-8B-Base~\citep{nie2025large} into a model with $M=50$ masks, which we call \multillada. The adaptation preserves the pretrained clean-token output head while learning the additional mask-token embeddings and low-rank model updates. We then evaluate the original and adapted models on MATH500, GSM8K, HumanEval, and MBPP at 1, 2, and 4 generated tokens per sampling step (TPS).

\begin{table}[H]
    \centering
    \small
    \setlength{\tabcolsep}{7pt}
    \begin{tabular}{llccc}
    \toprule
    \textbf{Benchmark} & \textbf{Model} & TPS=1 & 2 & TPS=4 \\
    \midrule
    \multirow{2}{*}{MATH500}
    & LLaDA & 27.0 & 28.6 & 17.5 \\
    & \textbf{\multillada} & \bestcell{30.2} & \bestcell{30.2} & \bestcell{20.6} \\
    \midrule
    \multirow{2}{*}{GSM8K}
    & LLaDA & \bestcell{73.9} & 70.9 & 54.6 \\
    & \textbf{\multillada} & 72.7 & \bestcell{71.5} & \bestcell{57.0} \\
    \midrule
    \multirow{2}{*}{HumanEval}
    & LLaDA & \bestcell{47.6} & 28.6 & \bestcell{14.3} \\
    & \textbf{\multillada} & \bestcell{47.6} & \bestcell{33.3} & 9.5 \\
    \midrule
    \multirow{2}{*}{MBPP}
    & LLaDA & 31.7 & \bestcell{28.6} & 11.1 \\
    & \textbf{\multillada} & \bestcell{33.3} & \bestcell{28.6} & \bestcell{14.9} \\
    \bottomrule
    \end{tabular}
    \caption{Math accuracy ($\uparrow$; MATH500/GSM8K) and code Pass@1 ($\uparrow$; HumanEval/MBPP) for LLaDA and \multillada at different TPS.}
    \label{tab:llada}
\end{table}

As shown in Tab.~\ref{tab:llada}, \multillada improves the benchmark score in a majority of the settings. The results provide preliminary evidence that continual multi-mask adaptation applies to larger pretrained MDMs and to code and math reasoning tasks, beyond unconditional text generation. This experiment uses LoRA rather than the full-parameter adaptation used in our 170M experiments and does not apply shared-Gumbel distillation at the 8B scale; matched full-parameter, fully instruction-tuned, and large-scale distilled evaluations remain future work. App.~\ref{app:multi_llada} provides the complete adaptation and evaluation setup.

\section{Discussion}
\label{sec:discussion}

We proposed \ours, a multi-mask discrete diffusion framework that retains the masking structure of mask diffusion while replacing the single absorbing state with a stochastic mask subspace. This improves the suitability of masked diffusion for few-step distillation without giving up the reconstruction advantages of masking. Experiments demonstrate the superior performance of \ours on both pretraining and few-step generation after consistency distillation, improving upon state-of-the-art discrete diffusion baselines. 

As the first utilization of the multi-mask flexibility, we expect further improvements through more carefully semantically optimized design of the designated mapping between clean and mask tokens, and its combination with other acceleration and distillation techniques.

\section*{Acknowledgments}

This work was completed during Sijin Chen, Yinuo Ren, Heyang Zhao, and Ziheng Cheng's internships at ByteDance Seed. We thank our colleagues at ByteDance Seed for helpful discussions and feedback throughout this project.

\bibliography{reference}
\bibliographystyle{colm2026_conference}

\newpage
\appendix

\section{Mathematical Formulation of \ours}
\label{app:detailed_proofs}

In this appendix, we give the full mathematical formulation of the \ours stochastic process and prove the main statements used in Sec.~\ref{sec:methodology}. We work first at the single-token level, where the forward and backward kernels admit explicit closed forms, and then lift the construction to sequences through coordinatewise factorization, culminating in the full training objective.

\subsection{Forward Kernel}
\label{app:forward_process}

In this subsection, we verify that the proposed two-time kernel indeed induces the desired one-time forward marginal used in the main text. This provides the basic consistency check for the \ours forward construction at the single-token level.

For $0\le s<t\le 1$, define $\alpha_{t\mid s}:=\alpha_t/\alpha_s$ and $\beta_{t\mid s}:=\beta_t/\beta_s$. The two-time kernel on $\gV\cup\gM$ used in \ours is defined as
\begin{equation}
\label{eq:two_time_kernel}
    p_{t\mid s}(x_t\mid x_s)
    =
    \begin{cases}
        \alpha_{t\mid s}\1\{x_t=x_s\} + (1-\alpha_{t\mid s})r_t^{x_s}(x_t), & x_s\in\gV,\\[1mm]
        \beta_{t\mid s}\1\{x_t=x_s\} + (1-\beta_{t\mid s})\uM(x_t), & x_s\in\gM,
    \end{cases}
\end{equation}
where the first line uses the notation from \eqref{eq:rt_def}.

\begin{proposition}[Consistency of the forward kernel]
\label{prop:two_time_consistency}
    For the kernel in \eqref{eq:two_time_kernel}, the single-token marginal started from a clean token $x_0\in\gV$ is exactly \eqref{eq:forward_marginal_single}, i.e., 
    \begin{equation*}
        p_t(x_t\mid x_0)
        =
        \alpha_t\1\{x_t=x_0\} + (1-\alpha_t)r_t^{x_0}(x_t),
        \qquad x_t\in\gV\cup\gM.
    \end{equation*} 
\end{proposition}

\begin{proof}
    Let $x_0\in\gV$. The only visible state reachable from $x_0$ is $x_0$ itself, so
    \[
        p_t(x_0\mid x_0)=\Pr(x_t=x_0\mid x_0)=\alpha_t.
    \]
    Now fix $k\in\gM$ and use \eqref{eq:two_time_kernel} at an intermediate time $s<t$:
    \begin{align*}
        p_t(k\mid x_0)
        &=
        p_s(x_0\mid x_0)\,p_{t\mid s}(k\mid x_0)
        +
        \sum_{j\in\gM} p_s(j\mid x_0)\,p_{t\mid s}(k\mid j)
        \\
        &=
        \alpha_s(1-\alpha_{t\mid s})r_t^{x_0}(k)
        +
        (1-\alpha_s)
        \sum_{j\in\gM} r_s^{x_0}(j)
        \Big(\beta_{t\mid s}\1\{k=j\} + (1-\beta_{t\mid s})\uM(k)\Big).
    \end{align*}
    The sum over $j$ equals
    \[
        \beta_{t\mid s}r_s^{x_0}(k) + (1-\beta_{t\mid s})\uM(k)
        =
        \beta_t\1\{k=m_{x_0}\} + (1-\beta_t)\uM(k)
        =
        r_t^{x_0}(k),
    \]
    so
    \[
        p_t(k\mid x_0)
        =
        \big(\alpha_s(1-\alpha_{t\mid s}) + 1-\alpha_s\big)r_t^{x_0}(k)
        =
        (1-\alpha_t)r_t^{x_0}(k).
    \]
    Combining the visible and masked parts proves \eqref{eq:forward_marginal_single}.
\end{proof}

\subsection{Backward Kernel}
\label{app:posterior_kernels}

In this subsection, we derive the exact conditional backward kernel associated with the forward process in Sec.~\ref{app:forward_process}. This closed-form reverse-time expression is the basis for both the training objective and the inference procedure used in the main text.

The forward kernel \eqref{eq:two_time_kernel} yields a closed-form conditional backward kernel as summarized by the following theorem.

\begin{theorem}[Conditional backward kernel]
\label{thm:conditional_backward_kernel}
Fix a clean token $x_0\in\gV$ and times $0\le s<t\le 1$. Then the conditional backward kernel is exactly \eqref{eq:two_time_backward_main}. In particular, when $x_t=k\in\gM$,
\begin{equation}
\label{eq:backward_kernel_masks_app}
    p_{s\mid t}(x_s\mid k,x_0)
    =
    \begin{cases}
        \dfrac{\alpha_s-\alpha_t}{1-\alpha_t}, & x_s=x_0,\\[3mm]
        \dfrac{1-\alpha_s}{1-\alpha_t}
        \dfrac{r_s^{x_0}(x_s)\bigl[\beta_{t\mid s}\1\{k=x_s\}+(1-\beta_{t\mid s})\uM(k)\bigr]}{r_t^{x_0}(k)}, & x_s\in\gM,\\[3mm]
        0, & \text{otherwise.}
    \end{cases}
\end{equation}
If $x_t=x_0$, then $p_{s\mid t}(x_s\mid x_0,x_0)=\1\{x_s=x_0\}$.
\end{theorem}

\begin{proof}
The case $x_t=x_0$ is immediate because the only visible state reachable from $x_0$ is $x_0$ itself. Now fix $k\in\gM$. By Bayes' rule,
\[
    p_{s\mid t}(x_s\mid k,x_0)
    =
    \frac{p_s(x_s\mid x_0)p_{t\mid s}(k\mid x_s)}{p_t(k\mid x_0)}.
\]
For $x_s=x_0$, use $p_s(x_0\mid x_0)=\alpha_s$, $p_{t\mid s}(k\mid x_0)=(1-\alpha_{t\mid s})r_t^{x_0}(k)$, and $p_t(k\mid x_0)=(1-\alpha_t)r_t^{x_0}(k)$ to obtain
\[
    p_{s\mid t}(x_0\mid k,x_0)
    =
    \frac{\alpha_s(1-\alpha_{t\mid s})}{1-\alpha_t}
    =
    \frac{\alpha_s-\alpha_t}{1-\alpha_t}.
\]
For $x_s=j\in\gM$, use $p_s(j\mid x_0)=(1-\alpha_s)r_s^{x_0}(j)$ and the masked branch of \eqref{eq:two_time_kernel} to get
\[
    p_{s\mid t}(j\mid k,x_0)
    =
    \frac{(1-\alpha_s)r_s^{x_0}(j)\bigl[\beta_{t\mid s}\1\{k=j\}+(1-\beta_{t\mid s})\uM(k)\bigr]}{(1-\alpha_t)r_t^{x_0}(k)},
\]
which is exactly \eqref{eq:backward_kernel_masks_app}. No other states have non-zero probability.
\end{proof}

\subsection{CTMC Formulation}
\label{app:ctmc_formulation}

The forward process is a time-inhomogeneous CTMC on the finite state space $\gV\cup\gM$. Writing $p_t$ for a generic marginal distribution on this state space, the forward Kolmogorov equation takes the form
\begin{equation}
\label{eq:forward_kolmogorov}
    \partial_t p_t = \mR_t p_t,
\end{equation}
where the generator $\mR_t$ satisfies
\begin{enumerate}
    \item $\mR_t$ has nonnegative off-diagonal entries,
    \item $\mR_t$ has diagonal entries chosen so that every column sums to zero.
\end{enumerate}

\begin{proposition}[Forward rate matrix]
\label{prop:forward_rates}
Assume $\alpha_t$ and $\beta_t$ are differentiable on $(0,1)$. Then the off-diagonal entries of the generator in \eqref{eq:forward_kolmogorov} are
\begin{equation}
\label{eq:forward_rates_app}
    R_t(k,i)= -\frac{\dot \alpha_t}{\alpha_t}r_t^i(k),
    \qquad i\in\gV,\ k\in\gM,
\end{equation}
between clean and masked states, and
\begin{equation}
\label{eq:forward_rates_masks_app}
    R_t(j,k)= -\frac{\dot \beta_t}{M\beta_t},
    \qquad j,k\in\gM,\ j\neq k,
\end{equation}
between masked states, and all other off-diagonal entries are zero.
\end{proposition}

\begin{proof}
Fix $i\in\gV$. For $\eps\to 0$, we have
\[
    \alpha_{t+\eps\mid t}
    =
    \frac{\alpha_{t+\eps}}{\alpha_t}
    =
    1 + \frac{\dot \alpha_t}{\alpha_t}\eps + o(\eps).
\]
Hence, for any $k\in\gM$,
\[
    p_{t+\eps\mid t}(k\mid i)
    =
    (1-\alpha_{t+\eps\mid t})r_{t+\eps}^i(k)
    =
    -\frac{\dot \alpha_t}{\alpha_t}r_t^i(k)\eps + o(\eps),
\]
which yields \eqref{eq:forward_rates_app}. Likewise, we also have
\[
    \beta_{t+\eps\mid t}
    =
    \frac{\beta_{t+\eps}}{\beta_t}
    =
    1 + \frac{\dot \beta_t}{\beta_t}\eps + o(\eps),
\]
so for distinct $j,k\in\gM$,
\[
    p_{t+\eps\mid t}(j\mid k)
    =
    \frac{1-\beta_{t+\eps\mid t}}{M}
    =
    -\frac{\dot \beta_t}{M\beta_t}\eps + o(\eps),
\]
which gives \eqref{eq:forward_rates_masks_app}. No other off-diagonal transitions are allowed by \eqref{eq:two_time_kernel}.
\end{proof}

For a single-token prior $p_0$ on $\gV$, let
\[
    p_t(y)=\sum_{x_0\in\gV} p_0(x_0)p_t(y\mid x_0),
    \qquad y\in\gV\cup\gM,
\]
be the corresponding marginal at time $t$. The concrete score is the marginal ratio
\begin{equation}
\label{eq:score_definition_app}
    s_t(y)_{y'}:=\frac{p_t(y')}{p_t(y)},
    \qquad p_t(y)>0.
\end{equation}

\begin{proposition}[Conditional ratios]
\label{prop:exact_scores}
Fix a clean token $x_0\in\gV$ and a current masked state $k\in\gM$. Then
\begin{equation}
\label{eq:exact_clean_score}
    \frac{p_t(i\mid x_0)}{p_t(k\mid x_0)}
    =
    \frac{\alpha_t}{(1-\alpha_t)r_t^{x_0}(k)}\,\1\{i=x_0\},
    \qquad i\in\gV,
\end{equation}
and, for $j\in\gM\setminus\{k\}$,
\begin{equation}
\label{eq:exact_mask_score}
    \frac{p_t(j\mid x_0)}{p_t(k\mid x_0)}
    =
    \frac{r_t^{x_0}(j)}{r_t^{x_0}(k)}
    :=
    \psi_t^{x_0}(j\mid k).
\end{equation}
If the current state is visible, then it must equal $x_0$, and every off-diagonal incoming rate into that visible state is zero.
\end{proposition}

\begin{proof}
For $i\in\gV$, \eqref{eq:forward_marginal_single} gives
\[
    p_t(i\mid x_0)=\alpha_t\1\{i=x_0\},
    \qquad
    p_t(k\mid x_0)=(1-\alpha_t)r_t^{x_0}(k),
\]
which proves \eqref{eq:exact_clean_score}. Likewise, for $j\in\gM\setminus\{k\}$,
\[
    \frac{p_t(j\mid x_0)}{p_t(k\mid x_0)}
    =
    \frac{(1-\alpha_t)r_t^{x_0}(j)}{(1-\alpha_t)r_t^{x_0}(k)}
    =
    \psi_t^{x_0}(j\mid k).
\]
The final statement follows from \eqref{eq:forward_rates_app} and \eqref{eq:forward_rates_masks_app}.
\end{proof}

\begin{proposition}[Marginal equivalence to a designated-entry CTMC]
\label{prop:entry_mixing_equivalence}
Consider the time-inhomogeneous CTMC on $\gV\cup\gM$ with off-diagonal rate matrix
\begin{equation}
\label{eq:entry_mixing_rates_app}
    \widetilde R_t(z,y)
    =
    \begin{cases}
        \lambda_t, & y\in\gV,\ z=m_y,\\[1mm]
        \gamma_t/M, & y\in\gM,\ z\in\gM,\ z\neq y,\\[1mm]
        0, & \text{otherwise,}
    \end{cases}
\end{equation}
where
\[
    \lambda_t=-\frac{\dot\alpha_t}{\alpha_t},
    \qquad
    \gamma_t
    =
    \frac{-\dot\alpha_t(1-\beta_t)-(1-\alpha_t)\dot\beta_t}{(1-\alpha_t)\beta_t}.
\]
Then, started from $x_0\in\gV$, this CTMC has one-time marginal
\[
    \tilde p_t(z\mid x_0)=p_t(z\mid x_0),
    \qquad z\in\gV\cup\gM,
\]
where $p_t(\cdot\mid x_0)$ is the forward law in \eqref{eq:forward_marginal_single}. In general, however, its two-time transition kernel does not coincide with \eqref{eq:two_time_kernel}; the equivalence holds only at the level of one-time marginals.
\end{proposition}

\begin{proof}
Let $\tilde p_t(\cdot\mid x_0)$ denote the law of the CTMC generated by \eqref{eq:entry_mixing_rates_app}, started from $x_0$. We verify that the target marginal
\[
    p_t(z\mid x_0)
    =
    \alpha_t\1\{z=x_0\}+(1-\alpha_t)r_t^{x_0}(z)
\]
solves the corresponding forward equation with the same initial condition.

Since the only visible state reachable from $x_0$ is $x_0$ itself, the forward equation for the visible component is
\[
    \partial_t \tilde p_t(x_0\mid x_0)
    =
    \widetilde R_t(x_0,x_0)\tilde p_t(x_0\mid x_0)
    =
    -\lambda_t\tilde p_t(x_0\mid x_0).
\]
For the candidate marginal $p_t(x_0\mid x_0)=\alpha_t$, we therefore have
\[
    \partial_t p_t(x_0\mid x_0)=\dot\alpha_t=-\lambda_t\alpha_t=-\lambda_t p_t(x_0\mid x_0),
\]
using $\lambda_t=-\dot\alpha_t/\alpha_t$. Also $p_0(x_0\mid x_0)=1$.

Now fix $k\in\gM$. By \eqref{eq:entry_mixing_rates_app},
\[
\begin{aligned}
    \partial_t \tilde p_t(k\mid x_0)
    &=
    \widetilde R_t(k,x_0)\tilde p_t(x_0\mid x_0)
    +
    \sum_{j\in\gM\setminus\{k\}}\widetilde R_t(k,j)\tilde p_t(j\mid x_0)
    +
    \widetilde R_t(k,k)\tilde p_t(k\mid x_0)\\
    &=
    \lambda_t\tilde p_t(x_0\mid x_0)\1\{k=m_{x_0}\}
    + \frac{\gamma_t}{M}\sum_{j\in\gM\setminus\{k\}}\tilde p_t(j\mid x_0)
    - \frac{M-1}{M}\gamma_t\tilde p_t(k\mid x_0)\\
    &=
    \lambda_t\tilde p_t(x_0\mid x_0)\1\{k=m_{x_0}\}
    + \frac{\gamma_t}{M}\bigl(1-\tilde p_t(x_0\mid x_0)\bigr)
    - \gamma_t\tilde p_t(k\mid x_0),
\end{aligned}
\]
where we used $\sum_{j\in\gM}\tilde p_t(j\mid x_0)=1-\tilde p_t(x_0\mid x_0)$.

For the candidate marginal,
\[
    p_t(k\mid x_0)
    =
    (1-\alpha_t)\Bigl[\beta_t\1\{k=m_{x_0}\}+(1-\beta_t)\frac1M\Bigr]
    =
    (1-\alpha_t)\beta_t\1\{k=m_{x_0}\}+\frac{(1-\alpha_t)(1-\beta_t)}{M}.
\]
Differentiating yields
\[
    \partial_t p_t(k\mid x_0)
    =
    \bigl[-\dot\alpha_t\beta_t+(1-\alpha_t)\dot\beta_t\bigr]\1\{k=m_{x_0}\}
    +
    \frac{-\dot\alpha_t(1-\beta_t)-(1-\alpha_t)\dot\beta_t}{M}.
\]
On the other hand, substituting $p_t(x_0\mid x_0)=\alpha_t$ and the above expression for $p_t(k\mid x_0)$ into the mask forward equation gives
\[
\begin{aligned}
    &\lambda_t p_t(x_0\mid x_0)\1\{k=m_{x_0}\}
    + \frac{\gamma_t}{M}\bigl(1-p_t(x_0\mid x_0)\bigr)
    - \gamma_t p_t(k\mid x_0)\\
    &=
    \bigl[\lambda_t\alpha_t-\gamma_t(1-\alpha_t)\beta_t\bigr]\1\{k=m_{x_0}\}
    + \frac{\gamma_t(1-\alpha_t)\beta_t}{M}.
\end{aligned}
\]
Using $\lambda_t\alpha_t=-\dot\alpha_t$ and
\[
    \gamma_t(1-\alpha_t)\beta_t
    =
    -\dot\alpha_t(1-\beta_t)-(1-\alpha_t)\dot\beta_t,
\]
this becomes exactly
\[
    \bigl[-\dot\alpha_t\beta_t+(1-\alpha_t)\dot\beta_t\bigr]\1\{k=m_{x_0}\}
    +
    \frac{-\dot\alpha_t(1-\beta_t)-(1-\alpha_t)\dot\beta_t}{M}
    =
    \partial_t p_t(k\mid x_0).
\]
Also $p_0(k\mid x_0)=0$ for every $k\in\gM$. Hence $p_t(\cdot\mid x_0)$ and $\tilde p_t(\cdot\mid x_0)$ solve the same finite-state Kolmogorov forward equation with the same initial condition, so they coincide for every $t$.
\end{proof}

\begin{remark}
    \label{rem:entry_mixing_no_contradiction}
    Prop.~\ref{prop:entry_mixing_equivalence} does not assert that \eqref{eq:entry_mixing_rates_app} and the original generator in Prop.~\ref{prop:forward_rates} induce the same Markov evolution on all initial laws, nor that they have the same two-time transition kernels. The statement is only that, for the fixed clean initial state $x_0\in\gV$, the resulting one-time marginal curve
    \[
        t\longmapsto p_t(\cdot\mid x_0)
    \]
    can be realized by a different time-inhomogeneous CTMC.
    
    This does not contradict uniqueness of the generator associated with a given propagator, \emph{i.e.}, if two processes had the same evolution from every initial distribution. Here the equivalence is much weaker: it is only a realization of the same marginal path started from a clean token $x_0\in\gV$. In particular, the two processes differ at the path level, and this difference is already visible in their two-time kernels.
\end{remark}
    
\begin{remark}[Two-time kernel of the designated-entry process]
    \label{rem:entry_mixing_kernel}
    Although the CTMC \eqref{eq:entry_mixing_rates_app} reproduces the correct one-time marginals, its two-time kernel is less convenient than that of the original process. Write
    \[
        \Lambda_{s,t}:=\int_s^t \lambda_u\,du,
        \qquad
        \Gamma_{s,t}:=\int_s^t \gamma_u\,du,
        \qquad 0\le s\le t\le 1.
    \]
    Since the process can leave $x_0$ only once and never return,
    \[
        \tilde p_{t\mid s}(x_0\mid x_0)=e^{-\Lambda_{s,t}},
        \qquad
        \tilde p_{t\mid s}(x_0\mid k)=0
        \quad\text{for }k\in\gM.
    \]
    Inside the mask sector, the chain refreshes to a uniformly chosen mask at total rate $\gamma_t$, so for $j,k\in\gM$,
    \[
        \tilde p_{t\mid s}(j\mid k)
        =
        e^{-\Gamma_{s,t}}\1\{j=k\}
        +
        \bigl(1-e^{-\Gamma_{s,t}}\bigr)\frac1M.
    \]
    The nontrivial part is the kernel from $x_0$ into the mask sector. To reach $j\in\gM$ at time $t$, the chain must first jump from $x_0$ to $m_{x_0}$ at some intermediate time $u\in[s,t]$, and then evolve within $\gM$ from $u$ to $t$. Therefore
    \[
        \tilde p_{t\mid s}(j\mid x_0)
        =
        \int_s^t
        \lambda_u e^{-\Lambda_{s,u}}
        \left[
            e^{-\Gamma_{u,t}}\1\{j=m_{x_0}\}
            +
            \bigl(1-e^{-\Gamma_{u,t}}\bigr)\frac1M
        \right]du.
    \]
    
    Thus, unlike the original kernel \eqref{eq:two_time_kernel}, the transition law of the designated-entry process is not parameterized solely by endpoint schedule values such as $\alpha_s,\alpha_t,\beta_s,\beta_t$. It involves an additional integration over the unknown entry time $u$ into the mask sector. As a consequence, the corresponding backward posterior kernel $\tilde p_{s|t}(\cdot\mid k,x_0)$ obtained by Bayes' rule is also less explicit. This makes the designated-entry realization conceptually useful for understanding the one-time marginal structure, but less attractive for training and sampling procedures that require clean two-time or reverse-time formulas.
\end{remark}

\subsection{KL Objective}
\label{app:kl_objective}

The diffusion-weighted denoising score-entropy objective can be written directly in terms of the forward generator and the concrete score. We first record the generic finite-state path-space KL bound and then specialize it to the present single-token process.
\begin{proposition}[Marginal score identity]
    \label{prop:marginal_score_identity}
    For any states $y,y'\in\gV\cup\gM$ with $p_t(y)>0$ and $p_t(y\mid x_0)>0$ on the support of $p_0(x_0\mid y)$,
    \begin{equation}
    \label{eq:marginal_score_identity}
        \frac{p_t(y')}{p_t(y)}
        =
        \sum_{x_0\in\gV} p_0(x_0\mid y)\frac{p_t(y'\mid x_0)}{p_t(y\mid x_0)}.
    \end{equation}
\end{proposition}
    
\begin{proof}
    Using Bayes' rule,
    \[
        p_0(x_0\mid y)=\frac{p_0(x_0)p_t(y\mid x_0)}{p_t(y)}.
    \]
    Therefore
    \[
        \begin{aligned}
            &\sum_{x_0\in\gV} p_0(x_0\mid y)\frac{p_t(y'\mid x_0)}{p_t(y\mid x_0)}
            =
            \sum_{x_0\in\gV}\frac{p_0(x_0)p_t(y\mid x_0)}{p_t(y)}\frac{p_t(y'\mid x_0)}{p_t(y\mid x_0)}\\
            =&
            \frac{1}{p_t(y)}\sum_{x_0\in\gV} p_0(x_0)p_t(y'\mid x_0)
            =
            \frac{p_t(y')}{p_t(y)}.
        \end{aligned}
    \]
\end{proof}
    
\begin{theorem}[{Path-space KL upper bound~\cite[][Theorem~3.3]{ren2025discrete}}]
    \label{thm:path_space_upper_bound}
    Consider a time-inhomogeneous CTMC on a finite state space $\gY$ with forward generator $R_t$ in the column-sum-zero convention, forward kernel $p_t(y\mid y_0)$, and initial distribution $p_0$ on $\gY$. We denote the path measure of this jump process as $\P$ and its counterpart of the approximate backward process with the true score function $s_t(y)$ replaced by the neural network $\theta$-approximated score function $s_{t,\theta}(y)$ as $\hat \P$.
    Under the path-space change-of-measure theorem,
    \begin{equation}
    \label{eq:path_space_upper_bound_app}
        D_{\mathrm{KL}}(\P \| \hat \P)
        \leq
        \int_0^1
        \gJ_t(\theta;p_0)\,\dif t,
    \end{equation}
    where
    \begin{equation}
    \label{eq:general_single_token_objective_app}
    \begin{aligned}
        &\gJ_t(\theta;p_0):=
        \\
        &
        \E_{y_0\sim p_0,\,y\sim p_t(\cdot\mid y_0)}
        \left[
            \sum_{y'\neq y} R_t(y,y')
            \Bigg(
                \ s_{t,\theta}(y)_{y'}
                -
                \frac{p_t(y'\mid y_0)}{p_t(y\mid y_0)}\log s_{t,\theta}(y)_{y'} 
                +
                K\!\left(\frac{p_t(y'\mid y_0)}{p_t(y\mid y_0)}\right)
            \Bigg)
        \right],
    \end{aligned}
    \end{equation}
    with $K(a)=a\log a-a$.
\end{theorem}

\begin{proof}
    By \citet[][Theorem~3.3]{ren2025discrete}, the path-space change-of-measure theorem gives
    \[
        D_{\mathrm{KL}}(\P \| \hat \P)
        =
        \int_0^1 \widetilde{\gJ}_t(\theta;p_0)\,\dif t,
    \]
    where
    \[
        \widetilde{\gJ}_t(\theta;p_0)
        :=
        \E_{y\sim p_t}
        \left[
            \sum_{y'\neq y} R_t(y,y')
            \left(
                s_{t,\theta}(y)_{y'}
                -
                \frac{p_t(y')}{p_t(y)}\log s_{t,\theta}(y)_{y'}
                +
                K\!\left(\frac{p_t(y')}{p_t(y)}\right)
            \right)
        \right].
    \]
    Using Proposition~\ref{prop:marginal_score_identity},
    \[
        \frac{p_t(y')}{p_t(y)}
        =
        \E_{y_0\sim p_0(\cdot\mid y)}
        \left[
            \frac{p_t(y'\mid y_0)}{p_t(y\mid y_0)}
        \right].
    \]
    Since $s_{t,\theta}(y)_{y'}$ depends only on $(y,t)$, the first two terms may be rewritten exactly under the conditional expectation:
    \[
        s_{t,\theta}(y)_{y'}
        -
        \frac{p_t(y')}{p_t(y)}\log s_{t,\theta}(y)_{y'}
        =
        \E_{y_0\sim p_0(\cdot\mid y)}
        \left[
            s_{t,\theta}(y)_{y'}
            -
            \frac{p_t(y'\mid y_0)}{p_t(y\mid y_0)}\log s_{t,\theta}(y)_{y'}
        \right].
    \]
    Moreover, since $K(a)=a\log a-a$ is convex, Jensen's inequality yields
    \[
        K\!\left(\frac{p_t(y')}{p_t(y)}\right)
        \le
        \E_{y_0\sim p_0(\cdot\mid y)}
        \left[
            K\!\left(\frac{p_t(y'\mid y_0)}{p_t(y\mid y_0)}\right)
        \right].
    \]
    Substituting these two relations into $\widetilde{\gJ}_t(\theta;p_0)$ and then using the tower property to rewrite
    \[
        \E_{y\sim p_t}\E_{y_0\sim p_0(\cdot\mid y)}[\cdot]
        =
        \E_{y_0\sim p_0,\,y\sim p_t(\cdot\mid y_0)}[\cdot],
    \]
    we obtain
    \[
        \widetilde{\gJ}_t(\theta;p_0)
        \le
        \gJ_t(\theta;p_0),
    \]
    where $\gJ_t(\theta;p_0)$ is exactly \eqref{eq:general_single_token_objective_app}. Therefore
    \[
        D_{\mathrm{KL}}(\P \| \hat \P)
        \le
        \int_0^1 \gJ_t(\theta;p_0)\,\dif t.
    \]
    This is the time-local integrand written in the column-sum-zero convention used here. We also refer readers to \citet[][Section~4.2]{ren2025unified} for a related discussion. 
\end{proof}
    
\begin{theorem}[{KL upper bound~\cite[][Corollary~3.4]{ren2025discrete}}]
    Let $\hat p_0^\theta$ denote the time-$0$ marginal produced by an approximate backward process whose terminal distribution matches the exact terminal law $p_1$. Then we have 
    \begin{equation}
        \label{eq:elbo_app}
            D_{\mathrm{KL}}(p_0\,\|\,\hat p_0^\theta)
            \le
            \int_0^1
            \gJ_t(\theta;p_0)\,\dif t,
    \end{equation}
\end{theorem}
    
\begin{proof}
    Equation \eqref{eq:path_space_upper_bound_app} implies \eqref{eq:elbo_app} by data processing, since $p_0$ and $\hat p_0^\theta$ are the time-$0$ marginals of $\P$ and $\hat \P$, respectively. We also refer to \citet[][Theorem~3.6]{lou2023discrete} for similar results.
\end{proof}

\begin{proposition}[Model-induced scores]
\label{prop:model_scores}
For a current masked state $k\in\gM$, define the model-induced score by replacing the unknown posterior $p_0(\cdot\mid k)$ in \eqref{eq:marginal_score_identity} with the model prediction $x_{t,\theta}(\cdot\mid k)$. Then, for $i\in\gV$,
\begin{equation}
\label{eq:model_clean_score_app}
    s_{t,\theta}(k)_i
    :=
    \sum_{x_0\in\gV} x_{t,\theta}(x_0\mid k)\frac{p_t(i\mid x_0)}{p_t(k\mid x_0)}
    =
    \frac{\alpha_t}{(1-\alpha_t)r_t^i(k)}x_{t,\theta}(i\mid k),
\end{equation}
and, for $j\in\gM\setminus\{k\}$,
\begin{equation}
\label{eq:model_mask_score_app}
    s_{t,\theta}(k)_j
    :=
    \sum_{x_0\in\gV} x_{t,\theta}(x_0\mid k)\frac{p_t(j\mid x_0)}{p_t(k\mid x_0)}
    =
    \psi_{t,\theta}(j\mid k)
    :=
    \sum_{x_0\in\gV} x_{t,\theta}(x_0\mid k)\psi_t^{x_0}(j\mid k).
\end{equation}
Moreover, if $m_{t,\theta}(\cdot\mid k)=\mT x_{t,\theta}(\cdot\mid k)$, then
\begin{equation}
\label{eq:psi_theta_simplified_app}
    \psi_{t,\theta}(j\mid k)
    =
    1
    -
    \frac{\beta_t}{\beta_t+(1-\beta_t)/M}m_{t,\theta}(k\mid k)
    +
    \frac{M\beta_t}{1-\beta_t}m_{t,\theta}(j\mid k).
\end{equation}
\end{proposition}

\begin{proof}
The definition is exactly the marginal-score identity \eqref{eq:marginal_score_identity} with the posterior $p_0(\cdot\mid k)$ replaced by its model approximation. Using \cref{prop:exact_scores},
\[
    s_{t,\theta}(k)_i
    =
    \sum_{x_0\in\gV} x_{t,\theta}(x_0\mid k)
    \frac{\alpha_t}{(1-\alpha_t)r_t^{x_0}(k)}\1\{i=x_0\}
    =
    \frac{\alpha_t}{(1-\alpha_t)r_t^i(k)}x_{t,\theta}(i\mid k),
\]
which proves \eqref{eq:model_clean_score_app}. Likewise,
\[
    s_{t,\theta}(k)_j
    =
    \sum_{x_0\in\gV} x_{t,\theta}(x_0\mid k)\psi_t^{x_0}(j\mid k),
\]
which is exactly \eqref{eq:model_mask_score_app}.

For the simplified form, note that $\psi_t^{x_0}(j\mid k)$ depends on $x_0$ only through its designated mask $m_{x_0}$. Writing the posterior mass on mask classes as $m_{t,\theta}(\cdot\mid k)=\mT x_{t,\theta}(\cdot\mid k)$, we obtain
\[
    \psi_{t,\theta}(j\mid k)
    =
    \sum_{m\in\gM} m_{t,\theta}(m\mid k)
    \frac{\beta_t\1\{j=m\}+(1-\beta_t)/M}{\beta_t\1\{k=m\}+(1-\beta_t)/M}.
\]
Because $j\neq k$, the terms $m=k$, $m=j$, and $m\notin\{j,k\}$ contribute
\[
    m_{t,\theta}(k\mid k)\frac{(1-\beta_t)/M}{\beta_t+(1-\beta_t)/M},
\]
\[
    m_{t,\theta}(j\mid k)\frac{\beta_t+(1-\beta_t)/M}{(1-\beta_t)/M},
\]
and
\[
    \sum_{m\in\gM\setminus\{j,k\}} m_{t,\theta}(m\mid k)
    =
    1-m_{t,\theta}(k\mid k)-m_{t,\theta}(j\mid k),
\]
respectively. Summing these three contributions and simplifying gives
\[
    \psi_{t,\theta}(j\mid k)
    =
    1
    -
    \frac{\beta_t}{\beta_t+(1-\beta_t)/M}m_{t,\theta}(k\mid k)
    +
    \frac{M\beta_t}{1-\beta_t}m_{t,\theta}(j\mid k),
\]
which is \eqref{eq:psi_theta_simplified_app}.
\end{proof}

\begin{theorem}[Single-token closed form]
\label{thm:single_token_closed_form}
Fix a clean token $x_0\in\gV$ and a current masked state $k\in\gM$. For $j\in\gM\setminus\{k\}$, define
\[
    \psi_t^{x_0}(j\mid k):=\frac{r_t^{x_0}(j)}{r_t^{x_0}(k)},
    \qquad
    \psi_{t,\theta}(j\mid k):=\sum_{a\in\gV} x_{t,\theta}(a\mid k)\,\psi_t^{a}(j\mid k).
\]
Then substituting \cref{prop:forward_rates,prop:exact_scores,prop:model_scores} into \eqref{eq:general_single_token_objective_app} yields the single-token loss density
\begin{equation}
\label{eq:single_token_loss_app}
\begin{aligned}
    \ell_t(\theta;x_0,k)
    :=
    &-
    \frac{\dot \alpha_t}{1-\alpha_t}
    \log\frac{1}{x_{t,\theta}(x_0\mid k)}
    \\
    &-
    \frac{\dot \beta_t}{M\beta_t}
    \sum_{j\in\gM\setminus\{k\}}
    \left[
        \psi_t^{x_0}(j\mid k)
        \log\frac{\psi_t^{x_0}(j\mid k)}{\psi_{t,\theta}(j\mid k)}
        +
        \psi_{t,\theta}(j\mid k)
        -
        \psi_t^{x_0}(j\mid k)
    \right].
\end{aligned}
\end{equation}
If the current state is visible, the contribution is zero.
\end{theorem}

\begin{proof}
Fix a masked current state $k\in\gM$. Then \eqref{eq:general_single_token_objective_app} becomes
\[
\begin{aligned}
    \ell_t(\theta;x_0,k)
    ={}&
    \sum_{i\in\gV} R_t(k,i)
    \left(
        s_{t,\theta}(k)_i
        -
        \frac{p_t(i\mid x_0)}{p_t(k\mid x_0)}\log s_{t,\theta}(k)_i
        +
        K\!\left(\frac{p_t(i\mid x_0)}{p_t(k\mid x_0)}\right)
    \right)
    \\
    &+
    \sum_{j\in\gM\setminus\{k\}} R_t(k,j)
    \left(
        s_{t,\theta}(k)_j
        -
        \frac{p_t(j\mid x_0)}{p_t(k\mid x_0)}\log s_{t,\theta}(k)_j
        +
        K\!\left(\frac{p_t(j\mid x_0)}{p_t(k\mid x_0)}\right)
    \right).
\end{aligned}
\]
We evaluate the two sums separately.

For the clean states, \cref{prop:forward_rates,prop:exact_scores,prop:model_scores} gives
\[
    R_t(k,i)= -\frac{\dot \alpha_t}{\alpha_t}r_t^i(k),
    \qquad
    \frac{p_t(i\mid x_0)}{p_t(k\mid x_0)}
    =
    \frac{\alpha_t}{(1-\alpha_t)r_t^{x_0}(k)}\1\{i=x_0\},
\]
\[
    s_{t,\theta}(k)_i
    =
    \frac{\alpha_t}{(1-\alpha_t)r_t^i(k)}x_{t,\theta}(i\mid k).
\]
Therefore the clean contribution equals
\begin{align*}
    &\sum_{i\in\gV}R_t(k,i)
    \left(
        s_{t,\theta}(k)_i
        -
        \frac{p_t(i\mid x_0)}{p_t(k\mid x_0)}\log s_{t,\theta}(k)_i
        +
        K\!\left(\frac{p_t(i\mid x_0)}{p_t(k\mid x_0)}\right)
    \right)
    \\
    &=
    -\frac{\dot \alpha_t}{1-\alpha_t}\sum_{i\in\gV}x_{t,\theta}(i\mid k)
    +
    \frac{\dot \alpha_t}{1-\alpha_t}
    \log\left(\frac{\alpha_t}{(1-\alpha_t)r_t^{x_0}(k)}x_{t,\theta}(x_0\mid k)\right)
    \\
    &\qquad
    -
    \frac{\dot \alpha_t}{1-\alpha_t}
    \left[
        \log\frac{\alpha_t}{(1-\alpha_t)r_t^{x_0}(k)} - 1
    \right].
\end{align*}
Since $\sum_{i\in\gV}x_{t,\theta}(i\mid k)=1$, the terms involving $\alpha_t/((1-\alpha_t)r_t^{x_0}(k))$ cancel exactly, leaving
\[
    -\frac{\dot \alpha_t}{1-\alpha_t}\log\frac{1}{x_{t,\theta}(x_0\mid k)}.
\]

For the mask states, \cref{prop:forward_rates,prop:exact_scores,prop:model_scores} gives
\[
    R_t(k,j)= -\frac{\dot \beta_t}{M\beta_t},
    \qquad
    \frac{p_t(j\mid x_0)}{p_t(k\mid x_0)}=\psi_t^{x_0}(j\mid k),
    \qquad
    s_{t,\theta}(k)_j=\psi_{t,\theta}(j\mid k).
\]
Hence the mask contribution is
\begin{align*}
    &\sum_{j\in\gM\setminus\{k\}}
    R_t(k,j)
    \left(
        s_{t,\theta}(k)_j
        -
        \frac{p_t(j\mid x_0)}{p_t(k\mid x_0)}\log s_{t,\theta}(k)_j
        +
        K\!\left(\frac{p_t(j\mid x_0)}{p_t(k\mid x_0)}\right)
    \right)
    \\
    &=
    -\frac{\dot \beta_t}{M\beta_t}
    \sum_{j\in\gM\setminus\{k\}}
    \left[
        \psi_t^{x_0}(j\mid k)
        \log\frac{\psi_t^{x_0}(j\mid k)}{\psi_{t,\theta}(j\mid k)}
        +
        \psi_{t,\theta}(j\mid k)
        -
        \psi_t^{x_0}(j\mid k)
    \right].
\end{align*}
Combining the clean and mask parts proves \eqref{eq:single_token_loss_app}. If the current state is visible, then every off-diagonal multiplier is zero, so the contribution vanishes.
\end{proof}

\begin{proposition}[Sequence-local ratio identity]
\label{prop:sequence_local_ratio_identity}
Let $p_t$ denote the sequence marginal induced by $p_0$ and the coordinatewise forward process \eqref{eq:prelim_factorization}. For $\vx=(x^1,\dots,x^L)\in(\gV\cup\gM)^L$, an index $\ell\in\{1,\dots,L\}$, and a state $z\in\gV\cup\gM$, write $\vx^{\ell\leftarrow z}$ for the sequence obtained by replacing the $\ell$th coordinate of $\vx$ by $z$. Whenever $p_t(\vx\mid \vx_0)>0$,
\begin{equation}
\label{eq:sequence_local_ratio_app}
    \frac{p_t(\vx^{\ell\leftarrow z}\mid \vx_0)}{p_t(\vx\mid \vx_0)}
    =
    \frac{p_t(z\mid x_0^\ell)}{p_t(x^\ell\mid x_0^\ell)}.
\end{equation}
If moreover $p_t(\vx)>0$, then
\begin{equation}
\label{eq:sequence_local_marginal_ratio_app}
    \frac{p_t(\vx^{\ell\leftarrow z})}{p_t(\vx)}
    =
    \sum_{a\in\gV}
    p_0(x_0^\ell=a\mid \vx)
    \frac{p_t(z\mid a)}{p_t(x^\ell\mid a)}.
\end{equation}
\end{proposition}

\begin{proof}
By \eqref{eq:prelim_factorization},
\[
    p_t(\vx\mid \vx_0)
    =
    \prod_{r=1}^L p_t(x^r\mid x_0^r),
    \qquad
    p_t(\vx^{\ell\leftarrow z}\mid \vx_0)
    =
    p_t(z\mid x_0^\ell)\prod_{r\neq \ell} p_t(x^r\mid x_0^r).
\]
Cancelling the common factors proves \eqref{eq:sequence_local_ratio_app}. For \eqref{eq:sequence_local_marginal_ratio_app}, use Bayes' rule:
\[
    p_0(\vx_0\mid \vx)
    =
    \frac{p_0(\vx_0)p_t(\vx\mid \vx_0)}{p_t(\vx)}.
\]
Therefore,
\begin{align*}
    \sum_{a\in\gV} p_0(x_0^\ell=a\mid \vx)\frac{p_t(z\mid a)}{p_t(x^\ell\mid a)}
    &=
    \sum_{\vx_0\in\gV^L} p_0(\vx_0\mid \vx)
    \frac{p_t(z\mid x_0^\ell)}{p_t(x^\ell\mid x_0^\ell)}
    \\
    &=
    \sum_{\vx_0\in\gV^L}
    \frac{p_0(\vx_0)p_t(\vx\mid \vx_0)}{p_t(\vx)}
    \frac{p_t(z\mid x_0^\ell)}{p_t(x^\ell\mid x_0^\ell)}
    \\
    &=
    \frac{1}{p_t(\vx)}
    \sum_{\vx_0\in\gV^L} p_0(\vx_0)p_t(\vx^{\ell\leftarrow z}\mid \vx_0)
    =
    \frac{p_t(\vx^{\ell\leftarrow z})}{p_t(\vx)},
\end{align*}
where the third step uses \eqref{eq:sequence_local_ratio_app}.
\end{proof}

\begin{corollary}[Sequence objective]
\label{cor:sequence_objective}
Under the coordinatewise factorization \eqref{eq:prelim_factorization}, the sequence-level time-local objective decomposes into the sum of the coordinate losses in \eqref{eq:per_coordinate_loss_main}, even when each predictor $x_{t,\theta}^\ell(\cdot\mid \vx_t)$ is conditioned on the full noisy sequence $\vx_t$. Consequently,
\begin{equation}
\label{eq:sequence_kl_upper_bound_app}
    D_{\mathrm{KL}}(p_0\,\|\,\hat p_0^\theta)
    \le
    \gL_{\mathrm{disc}}(\theta).
\end{equation}
\end{corollary}

\begin{proof}
Let $\gS:=(\gV\cup\gM)^L$ be the sequence state space. Because the forward process factorizes across coordinates, the sequence transition kernel over an infinitesimal interval also factorizes:
\[
    p_{t+\eps\mid t}(\vy\mid \vx)
    =
    \prod_{r=1}^L p_{t+\eps\mid t}(y^r\mid x^r).
\]
Hence transitions that change two or more coordinates have probability $O(\eps^2)$ and therefore zero instantaneous rate. If $\vx$ and $\vy$ differ only at coordinate $\ell$, then
\[
    p_{t+\eps\mid t}(\vy\mid \vx)
    =
    p_{t+\eps\mid t}(y^\ell\mid x^\ell)\prod_{r\neq \ell} p_{t+\eps\mid t}(x^r\mid x^r)
    =
    p_{t+\eps\mid t}(y^\ell\mid x^\ell)+o(\eps),
\]
so the sequence generator is
\begin{equation}
\label{eq:sequence_generator_app}
    \mathcal{R}_t(\vx,\vy)
    =
    \begin{cases}
        R_t(x^\ell,y^\ell), & \text{if $\vx$ and $\vy$ differ at exactly one coordinate $\ell$,}\\[1mm]
        0, & \text{if they differ at two or more coordinates.}
    \end{cases}
\end{equation}
Here, as in the single-token generator, the first argument is the destination state and the second is the source state.

Fix a current sequence $\vx\in\gS$ and a source neighbor $\vy=\vx^{\ell\leftarrow z}$ that differs only at coordinate $\ell$. By Prop.~\ref{prop:sequence_local_ratio_identity},
\begin{equation}
\label{eq:sequence_edge_ratio_app}
    \frac{p_t(\vy\mid \vx_0)}{p_t(\vx\mid \vx_0)}
    =
    \frac{p_t(z\mid x_0^\ell)}{p_t(x^\ell\mid x_0^\ell)}.
\end{equation}
Thus the exact sequence-level marginal ratio along the edge $\vy\to\vx$ depends on $\vx_0$ only through the clean token $x_0^\ell$. Replacing the posterior $p_0(x_0^\ell\mid \vx)$ in \eqref{eq:sequence_local_marginal_ratio_app} with the model prediction $x_{t,\theta}^\ell(\cdot\mid \vx)$ gives the model-induced score
\begin{equation}
\label{eq:sequence_model_score_app}
    s_{t,\theta}(\vx)_{\vy}
    :=
    \sum_{a\in\gV} x_{t,\theta}^\ell(a\mid \vx)
    \frac{p_t(z\mid a)}{p_t(x^\ell\mid a)}.
\end{equation}
This is the same local construction as in Prop.~\ref{prop:model_scores}, now applied with conditioning variable $\vx$. Therefore, if the current coordinate is masked, $x^\ell=k\in\gM$, then for every clean token $i\in\gV$,
\begin{equation}
\label{eq:sequence_clean_score_app}
\begin{aligned}
    s_{t,\theta}(\vx)_{\vx^{\ell\leftarrow i}}
    &=
    \sum_{a\in\gV} x_{t,\theta}^\ell(a\mid \vx)
    \frac{p_t(i\mid a)}{p_t(k\mid a)}
    \\
    &=
    x_{t,\theta}^\ell(i\mid \vx)
    \frac{\alpha_t}{(1-\alpha_t)r_t^i(k)},
\end{aligned}
\end{equation}
where we used Prop.~\ref{prop:exact_scores} and the identity $p_t(i\mid a)=\alpha_t\1\{i=a\}$. Likewise, for every $j\in\gM\setminus\{k\}$,
\begin{equation}
\label{eq:sequence_mask_score_app}
    s_{t,\theta}(\vx)_{\vx^{\ell\leftarrow j}}
    =
    \psi_{t,\theta}^\ell(j\mid \vx).
\end{equation}
If instead $x^\ell\in\gV$, then all off-diagonal incoming rates at coordinate $\ell$ vanish.

Now apply Thm.~\ref{thm:path_space_upper_bound} on the full sequence state space $\gS$. Its time-local integrand is
\begin{equation}
\label{eq:sequence_local_objective_app}
\begin{aligned}
    &\gJ_t(\theta;p_0)\\
    :=&
    \E_{\vx_0\sim p_0,\,\vx\sim p_t(\cdot\mid \vx_0)}
    \Bigg[
        \sum_{\vy\neq \vx} \mathcal{R}_t(\vx,\vy)
        \Bigg(
            s_{t,\theta}(\vx)_{\vy}
            -
            \frac{p_t(\vy\mid \vx_0)}{p_t(\vx\mid \vx_0)}
            \log s_{t,\theta}(\vx)_{\vy}
            +
            K\!\left(\frac{p_t(\vy\mid \vx_0)}{p_t(\vx\mid \vx_0)}\right)
        \Bigg)
    \Bigg].
\end{aligned}
\end{equation}
Using \eqref{eq:sequence_generator_app}, the inner sum contains only single-coordinate neighbors and therefore splits as
\[
    \sum_{\ell=1}^L
    \sum_{z\neq x^\ell}
    R_t(x^\ell,z)
    \left(
        s_{t,\theta}(\vx)_{\vx^{\ell\leftarrow z}}
        -
        \frac{p_t(z\mid x_0^\ell)}{p_t(x^\ell\mid x_0^\ell)}
        \log s_{t,\theta}(\vx)_{\vx^{\ell\leftarrow z}}
        +
        K\!\left(\frac{p_t(z\mid x_0^\ell)}{p_t(x^\ell\mid x_0^\ell)}\right)
    \right),
\]
where \eqref{eq:sequence_edge_ratio_app} was used to replace the sequence ratio by its one-coordinate form. Hence, for each fixed coordinate $\ell$,
\[
    \begin{aligned}
        &\sum_{z\neq x^\ell}
        R_t(x^\ell,z)
        \left(
            s_{t,\theta}(\vx)_{\vx^{\ell\leftarrow z}}
            -
            \frac{p_t(z\mid x_0^\ell)}{p_t(x^\ell\mid x_0^\ell)}
            \log s_{t,\theta}(\vx)_{\vx^{\ell\leftarrow z}}
            +
            K\!\left(\frac{p_t(z\mid x_0^\ell)}{p_t(x^\ell\mid x_0^\ell)}\right)
        \right) \\
        =&
        \1\{x^\ell\in\gM\}\,\ell_t^\ell(\theta;\vx_0,\vx).
    \end{aligned}
\]
Substituting this identity into \eqref{eq:sequence_local_objective_app} yields
\[
    \gJ_t(\theta;p_0)
    =
    \E_{\vx_0\sim p_0,\,\vx_t\sim p_t(\cdot\mid \vx_0)}
    \left[
        \sum_{\ell=1}^L
        \1\{x_t^\ell\in\gM\}
        \,\ell_t^\ell(\theta;\vx_0,\vx_t)
    \right].
\]
Integrating over $t\in[0,1]$ and using that \eqref{eq:sequence_objective_main} averages over $t\sim\operatorname{Unif}(0,1)$ yields exactly the objective \eqref{eq:sequence_objective_main}. Applying Thm.~\ref{thm:path_space_upper_bound} on the sequence state space with this explicit integrand gives \eqref{eq:sequence_kl_upper_bound_app}.
\end{proof}

\section{Theoretical Analysis of Discrete Consistency Distillation}
\label{app:consistency_distillation}

In this appendix, we analyze the theoretical ingredients behind the consistency-distillation scheme proposed in Sec.~\ref{sec:distillation}. We first study structural properties of the shared-Gumbel coupling and then derive the corresponding future-filtration consistency target.

\subsection{Shared-Gumbel Paths}
\label{app:shared_gumbel}

We begin by analyzing the pathwise structure induced by the shared-Gumbel construction. In particular, we show that under the schedule family used in our experiments, the resulting trajectories are highly structured and admit only a small number of jumps.

\begin{proposition}[At-most-two jumps under a power-law schedule]
    \label{prop:mm_gumbel_two_jumps}
    Fix a coordinate $\ell$ and a clean token $x_0^\ell\in\gV$. Assume that
    \[
        \beta_t=\alpha_t^\rho,\qquad t\in[0,1],
    \]
    for some $0<\rho\le 1$. Then for the shared-Gumbel path in
    \eqref{eq:gumbel_coupled_path}, the trajectory $t\mapsto \widetilde x_t^\ell$
    is almost surely piecewise constant with at most two jumps. More precisely,
    there exist random states $z_0,z_1,z_2\in\{x_0^\ell\}\cup\gM$ and times
    $0\le \tau_1\le \tau_2\le 1$ such that
    \[
        \widetilde x_t^\ell=
        \begin{cases}
            z_0, & 0\le t<\tau_1,\\
            z_1, & \tau_1\le t<\tau_2,\\
            z_2, & \tau_2\le t\le 1.
        \end{cases}
    \]
    In the single-mask or uniform-state special case, the path jumps at most once.
\end{proposition}
    
\begin{proof}
    Fix the shared Gumbel variables $\{g_z^\ell\}_{z\in\gV\cup\gM}$. Since
    $p_t(z\mid x_0^\ell)=0$ for all visible states $z\neq x_0^\ell$, only the
    clean token $x_0^\ell$ and the mask states compete in
    \eqref{eq:gumbel_coupled_path}. If $M=1$, there is only one mask state, so
    only the clean branch and that single mask branch can compete; in this case
    the path jumps at most once and the claim is immediate. Hence assume
    $M\ge 2$.

    Among the non-designated masks
    $m\in\gM\setminus\{m_{x_0^\ell}\}$, all log-probabilities are identical at each
    time $t$, so the best such competitor is a single time-independent mask
    \[
        m_*:=\argmax_{m\in\gM\setminus\{m_{x_0^\ell}\}} g_m^\ell.
    \]
    Hence the maximization only involves the three candidates
    $x_0^\ell$, $m_{x_0^\ell}$, and $m_*$. Their scores are
    \[
        \log \alpha_t + g_{x_0^\ell}^\ell,
    \]
    \[
        \log\!\left((1-\alpha_t)\Bigl(\alpha_t^\rho+\frac{1-\alpha_t^\rho}{M}\Bigr)\right)
        + g_{m_{x_0^\ell}}^\ell,
    \]
    and
    \[
        \log\!\left(\frac{(1-\alpha_t)(1-\alpha_t^\rho)}{M}\right)+g_{m_*}^\ell.
    \]
    
    First compare the clean token and the designated mask. Their score difference is
    \[
        \log\!\left(
            \frac{\alpha_t}
            {(1-\alpha_t)\Bigl(\alpha_t^\rho+\frac{1-\alpha_t^\rho}{M}\Bigr)}
        \right)
        +\bigl(g_{x_0^\ell}^\ell-g_{m_{x_0^\ell}}^\ell\bigr).
    \]
    Set
    \[
        R(a):=\frac{a}{(1-a)\Bigl(a^\rho+\frac{1-a^\rho}{M}\Bigr)}.
    \]
    A direct differentiation gives
    \[
        \frac{d}{da}\log R(a)
        =
        \frac{1+(M-1)a^\rho(1-\rho+\rho a)}
        {a(1-a)\bigl(1+(M-1)a^\rho\bigr)}.
    \]
    Since $0<\rho\le 1$, we have $1-\rho+\rho a\ge 0$ for every $a\in(0,1)$, so
    $\frac{d}{da}\log R(a)>0$. Thus $R(a)$ is increasing in $a$. Because
    $\alpha_t$ is nonincreasing in $t$, the clean-versus-designated score
    difference is monotone in $t$, so these two branches can cross at most once.
    
    Next compare the clean token and the best non-designated mask. Their score
    difference is
    \[
        \log\!\left(
            \frac{M\alpha_t}{(1-\alpha_t)(1-\alpha_t^\rho)}
        \right)
        +\left(g_{x_0^\ell}^\ell-g_{m_*}^\ell\right).
    \]
    Since $\alpha_t$ is nonincreasing, while both $1-\alpha_t$ and
    $1-\alpha_t^\rho$ are nondecreasing, this difference is nonincreasing in $t$.
    Hence these two branches also cross at most once.
    
    Finally, compare the designated mask and the best non-designated mask. Their
    score difference is
    \[
        \log\!\left(
            1+\frac{M\alpha_t^\rho}{1-\alpha_t^\rho}
        \right)
        +\left(g_{m_{x_0^\ell}}^\ell-g_{m_*}^\ell\right),
    \]
    which is nonincreasing in $t$ because $\alpha_t^\rho$ is nonincreasing. Hence
    these two mask branches also cross at most once.
    
    Therefore once the clean-token branch loses maximality, it can never become
    maximal again; and after the maximizer has moved to the mask side, the identity
    of the maximizing mask can change at most once. So the path
    $t\mapsto \widetilde x_t^\ell$ has at most two jumps.
    
    In the single-mask or uniform-state special case there is only one mask-side
    competitor, so only the clean-to-mask transition may occur, and the path jumps
    at most once.
\end{proof}

\begin{theorem}[Shared-Gaussian and shared-Gumbel representations of the Duo path]
    \label{thm:duo_gumbel_path}
    Fix a coordinate $\ell$ and a clean token $x_0^\ell\in\gV$. Consider the Duo path
    \begin{equation}
    \label{eq:duo_gaussian_path}
        \widetilde x_t^{\mathrm{duo},\ell}
        :=
        \argmax_{z\in\gV}
        \Bigl\{
            \widetilde\alpha_t\,\1\{z=x_0^\ell\}
            +
            \sqrt{1-\widetilde\alpha_t^2}\,\varepsilon_z^\ell
        \Bigr\},
        \qquad t\in[0,1],
    \end{equation}
    where $(\varepsilon_z^\ell)_{z\in\gV}$ are i.i.d.\ standard Gaussian random variables, sampled once and shared across all times. Let $p_t^{\mathrm{duo}}(\cdot\mid x_0^\ell)$ denote its one-time marginal, so that
    \[
        p_t^{\mathrm{duo}}(z\mid x_0^\ell)
        =
        \alpha_t\,\1\{z=x_0^\ell\}+(1-\alpha_t)\,u(z),
    \]
    with $u$ the uniform distribution on $\gV$. Define also the shared-Gumbel path
    \begin{equation}
    \label{eq:duo_gumbel_path}
        \widetilde x_t^{\mathrm{gum},\ell}
        :=
        \argmax_{z\in\gV}
        \bigl\{
            \log p_t^{\mathrm{duo}}(z\mid x_0^\ell)+g_z^\ell 
        \bigr\},
        \qquad t\in[0,1],
    \end{equation}
    where $(g_z^\ell)_{z\in\gV}$ are i.i.d.\ standard Gumbel random variables, again sampled once and shared across all times. Then
    \[
        (\widetilde x_t^{\mathrm{duo},\ell})_{t\in[0,1]}
        \overset{d}{=}
        (\widetilde x_t^{\mathrm{gum},\ell})_{t\in[0,1]}.
    \]
    \end{theorem}
    
    \begin{proof}
    Fix times $0\le t_1<\cdots<t_m\le 1$ and write
    \[
        p_t:=p_t^{\mathrm{duo}}(x_0^\ell\mid x_0^\ell)=\alpha_t+\frac{1-\alpha_t}{|\gV|}.
    \]
    For the Gaussian path \eqref{eq:duo_gaussian_path}, all non-clean states have scores proportional to $\varepsilon_z^\ell$, so the best competing token
    \[
        \argmax_{z\neq x_0^\ell}\varepsilon_z^\ell
    \]
    does not depend on $t$. Hence $\widetilde x_t^{\mathrm{duo},\ell}$ starts at $x_0^\ell$, jumps at most once, and after leaving $x_0^\ell$ stays forever at a single rival token. Its survival probability is exactly
    \[
        \Pr(\widetilde x_t^{\mathrm{duo},\ell}=x_0^\ell)=p_t,
    \]
    and symmetry shows that, conditional on leaving $x_0^\ell$, the rival token is uniform on $\gV\setminus\{x_0^\ell\}$.

    The same description holds for the Gumbel path \eqref{eq:duo_gumbel_path}. Indeed, all non-clean logits are equal at each time, so the best competing token
    \[
        \argmax_{z\neq x_0^\ell} g_z^\ell
    \]
    is again independent of $t$. Thus $\widetilde x_t^{\mathrm{gum},\ell}$ also starts at $x_0^\ell$, jumps at most once, and then stays forever at a single rival token. Moreover, the Gumbel-max trick gives
    \[
        \Pr(\widetilde x_t^{\mathrm{gum},\ell}=z)=p_t^{\mathrm{duo}}(z\mid x_0^\ell),
    \]
    so its survival probability is again $p_t$, and its post-jump rival token is again uniform on $\gV\setminus\{x_0^\ell\}$.

    Therefore both paths are the same one-jump process: they share the same survival curve $p_t$ and the same uniform law over the rival token after the jump. Consequently, for either choice
    \[
        \star\in\{\mathrm{duo},\mathrm{gum}\},
    \]
    any $r\in\{0,\dots,m-1\}$, and any $j\neq x_0^\ell$,
    \[
        \Pr\bigl(
            \widetilde x_{t_1}^{\star,\ell}=\cdots=\widetilde x_{t_r}^{\star,\ell}=x_0^\ell,\ 
            \widetilde x_{t_{r+1}}^{\star,\ell}=\cdots=\widetilde x_{t_m}^{\star,\ell}=j
        \bigr)
        =
        \frac{p_{t_r}-p_{t_{r+1}}}{|\gV|-1},
    \]
    with the convention $p_{t_0}:=1$, while
    \[
        \Pr\bigl(\widetilde x_{t_1}^{\star,\ell}=\cdots=\widetilde x_{t_m}^{\star,\ell}=x_0^\ell\bigr)
        =
        p_{t_m}.
    \]
    These finite-dimensional distributions agree for the two constructions, proving
    \[
        (\widetilde x_t^{\mathrm{duo},\ell})_{t\in[0,1]}
        \overset{d}{=}
        (\widetilde x_t^{\mathrm{gum},\ell})_{t\in[0,1]}.
    \]
\end{proof}

\subsection{Consistency Distillation}
\label{app:consistency_targets}

For the Gumbel-coupled trajectory in \eqref{eq:gumbel_coupled_path}, let
\[
    \gG_t:=\sigma(\widetilde\vx_u:u\in[t,1]),
    \qquad
    \widetilde\vx_{[t,1]}:=\{\widetilde\vx_u:u\in[t,1]\}.
\]
The derivation below is coordinatewise and does not require the coupled path to be Markov; only the decreasing-filtration property $\gG_t\subseteq \gG_s$ for $s<t$ is used.

\begin{proof}[Proof of Thm.~\ref{thm:consistency_target}]
Fix a coordinate $\ell$ and times $0\le s<t\le 1$. Let $f_s^\ell$ be any simplex-valued, $\gG_s$-measurable random vector, and let $f_t^\ell$ be any simplex-valued, $\gG_t$-measurable predictor. Expanding the KL divergence gives
\[
    \E\bigl[D_{\mathrm{KL}}(f_s^\ell\|f_t^\ell)\bigr]
    =
    \E\left[\sum_{i\in\gV}(f_s^\ell)_i\log (f_s^\ell)_i\right]
    -
    \E\left[\sum_{i\in\gV}(f_s^\ell)_i\log (f_t^\ell)_i\right].
\]
The first term is independent of $f_t^\ell$. Since $f_t^\ell$ is $\gG_t$-measurable,
\[
    \E\left[\sum_{i\in\gV}(f_s^\ell)_i\log (f_t^\ell)_i\right]
    =
    \E\left[\sum_{i\in\gV}\E[(f_s^\ell)_i\mid \gG_t]\log (f_t^\ell)_i\right].
\]
Set
\[
    a:=\E[f_s^\ell\mid \gG_t]\in\Delta^{V-1}.
\]
Therefore minimizing the expected KL divergence over $\gG_t$-measurable $f_t^\ell$ is equivalent, pointwise on $\gG_t$, to minimizing
\[
    -\sum_{i\in\gV} a_i\log q_i
    \qquad\text{over } q\in\Delta^{V-1}.
\]
Using the identity
\[
    -\sum_{i\in\gV} a_i\log q_i
    =
    -\sum_{i\in\gV} a_i\log a_i + D_{\mathrm{KL}}(a\|q),
\]
we see that the unique minimizer is $q=a$. Hence
\[
    f_t^{\ell,\star}
    =
    \E[f_s^\ell\mid \gG_t].
\]

If $f_s^\ell=\E[\ve_{x_0^\ell}\mid \gG_s]$, then $\gG_t\subseteq\gG_s$ for $s<t$, so the tower property yields
\[
    f_t^{\ell,\star}
    =
    \E[\E[\ve_{x_0^\ell}\mid \gG_s]\mid \gG_t]
    =
    \E[\ve_{x_0^\ell}\mid \gG_t],
\]
which is exactly \eqref{eq:consistency_target_main}. The boundary condition follows because $\widetilde\vx_0=\vx_0$, and hence $\E[\ve_{x_0^\ell}\mid \gG_0]=\ve_{x_0^\ell}$.
\end{proof}

\begin{remark}
    If the coupled backward path is Markov with respect to the sequence state, then $x_0^\ell\to \widetilde\vx_t\to \widetilde\vx_{[t,1]}$, so the future-filtration posterior reduces to the one-time posterior:
\[
    \Pr(x_0^\ell=\cdot\mid \widetilde\vx_{[t,1]})
    =
    \Pr(x_0^\ell=\cdot\mid \widetilde\vx_t).
\]
\end{remark}

\section{Experiment Settings and Additional Results}
\label{app:experiments}

In this appendix, we summarize the experimental settings used in Secs.~\ref{sec:methodology} and~\ref{sec:experiments}, and provide additional results that complement the main-text evaluations. We first describe the synthetic conditional-entropy experiment and the training and sampling details, and then present additional Pareto-front comparisons together with representative generation examples.

\subsection{Synthetic Dataset for Conditional Entropy Evaluation}
\label{app:conditional_entropy}

For the synthetic conditional-entropy experiment, we consider clean tokens \(x_0\in\gV\), where the vocabulary has size \(V=100\). The clean-token prior is a Zipf law
\[
p_0(i) = p_{\mathrm{Zipf}}(i)\propto i^{-s},\qquad i\in\gV,
\]
with exponent \(s=1.2\). We evaluate posterior uncertainty along the time grid
\[
t\in\Bigl\{0,\frac1K,\dots,1\Bigr\},\qquad K=24,
\]
and use the convention
$
\alpha_t=1-t.
$
All path constructions in the figure are matched through this same discrete clean-signal level \(\alpha_t\).

The exact uniform-diffusion baseline is the forward process with one-time marginal \(p_t^{\mathrm{uni}}\) in \eqref{eq:uniform_forward_prelim}. We also include the corresponding shared-Gumbel path, defined by
\[
x_t
=
\argmax_{z\in\gV}
\bigl\{
\log p_t^{\mathrm{uni}}(z\mid x_0)+g_z
\bigr\},
\qquad
g_z\stackrel{\mathrm{i.i.d.}}{\sim}\mathrm{Gumbel}(0,1).
\]
This coupling has the same one-time marginal as uniform diffusion while providing a pathwise realization driven by shared auxiliary randomness.

For the exact uniform-diffusion baseline, no auxiliary randomness is present, and \(H(x_0\mid x_t)\) is computed analytically from the closed-form forward marginal. 
For the shared-Gaussian path, we adopt the construction from \citet{sahoo2025diffusion} as follows:
\[
x_t
=
\argmax_{z\in\gV}
\left\{
\widetilde\alpha_t\,\1\{z=x_0\}
+
\sqrt{1-\widetilde\alpha_t^{\,2}}\,\varepsilon_z
\right\},
\qquad
\varepsilon_z\stackrel{\mathrm{i.i.d.}}{\sim}\mathcal N(0,1).
\]
Its induced discrete marginal is matched to the uniform-diffusion signal level through the transformation
\[
\alpha_t=T(\widetilde\alpha_t),
\]
where
\[
T(\widetilde\alpha)
=
\frac{V}{V-1}
\left(
\int_{-\infty}^{\infty}
\phi\!\left(z-\frac{\widetilde\alpha}{\sqrt{1-\widetilde\alpha^2}}\right)
\Phi(z)^{V-1}\,dz
-\frac1V
\right),
\]
and \(\phi\) and \(\Phi\) denote the standard Gaussian density and cdf. Thus, for each prescribed \(\alpha_t\), we first invert \(T\) numerically to obtain \(\widetilde\alpha_t\), and then sample the Gaussian path at that value. 

For the multi-mask constructions, we compare both the nondesignated masked baseline and the designated multi-mask shared-Gumbel paths. In the designated family, each token \(x\in\gV\) is assigned a designated mask \(m_x=x\bmod M\), and the one-time marginal is given by \eqref{eq:forward_marginal_single}, with
\[
\beta_t=\alpha_t^\rho,\qquad \rho=0.5.
\]

The curve labeled ``\texttt{(Multi-)Masked, Unconditioned}'' corresponds to the nondesignated masked baseline. In that baseline, once the output is observed to be masked, the particular mask identity carries no information about the underlying clean token. Hence the posterior over \(x_0\) remains the prior \(p_{\mathrm{Zipf}}\), and the resulting conditional entropy is always \(H(p_{\mathrm{Zipf}})\), independent of the number of mask symbols \(M\). For this reason, all nondesignated multi-mask variants yield the same entropy curve, so we group them together under the label ``\texttt{(Multi-)Masked, Unconditioned}.''

We now describe how \(H(x_0\mid x_t,\omega)\) is evaluated. For each time \(t\), the experiment draws \(10^5\) independent clean tokens \(x_0^{(n)}\sim p_{\mathrm{Zipf}}\), together with the corresponding auxiliary randomness \(\omega^{(n)}\) for the chosen path construction, and forms
\[
x_t^{(n)} = F_t\!\bigl(x_0^{(n)},\omega^{(n)}\bigr),
\]
where \(F_t\) denotes the relevant pathwise map. For a fixed realized pair \((x_t,\omega)\), the posterior over a candidate clean token \(j\in\gV\) is
\[
q( x_0 = j \mid x_t, \omega)
=
\frac{w_j(x_t,\omega)}{\sum_{i\in\gV} w_i(x_t,\omega)},
\qquad
w_j(x_t,\omega)
=
p_{\mathrm{Zipf}}(j)\,\1\!\left\{F_t(j,\omega)=x_t\right\}.
\]
That is, a token \(j\) receives weight equal to its prior mass if, under the same shared randomness \(\omega\), it would produce the observed output \(x_t\), and zero otherwise. The corresponding conditional entropy is therefore
\[
H(x_0\mid x_t,\omega)
=
-\sum_{i\in\gV} q( x_0 = i \mid x_t, \omega) \log q( x_0 = i \mid x_t, \omega).
\]
The reported quantity is the Monte Carlo average of this conditional entropy over the sampled pairs \((x_0^{(n)},\omega^{(n)})\). 

\subsection{Experimental Details}
\label{app:experimental_details}

In this subsection, we collect the dataset, architecture, training, distillation, and sampling details used in our main experiments. 

\subsubsection{Datasets}
\label{app:datasets}

\paragraph{Data-governance note.}
The controlled 170M experiments use two English-language training corpora, denoted \corpusA and \corpusB; their identities are withheld to comply with data-disclosure constraints.

\paragraph{\corpusB.}
For \corpusB, we use the standard training, validation, and test splits, with preprocessing following~\cite{sahoo2025diffusion}. We tokenize the text using the \texttt{bert-base-uncased} tokenizer, with vocabulary size 30522. During training, examples are packed into fixed-length sequences of length 128.

\paragraph{\corpusA.}
For \corpusA, a large-scale English-language corpus, we use standard concatenation, tokenize the text with the \texttt{GPT2} tokenizer, and pack the resulting token stream into sequences of length 1024.

\subsubsection{Training Details}
\label{app:training_details}

\paragraph{Model architecture.}
Unless otherwise specified, all experiments use a diffusion transformer (DiT) backbone with approximately 170M parameters. The detailed architecture consists of 12 transformer blocks, hidden dimension 768, and 12 attention heads. We use rotary positional embedding and sinusoidal time embedding. Additional architecture details follow the implementation in ~\cite{sahoo2025diffusion}.

\paragraph{Pretraining.}
We pretrain all models with the loss~\eqref{eq:sequence_objective_main}. For the main pretraining runs, models are trained for 200K optimization steps using a minibatch size of 1024 and a learning rate of $3\times 10^{-4}$. We use the AdamW optimizer with $\beta_1=0.9$, $\beta_2=0.999$, and constant learning-rate schedule with warmup.

For the uniform-state diffusion~\citep{sahoo2025diffusion} and masked diffusion~\citep{sahoo2024simple} schedules in~\eqref{eq:uniform_forward_prelim} and~\eqref{eq:masked_forward_prelim}, respectively, we set
\[
\alpha_t = 1-t.
\]
For the multi-mask schedule, we use
\[
\beta_t = \alpha_t = 1-t
\]
throughout all language-modeling experiments. For the designated mapping, we set the designated mask as 
$$
    m_x=x\bmod M, \qquad x \in \gV,
$$
where we assume the numbering of the vocabulary $\gV$ given by the tokenizer.

For the continual-training experiments (\ours-cont), we initialize from an MDLM checkpoint trained for 150K steps and continue training the multi-mask model for an additional 50K steps under the same optimization settings unless otherwise noted.

\paragraph{Inference.}
At inference time, generation starts from the terminal distribution of the corresponding forward process and then applies the learned reverse transition operator for a prescribed number of denoising steps. For MDLM, this corresponds to starting from the all-mask state. For uniform-state diffusion and \ours, initialization is drawn from the corresponding terminal distribution described in Sec.~\ref{sec:prelim} and~\ref{sec:methodology}, respectively.

Unless otherwise stated, we use ancestral sampling without classifier-free guidance. For each evaluation setting, we generate 128 samples to estimate the average generative perplexity (GenPPL) and the sample entropy. GenPPL is computed by GPT-2 Large model, and the sample entropy is computed according to the token frequencies in the generated samples.

\paragraph{Consistency distillation.}

In our experiments, we empirically find that the reverse-KL loss
\[
\widetilde{\gL}_{\mathrm{CD}}^\delta(\theta,\bar\theta)
:=
\E_{\vx_0\sim p_0,\,\omega,\, t\sim\operatorname{Unif}(\delta,1)}
\left[
    \sum_{\ell=1}^L
    D_{\mathrm{KL}}\left(
        f_{t,\theta}^\ell(\cdot\mid \widetilde\vx_t)
        \,\middle\|\,
        \operatorname{sg}\bigl[f_{t-\delta,\bar\theta}^\ell(\cdot\mid \widetilde\vx_{t-\delta})\bigr]
    \right)
\right],
\]
obtained by reversing the order of the two distributions in the consistency-distillation loss~\eqref{eq:consistency_distillation_loss_main}, yields better numerical stability and better empirical performance. A natural intuition is that the teacher distribution at time \(t-\delta\) is typically sharper than the student distribution at time \(t\), since it corresponds to a less corrupted state and hence lower posterior uncertainty. In this regime, reverse KL is often preferable because it is more mode-seeking, as it encourages the student model to place mass on the dominant modes of the sharper teacher distribution. By contrast, the forward KL is more mode-covering and can therefore encourage overly diffuse predictions when the target distribution is narrow. We therefore adopt this reverse-KL objective in all consistency-distillation experiments.

We perform distillation for \(R = 5\) rounds and adopt a dyadic step-size schedule in which the step size doubles after every \(K_R = 10^4\) optimization steps, namely
\[
\delta_k = 2^{-9 + \lfloor k/K_R \rfloor},
\qquad k = 1, \dots, K.
\] 
We set the minibatch size as 64 and the learning rate as $3\times10^{-4}$. Unless otherwise stated, the student is initialized from the pretrained \ours\ checkpoint, while the teacher is taken to be the EMA teacher with stopped gradients.

\subsubsection{Sampling Temperature}
\label{app:tau}

At inference time, for each coordinate, the DiT model outputs a vector
\[
g_{t,\theta}(\cdot \mid \vx_t)\in\mathbb{R}^{|\gV|},
\]
whose entries are the pre-softmax scores used to form the model prediction conditioned on the current state \(\vx_t\). The corresponding probability vector is then
\[
x_{t,\theta}(\cdot \mid \vx_t)
=
\softmax\!\bigl(g_{t,\theta}(\cdot \mid \vx_t)\bigr),
\]
where the softmax is taken over the clean vocabulary \(\gV\). To control the sharpness of this prediction at sampling time, we introduce a temperature parameter \(\tau>0\) and rescale the pre-softmax scores before applying the softmax:
\[
x_{t,\theta}^{\tau}(\cdot \mid \vx_t)
=
\softmax\!\left(\frac{g_{t,\theta}(\cdot \mid \vx_t)}{\tau}\right).
\]
Smaller values of \(\tau\) make the prediction more concentrated on high-score states, while larger values of \(\tau\) flatten the distribution and yield more stochastic samples.

In the temperature-aligned experiments in Sec.~\ref{sec:experiments}, we compare methods under a common fixed temperature, with default choice \(\tau=1\). In the entropy-aligned experiments, we instead tune \(\tau\) separately for each method and each sampling budget so that the generated samples match the prescribed entropy level.

\subsection{Additional Experiments}
\label{app:additional}

In this subsection, we provide additional Pareto-front visualizations beyond those shown in the main text. These figures offer a more complete view of the perplexity-entropy trade-off across different sampling budgets, numbers of masks, and distillation strategies.

The models shown in Figs.~\ref{fig:distill-exp} and \ref{fig:ablation-4-main} are trained on \corpusA; their setup is described below.

\paragraph{Pretrained models under varying sampling budgets (Figs.~\ref{fig:corpus-a-pretrain} and \ref{fig:corpus-b-pretrain}).}
Figs.~\ref{fig:corpus-a-pretrain} and \ref{fig:corpus-b-pretrain} report the perplexity-entropy Pareto fronts for models pretrained separately on \corpusA and \corpusB under different reverse-step budgets $K$. Within each figure, performance improves across the three panels as the sampling budget increases, reflected by a systematic movement of the Pareto fronts toward the lower-right region. \ours is consistently competitive with or better than the baselines, and the continual-training variant is especially strong at moderate and large sampling budgets. These results support the claim that the multi-mask design improves the pretrained teacher itself, even before applying consistency distillation.

\paragraph{Effect of the number of masks in pretraining (Fig.~\ref{fig:corpus-a-pretrain-M-ablation}).}
Fig.~\ref{fig:corpus-a-pretrain-M-ablation} studies pretrained \ours under different numbers of masks $M$ and different sampling budgets. The overall trend is that increasing $M$ from the single-mask regime substantially improves the perplexity-entropy trade-off, while excessively large $M$ eventually yields diminishing returns and may slightly hurt performance. In our experiments, intermediate values, especially around $M=50$, give the strongest overall performance. This is consistent with the intuition in Sec.~\ref{sec:methodology} that additional masks enrich the terminal distribution and intermediate path geometry, but only up to a point.

\paragraph{Distilled models with DCD/shared-Gaussian vs. shared-Gumbel coupling (Fig.~\ref{fig:corpus-a-distill-gumbel}).}
Fig.~\ref{fig:corpus-a-distill-gumbel} extends the main-text distillation comparison by showing Pareto fronts at multiple sampling budgets for DCD-based distillation. The shared-Gumbel distilled \ours remains consistently favorable across step counts, especially in the practically important few-step regime. Relative to pretrained models, the distilled models move the Pareto frontier downward at matched entropy, indicating improved generation quality after distillation. This figure also shows that the gains from the multi-mask teacher persist beyond the specific $4$-step setting highlighted in the main text.

\paragraph{Distilled models with SDTT (Fig.~\ref{fig:corpus-a-distill-sdtt}).}
Fig.~\ref{fig:corpus-a-distill-sdtt} presents the corresponding Pareto-front comparison under the SDTT distillation procedure. The same qualitative pattern holds: \ours yields a stronger distilled generator than the single-mask and uniform-state baselines, and the improvement is most visible in the low-step regime where the distillation problem is hardest. Together with Fig.~\ref{fig:corpus-a-pretrain-M-ablation}, these results indicate that the advantage of \ours is not tied to one particular distillation algorithm, but instead comes from the underlying multi-mask teacher construction.

\begin{figure}[h]
    \centering
    \begin{subfigure}{0.32\textwidth}
        \centering
        \includegraphics[width=\linewidth]{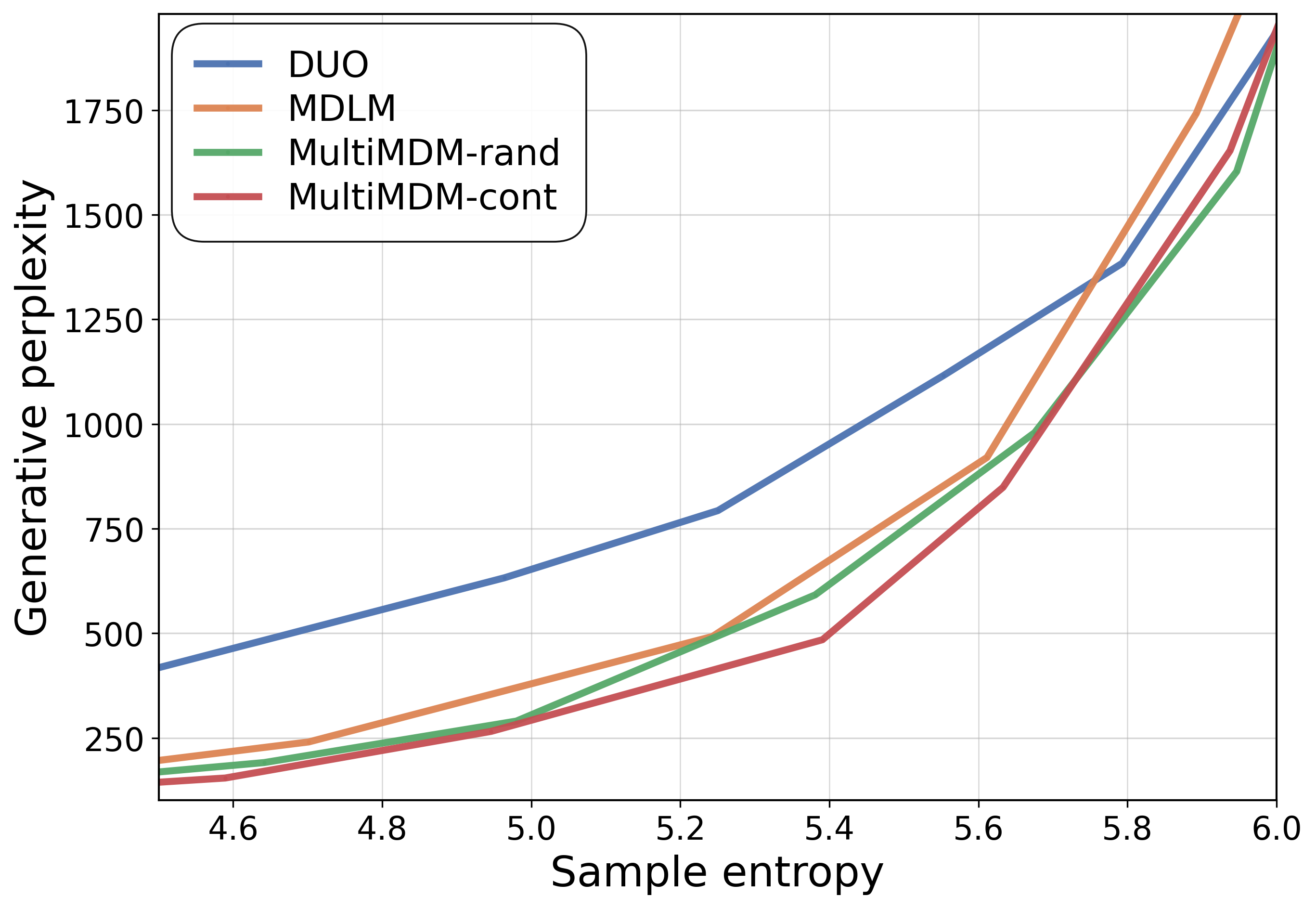}
        \caption{4 sampling steps}
        \label{fig:corpus-a-pretrain-4}
    \end{subfigure}
    \hfill
    \begin{subfigure}{0.32\textwidth}
        \centering
        \includegraphics[width=\linewidth]{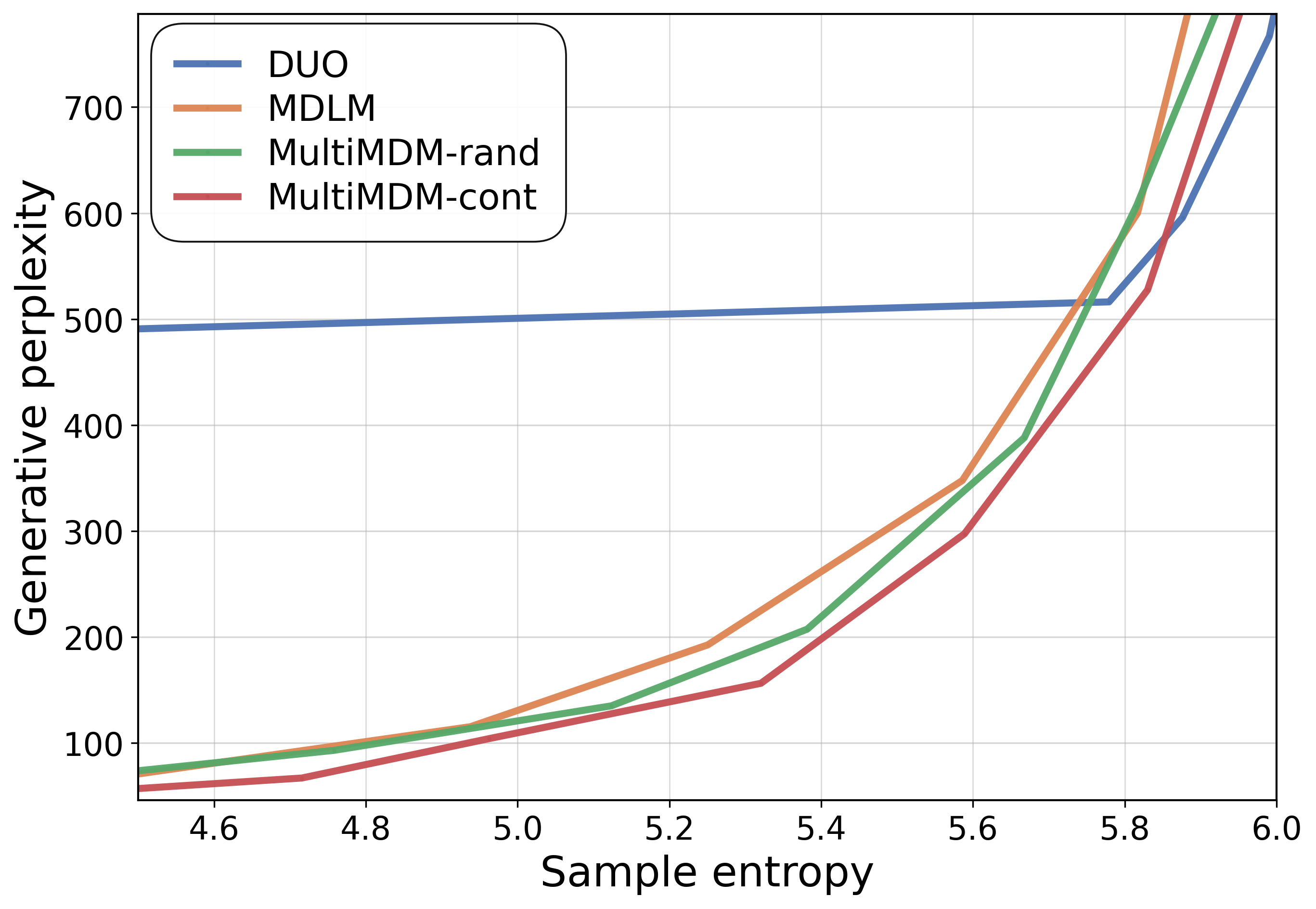}
        \caption{8 sampling steps}
        \label{fig:corpus-a-pretrain-8}
    \end{subfigure}
    \hfill
    \begin{subfigure}{0.32\textwidth}
        \centering
        \includegraphics[width=\linewidth]{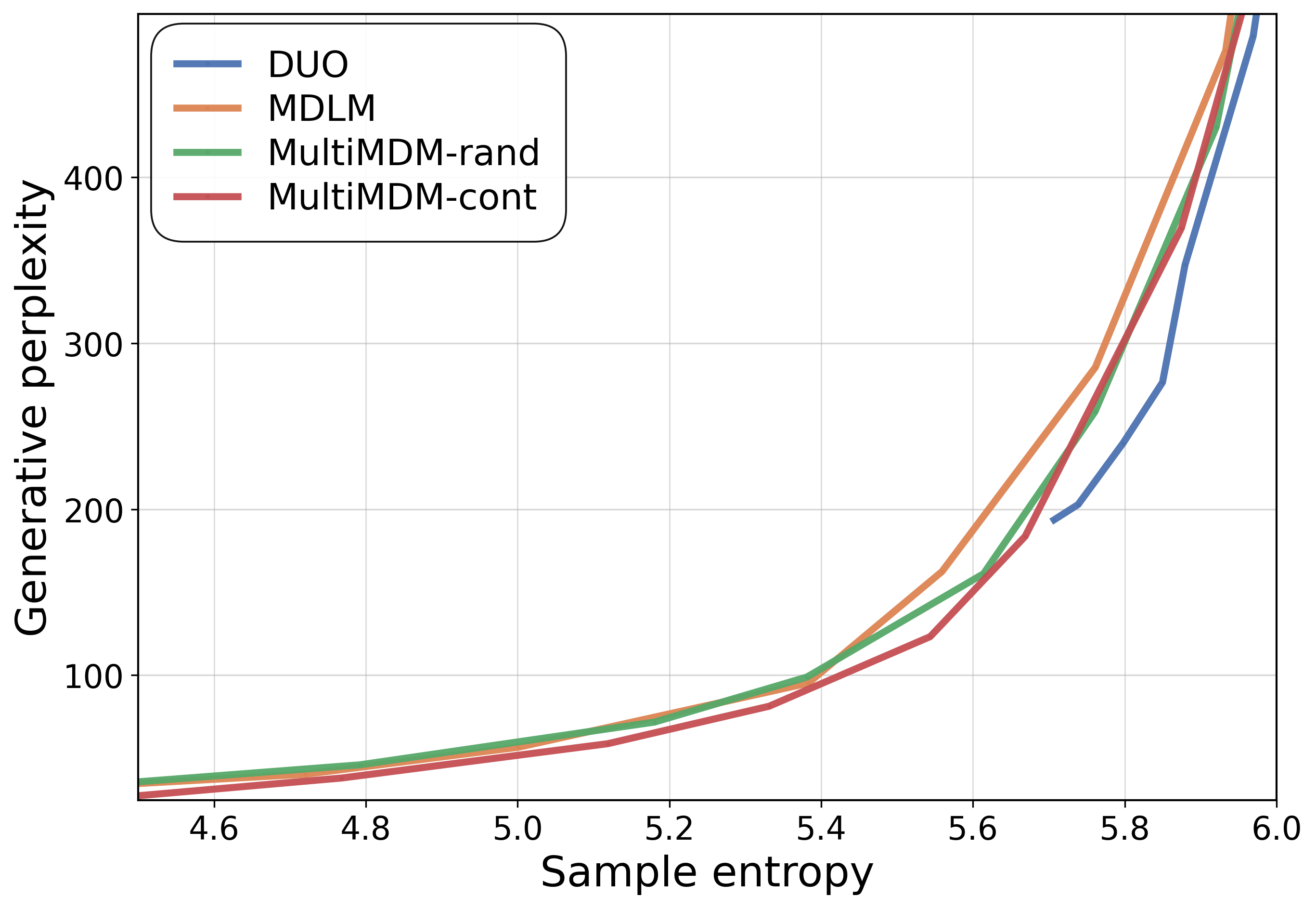}
        \caption{16 sampling steps}
        \label{fig:corpus-a-pretrain-16}
    \end{subfigure}
    \caption{Perplexity-entropy Pareto fronts for models pretrained on \corpusA (corpus entropy $H=5.44$) under different sampling budgets $K$.}
    \label{fig:corpus-a-pretrain}
\end{figure}

\begin{figure}[h]
    \centering
    \begin{subfigure}{0.32\textwidth}
        \centering
        \includegraphics[width=\linewidth]{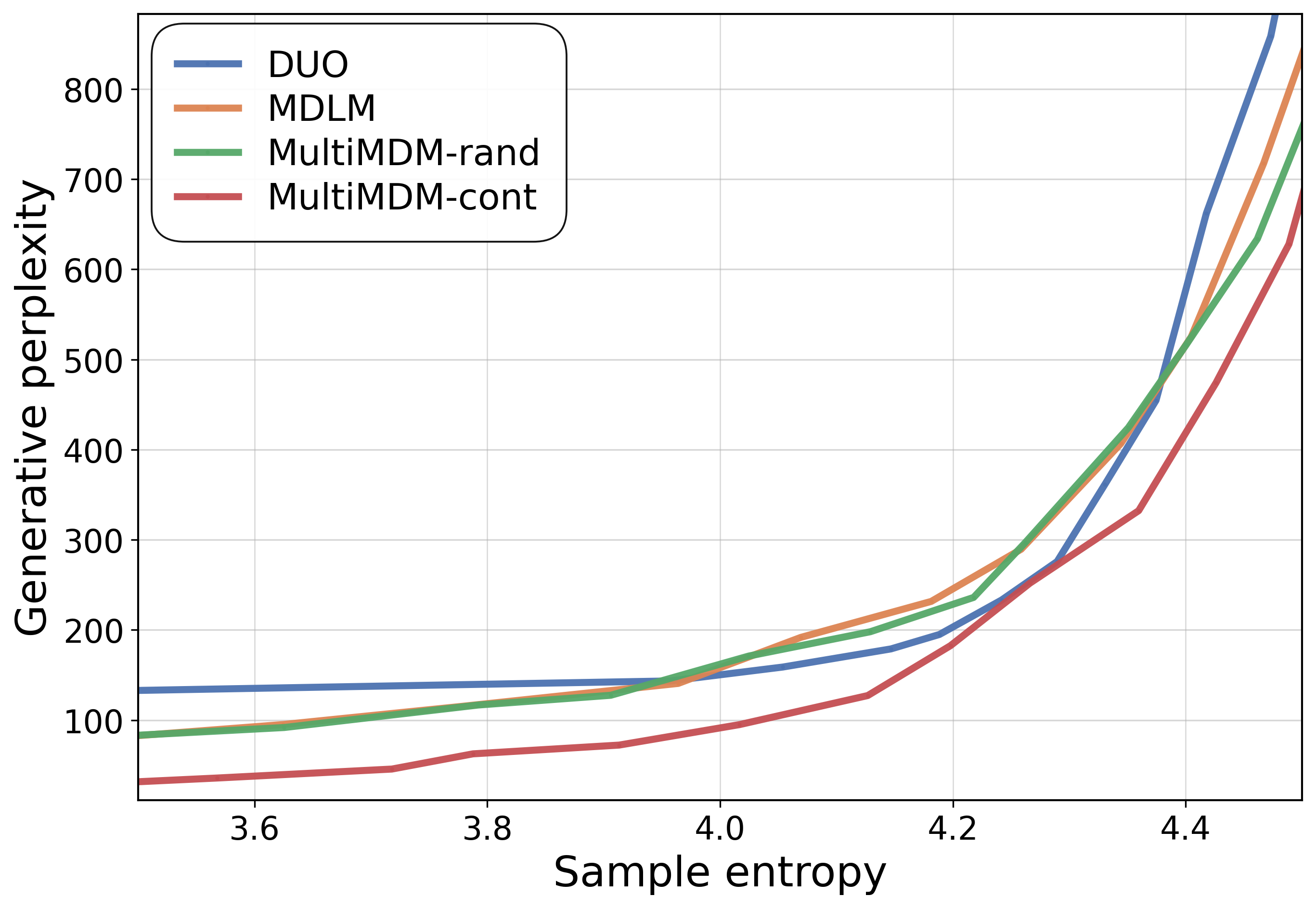}
        \caption{4 sampling steps}
        \label{fig:corpus-b-pretrain-4}
    \end{subfigure}
    \hfill
    \begin{subfigure}{0.32\textwidth}
        \centering
        \includegraphics[width=\linewidth]{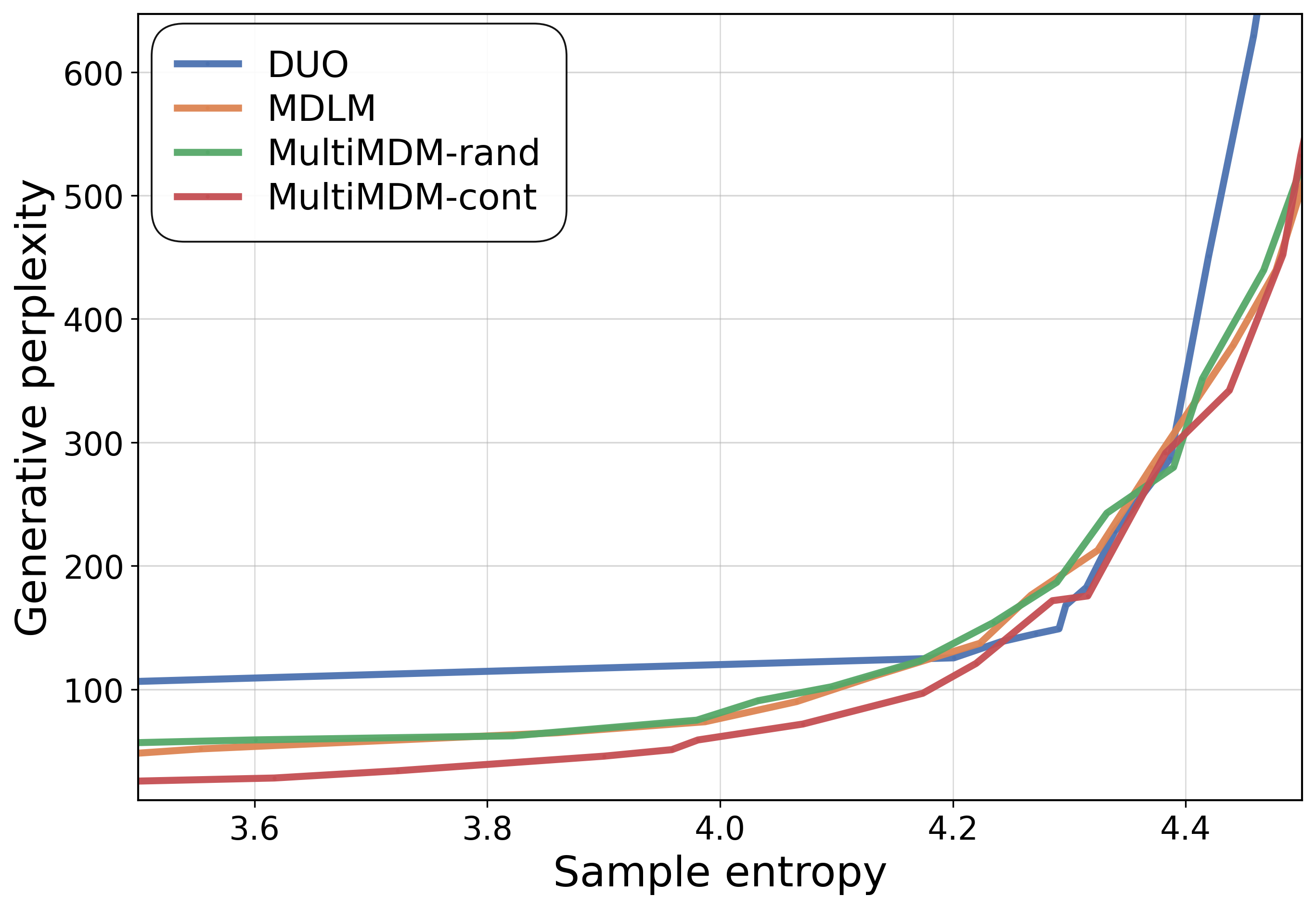}
        \caption{8 sampling steps}
        \label{fig:corpus-b-pretrain-8}
    \end{subfigure}
    \hfill
    \begin{subfigure}{0.32\textwidth}
        \centering
        \includegraphics[width=\linewidth]{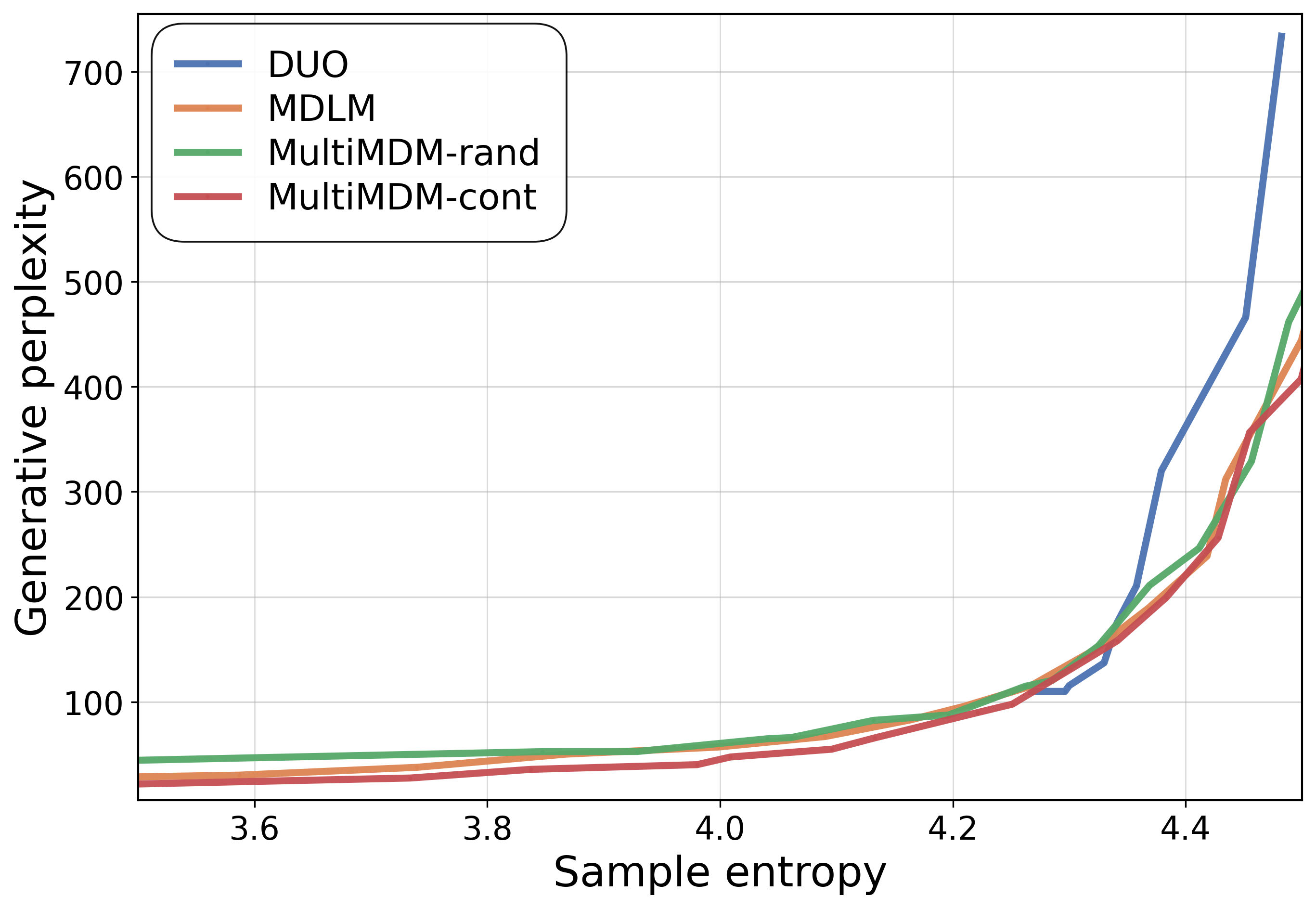}
        \caption{16 sampling steps}
        \label{fig:corpus-b-pretrain-16}
    \end{subfigure}
    \caption{Perplexity-entropy Pareto fronts for models pretrained on \corpusB (corpus entropy $H=4.31$) under different sampling budgets $K$.}
    \label{fig:corpus-b-pretrain}
\end{figure}

\begin{figure}[h]
    \centering
    \begin{subfigure}{0.32\textwidth}
        \centering
        \includegraphics[width=\linewidth]{images/frontiers_ablation_4.png}
        \caption{4 sampling steps}
        \label{fig:corpus-a-pretrain-M-ablation-4}
    \end{subfigure}
    \hfill
    \begin{subfigure}{0.32\textwidth}
        \centering
        \includegraphics[width=\linewidth]{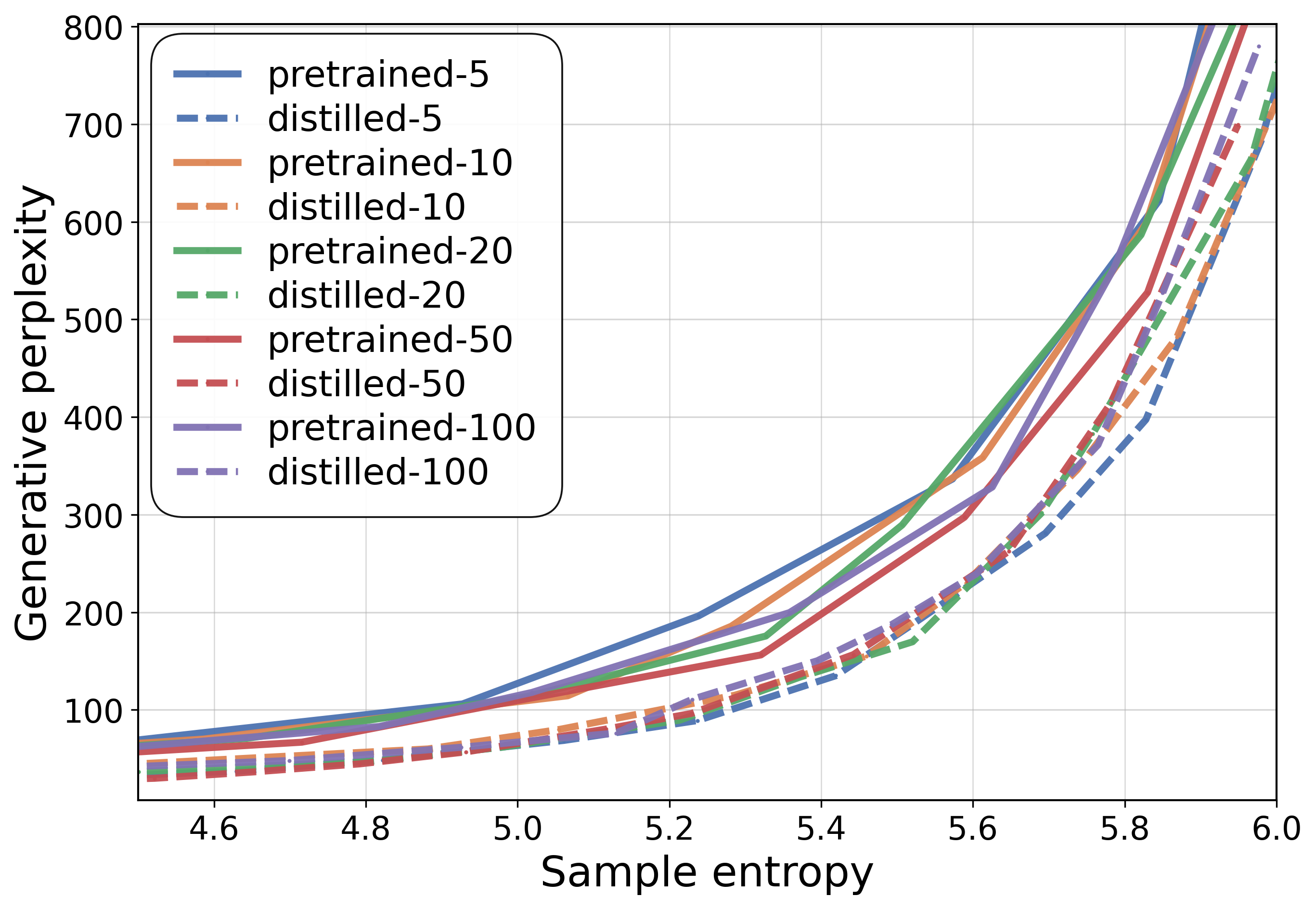}
        \caption{8 sampling steps}
        \label{fig:corpus-a-pretrain-M-ablation-8}
    \end{subfigure}
    \hfill
    \begin{subfigure}{0.32\textwidth}
        \centering
        \includegraphics[width=\linewidth]{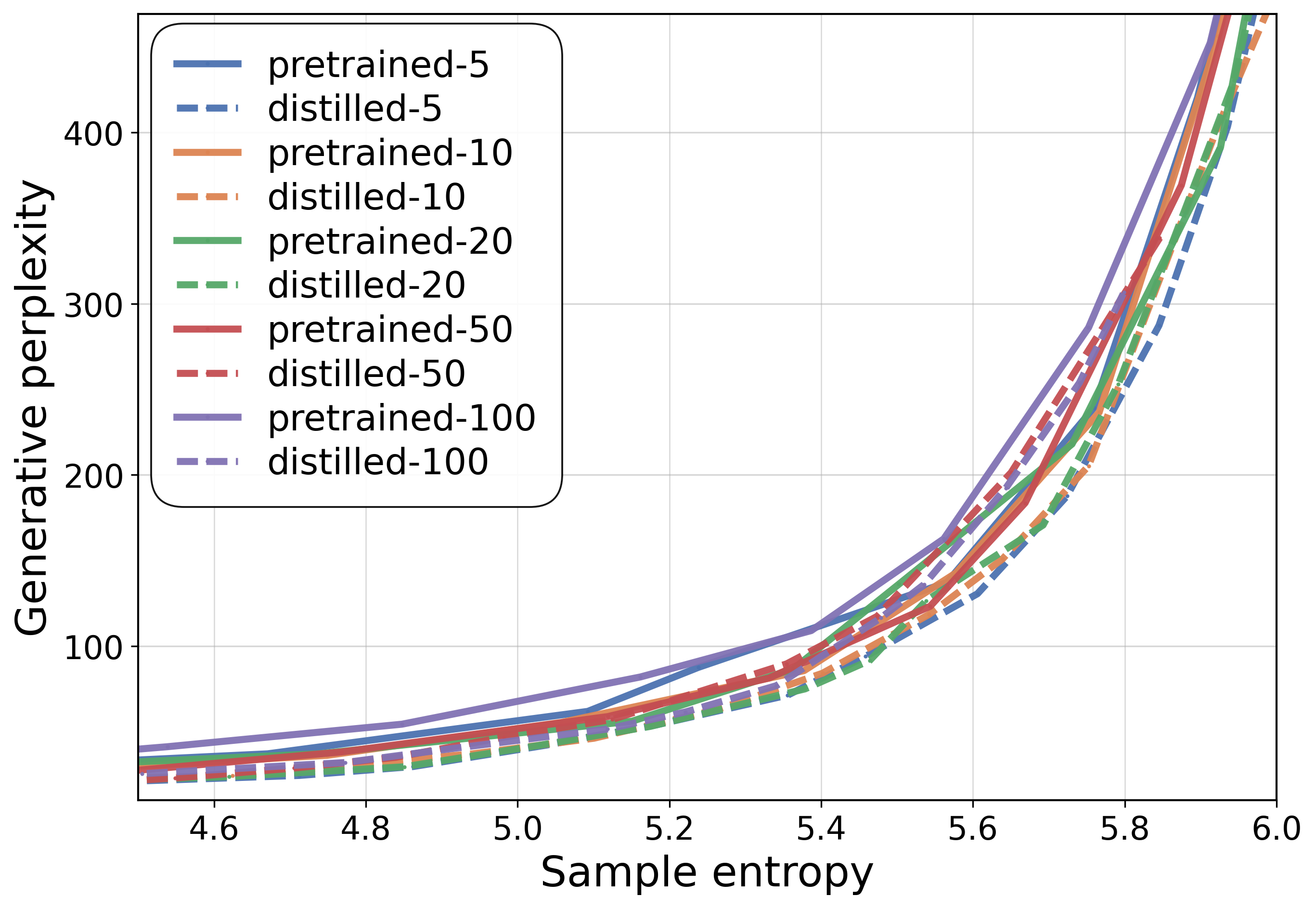}
        \caption{16 sampling steps}
        \label{fig:corpus-a-pretrain-M-ablation-16}
    \end{subfigure}
    \caption{Perplexity-entropy Pareto fronts for models pretrained on \corpusA with different numbers of masks $M$ and sampling steps $K$.}
    \label{fig:corpus-a-pretrain-M-ablation}
\end{figure}

\begin{figure}[h]
    \centering
    \begin{subfigure}{0.32\textwidth}
        \centering
        \includegraphics[width=\linewidth]{images/frontiers_corpus_a_gumbel_4.png}
        \caption{4 sampling steps}
        \label{fig:corpus-a-distill-gumbel-4}
    \end{subfigure}
    \hfill
    \begin{subfigure}{0.32\textwidth}
        \centering
        \includegraphics[width=\linewidth]{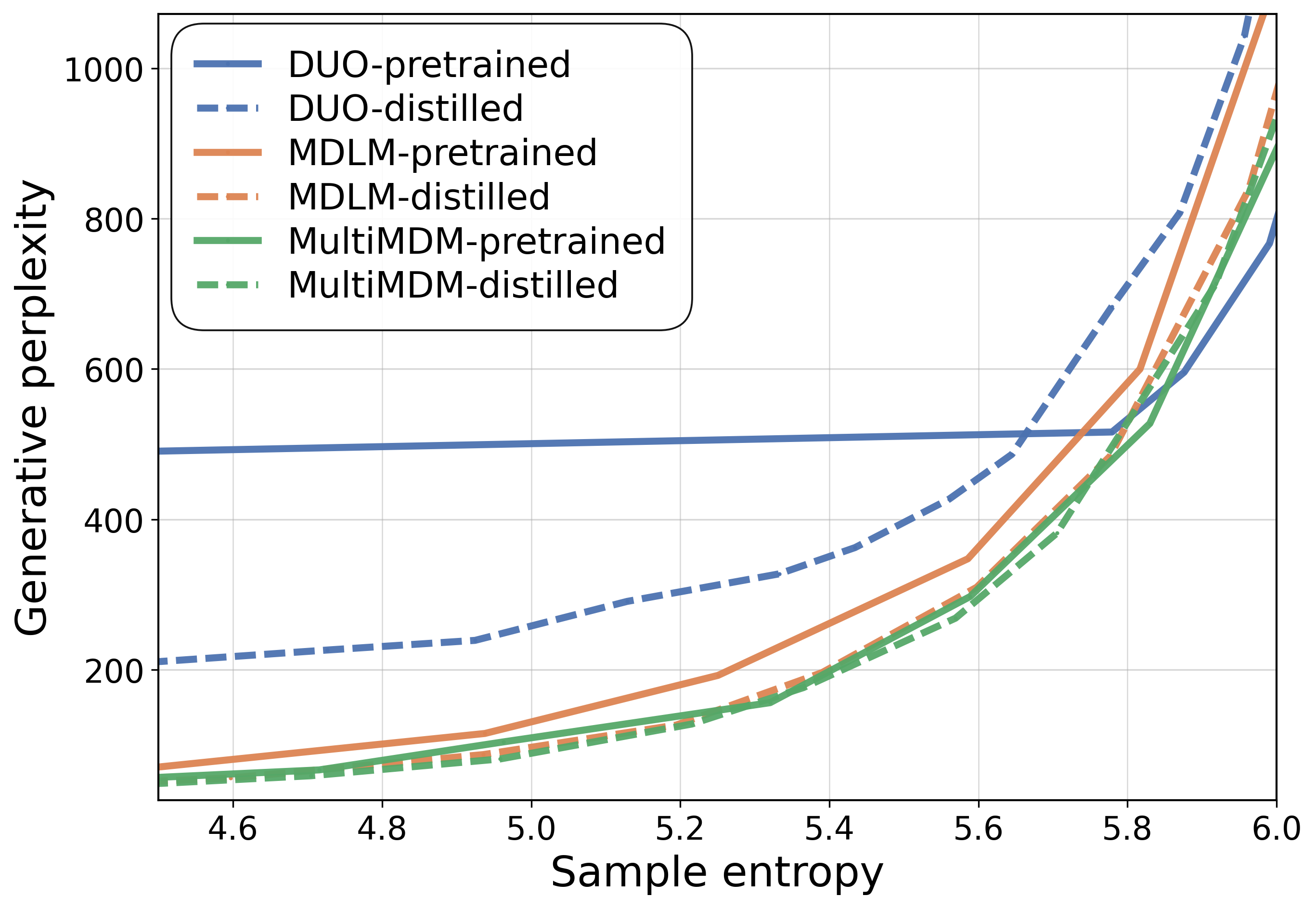}
        \caption{8 sampling steps}
        \label{fig:corpus-a-distill-gumbel-8}
    \end{subfigure}
    \hfill
    \begin{subfigure}{0.32\textwidth}
        \centering
        \includegraphics[width=\linewidth]{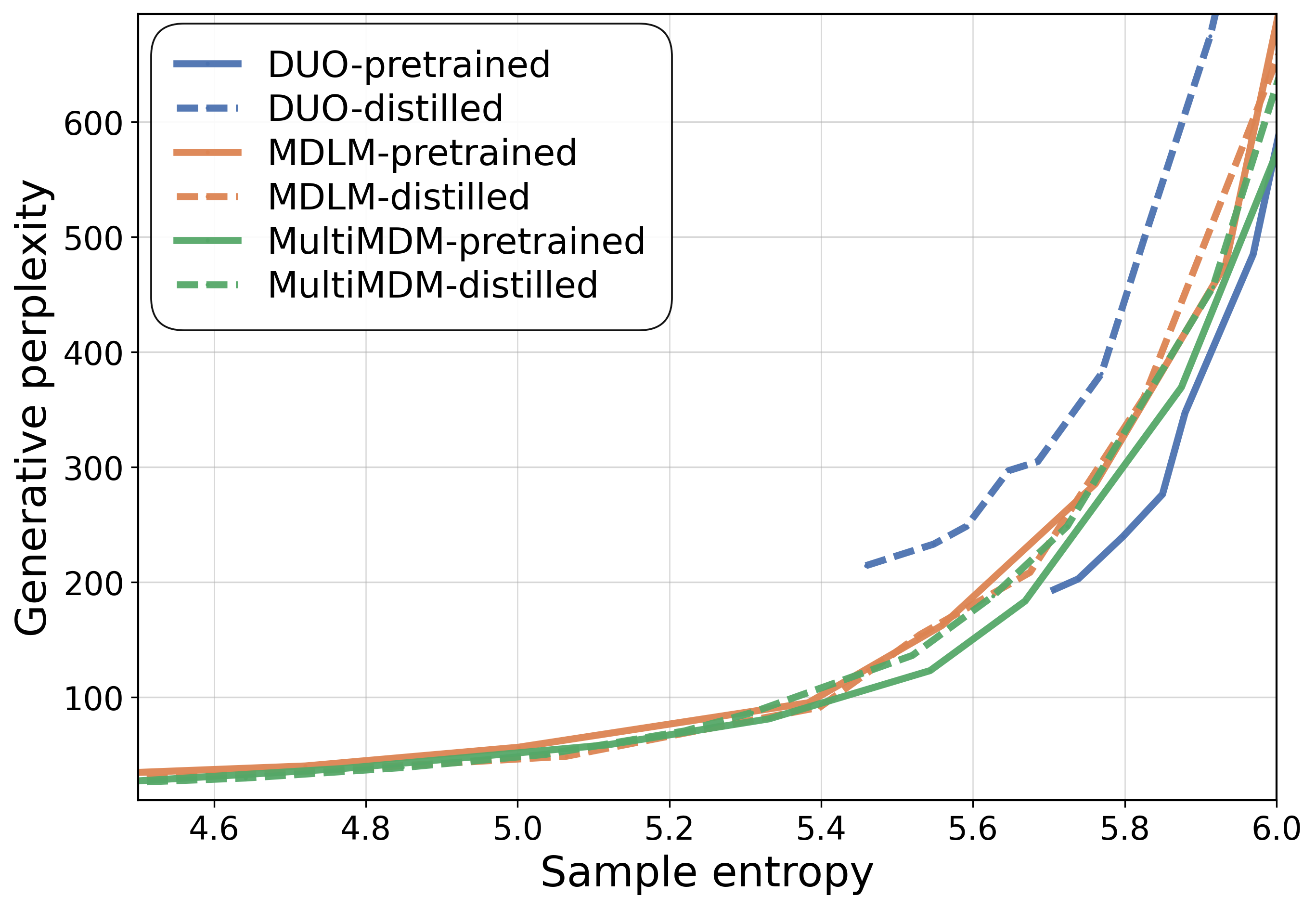}
        \caption{16 sampling steps}
        \label{fig:corpus-a-distill-gumbel-16}
    \end{subfigure}
    \caption{Perplexity-entropy Pareto fronts for models distilled on \corpusA using either DCD with shared-Gaussian coupling or our shared-Gumbel method, under different sampling budgets $K$.}
    \label{fig:corpus-a-distill-gumbel}
\end{figure}

\begin{figure}[h]
    \centering
    \begin{subfigure}{0.32\textwidth}
        \centering
        \includegraphics[width=\linewidth]{images/frontiers_corpus_a_sdtt_4.png}
        \caption{4 sampling steps}
        \label{fig:corpus-a-distill-sdtt-4}
    \end{subfigure}
    \hfill
    \begin{subfigure}{0.32\textwidth}
        \centering
        \includegraphics[width=\linewidth]{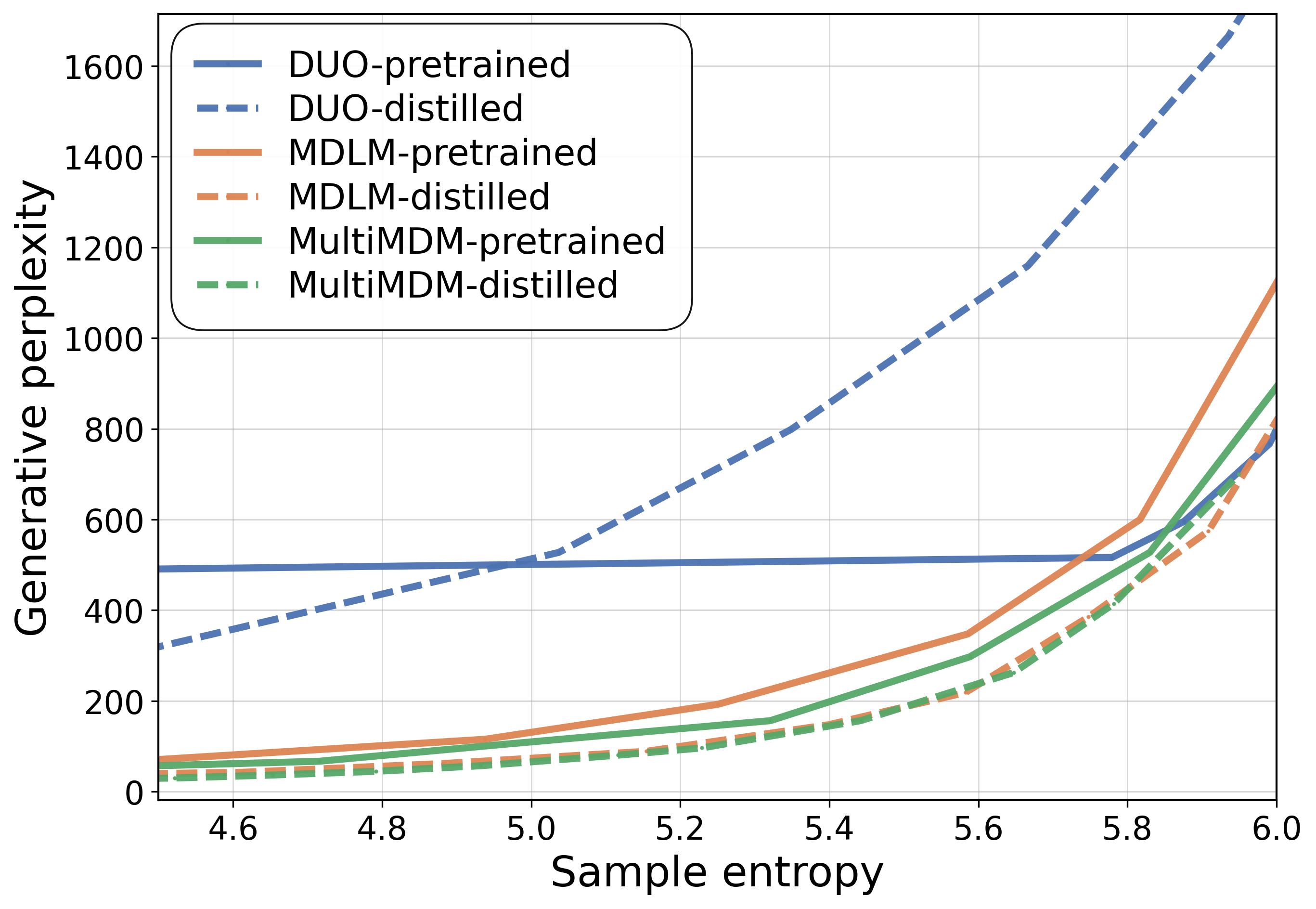}
        \caption{8 sampling steps}
        \label{fig:corpus-a-distill-sdtt-8}
    \end{subfigure}
    \hfill
    \begin{subfigure}{0.32\textwidth}
        \centering
        \includegraphics[width=\linewidth]{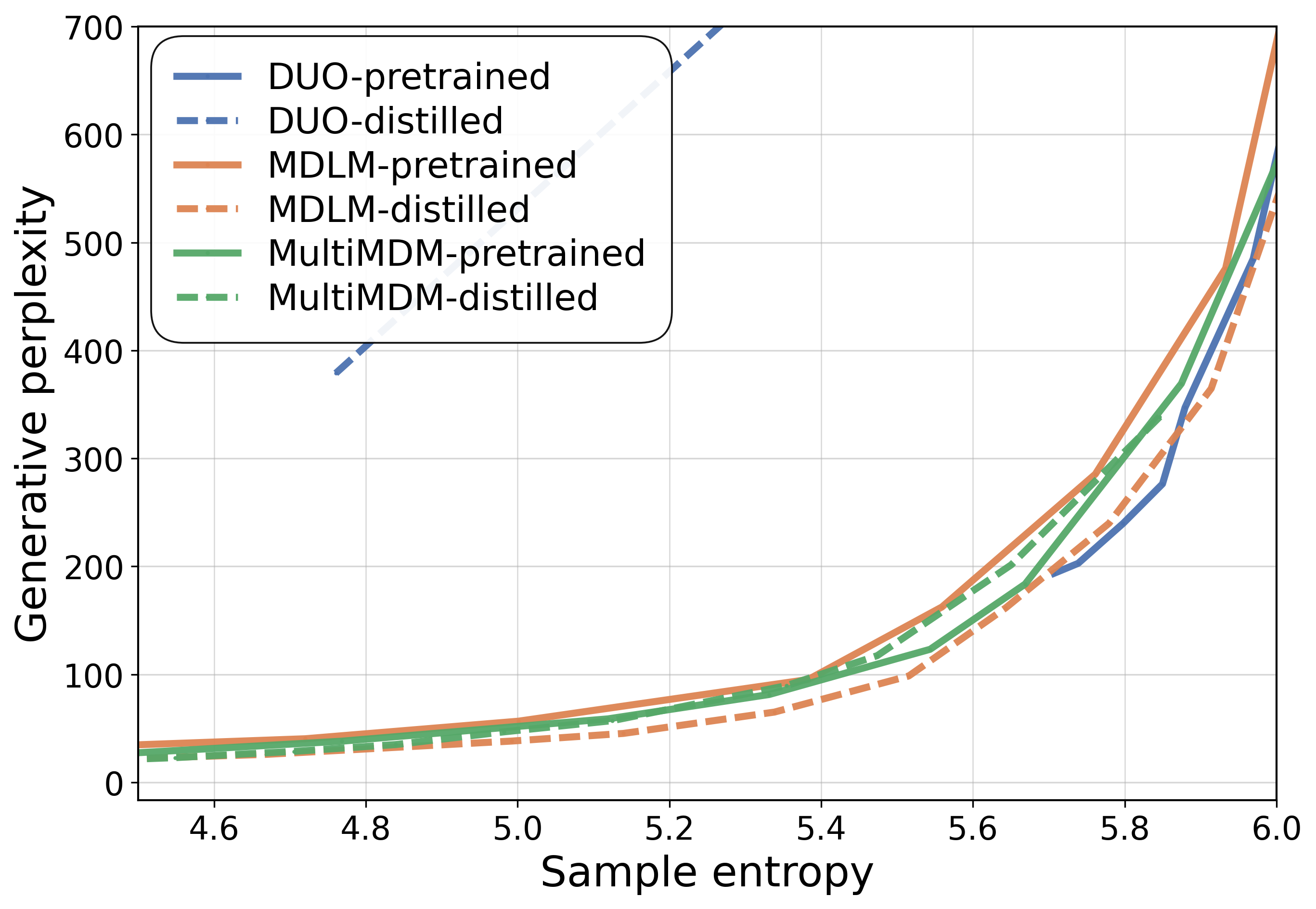}
        \caption{16 sampling steps}
        \label{fig:corpus-a-distill-sdtt-16}
    \end{subfigure}
    \caption{Perplexity-entropy Pareto fronts for models distilled with SDTT on \corpusA under different sampling budgets $K$.}
    \label{fig:corpus-a-distill-sdtt}
\end{figure}

\subsection{\multillada Adaptation and Evaluation}
\label{app:multi_llada}

We initialize \multillada from the LLaDA-8B-Base checkpoint~\citep{nie2025large} and expand its diffusion state space to $M=50$ mask tokens. Following the continual-training construction in Sec.~\ref{sec:exp_pretraining}, we leave the clean-token output head unchanged, introduce the additional mask-token embeddings, and adapt the model with low-rank updates using DCLM and OpenCoder data. We train for 5,000 total optimization steps with batch size 256 and sequence length 1024. The LoRA rank is $8$ with scaling parameter $16$, and we use AdamW with learning rate $3\times10^{-5}$. LoRA was chosen as a resource-efficient 8B adaptation setting; it is not a matched full-parameter replication of the 170M continual-training experiment.

We evaluate both LLaDA and \multillada with the official LLaDA evaluation harness.\footnote{\url{https://github.com/ML-GSAI/LLaDA}} The evaluation covers the math benchmarks MATH500 and GSM8K and the code benchmarks HumanEval and MBPP. We report answer accuracy for the math tasks and Pass@1 for the code tasks. As generation length $L$ varies across benchmarks in the original harness setting, we use three reverse-step budgets: $L$, $L/2$, and $L/4$ steps, corresponding respectively to 1, 2, and 4 generated tokens per step. All other evaluation settings are shared between the base and adapted models. The resulting scores are reported in Tab.~\ref{tab:llada}. We have not yet tested full-parameter 8B adaptation or shared-Gumbel distillation at this scale.

\subsection{Generation Examples}
\label{app:gen_examples}

In this subsection, we provide representative unconditional generation examples from pretrained \ours models at different sampling budgets. These examples are intended to qualitatively illustrate how sample coherence degrades as the number of sampling steps decreases, complementing the quantitative GenPPL and entropy evaluations reported in Sec.~\ref{sec:experiments}.

\subsubsection{Generations from \ours pretrained on \corpusB}

We first show generations from this model. Because it was trained using shorter packed sequences, these examples more directly reflect local grammaticality and short-range coherence under different sampling budgets. The examples below are 128 tokens long.

\paragraph{Generation example at 32 sampling steps.}

\begin{wrappedlisting}
[CLS] the world's authors will now be able to exploit the gains by reader applications and the book mainstream virtually any book, a major independent publisher has said. [CLS] democrats are that there was just no time in sailing without anyone knocking on the head of voters. [CLS] the paper has the goods arriving on thursday since then, and war remains a day for the bungled delivery. [CLS] the nhl also reported it was in talks on thursday about naming a new general commissioner. [CLS] entry notices which celebrating athletes'birthdays have been blacked out on officials websites. [CLS] miss jones's victory against larisa wilco will have marytti [CLS]
\end{wrappedlisting}

\paragraph{Generation example at 16 sampling steps.}

\begin{wrappedlisting}
[CLS] five householders are rescued from a blaze in which an elderly british woman died. [CLS] brian myers ( 3 - 3 ) gave up two hits for the losing pitcher, but perez ( 1 - 5 ) allowed eight strikeouts innings as the mets slumped to a loss for a fifth straight games. [CLS] and there are bound to be hidden risks - the potential for them is unknown. [CLS] police paramedics rescued a woman to the safety of the house. [CLS] the advantage ( under the limitations inherent in broadcast regulation, they need salary legal ) is that they can show off of ofcom when it's \" normal \" as required by the [CLS]
\end{wrappedlisting}

\paragraph{Generation example at 8 sampling steps.}

\begin{wrappedlisting}
[CLS] more than 24 he did each time the freshman went to ucla. [CLS] with 4 - 1 up against fernando alonsotel vettel came through there to sit chinalco's. [CLS] do not expect fatalities, with dozens more staff staph infections likely [CLS] the gap - donut favorite turntable bird recently occupied king abdullah ii? [CLS] the federal housing agency'more than altruistic low - income borrower and credit, as well as money, makes people ms. also features more concerned in the program. missing an average saturday afternoon sleep. five lane is from ambitions. and the den, bank commute, restaurant seventh, [CLS]
\end{wrappedlisting}

\subsubsection{Generations from \ours pretrained on \corpusA}

We next show generations from this model. Because the training sequences for this corpus are much longer, these examples are primarily intended to illustrate global fluency, topic continuity, and failure modes at different step counts. The examples below are 1,024 tokens long.

\paragraph{Generation example at 64 sampling steps.}

\begin{wrappedlisting}
<|endoftext|> thousands of people to flee the Libyan military if the violence breaks out in weeks.\n\nIn late February, Brahmenya hospitals in Tripoli said they treat 14,000 people, while Edaara hospital in Benghazi in Benghazi had treated over 80 people.\n\nGdadara hospitals said only 12 people had been confirmed dead during the protests so so far, killing 21, including 158, without further details.\n\nTrust International swiftly denied the release of the figures. The news agency denied being quoted by Bangladeshi reporter on whether it had joined opposition because it sought to expose its affiliations, which were exposed in Shahan\u2019s interaction on November 12, 2011.\n\nThe BBC said it felt Shahsan had abused the authority.\n\nThis does not mean we are opposed,\u201d it said, \u201cbut it really exists in the identity of the Libyan Special Operations teams.\u201d\n\nThe Ministry of Foreign Affairs Jamaat Online said on June 8 reports relating to military intervention in Libya that \u201cWe had not to concede and we had not to commit anything\u201d.\n\nThe ECEC and Ori Investigation Agency monitoring group in February 2014 said Mohan Ahmet was concerned by \u201cgreatly independence\u201d of DNE in last year\u2019s uprising.\n\nIn late February 2014, he said the uprising helped black out public opinion. He had been forced to meet senior opposition figures and imprisoned. Authorities promised to return \u201csolidarity and freedom to speak out\u201d.\n\nA document released in 2012, acknowledged the encouragement for protests but said he could not have spoken to any political organisation of the national security forces in Libya to prove there was not enough of national unity\n\nLibyan Independence Party leader goieah readies a pae during a protest rally in Gaddafi's Drum de Sa in Tripoli November 12, 2014. REUTERS/Aneas\n\nHahoman had met with Egyptian Musaam Front National Solidarity Fighters (MARC) for information on behalf of the RSF.\n\nThe independence struggle \u201cis territory we\u2019re there to protect,\u201d he told Guardian I. \u201cToday we will continue to protest and defend the revolution,\u201d he said.<|endoftext|>You think you fit a doctor? That's probably the best way to avoid getting sick and ill. All countries have higher levels of infection and death, with malnutrition and developmental delays in 9th, according to the World Health Organization's 2012 survey.\n\nDespite continuous diarrhea of about two-thirds of all childhood in the region, malnutrition has led to about 40 percent of all deaths in any given year. According to a 2009 WHO report U.S. patients are said to generate 75 percent of gastrointestinal symptoms including vomiting, diarrhoea, diarrhea and diarrhea.\n\n\"Studies show the greatest cases come from larger entraps where people are concentrated throughout southern and northern parts of the outbreak,\" said Frank M. Burni, Ph.D. surgeon. \"Rapid \u223cinid vomiting, diarrhea and other complications include examples of a loss of immunity that can slow processes in the body\u2019s metabolism.\"\n\nDifferent foods present nearly a full year of risk. UN approved Friday its annual definition of chronic and acute foodborne outbreaks, which suggests elimination of doctor-made animal feed and vaccines.\n\nA staggering number of doctor-metrica outbreaks have come from Afghanistan, Iraq, Iraq and Syria. In a report, the World Health Organisation \"described how so far there were 1.1 million in in 2010 where just 1 kg of food stock was ever collected.\" According to the World Health Organization, the Montreal's Institute, which report on outbreaks linked to food diseases, and Iraq, Iraq, Afghanistan, Syria, the total gap is \"624 million.\"\n\nRelated Stories Supported international organizations: Concerns and Prevent occurence\n\nChronic malnutrition and fighting against it through estimated proportions. The MORE\n\nThe rest of the World contains only about 125 million people, a combination of health, security and panic proportions, accordingoc evaluated data by London-based Sage Foundation's and principal physicians.\n\n\u201cThat means about half of that population is concentrated in the Americas \u2014 so solid collective nutritional intake is difficult to track,\u201d said Frank M. Burni, M.D. of Tiro University. It\u2019s \u201ca sort of medical emergency but global proportions are running their course rapidly.\u201d\n\nCronuts and processed foods eaten by thefed, starches to the liquid, drinking water causes diarrhea and vomiting. Nose infections and a different type of microorganisms are also more severe, from diarrhea to hepatic as well as severe stomach and vomiting. The disease occurs 57 percent of time in therisk tier, according to the latest review of 231 new cases. Studies in children who ate the upstream diet have found that parents may be more likely to present the disease.\n\nThere are increasing rates of the disease in the poorest countries. According to one study<|endoftext|>
\end{wrappedlisting}

\paragraph{Generation example at 32 sampling steps.}

\begin{wrappedlisting}
<|endoftext|> into one of the certificates for pirating machines. The reason they could then barely escape was to start talking in front of lots of restaurants and resemble a bar.\n\nAnd as of, the advantages of working official passports are very real. The. has spiraled, from many passports to plummet as high as $US42.6 billion the day of the weekend.\n\nIn the time since the West Point border disagreement of gold has skyrocketed from the value landong and for months more than $40 billion each day worth of remained unfilt.\n\nThe enthusiasm is shifting.<|endoftext|>After May 4 used to say Best in British movie SBE!\" Girl's girl are matesigham Lynn from Camden animated adventure drama \"No Kiss on your old girl.\"Happy to confirm this is an real voice, it has already been admitted byThe film 'Tattler and Not partner' say 'old girl' is Chads in fact, in 1985 film \u2018The Marine STER' and also your Telled Everything on Me the \"teen's song to action I have spitting.\"Getting to the point of singing, some of them now have a lot of skin under their face both, with their Black Stipragix 'as well as especially iconic, their Kenneth fangaroo bass.Kim Rivel, who has sung 'The Finnswear of a tune called The Scarlet Merse It All Mean, only has her boy's 'old girl' because she sings in Former Buta.\"At the event, the stars of absolute school bands got a serious job singing 'Party' to sing in:Will Ferrell sing 'The Village'. Does Shick. Hada's 'Death Valley' \u2013 but not funny enough for audiences to get to. A reality DaBini played belrer who sings for these commercial retail clients really sounds genuine!'Pharrell's who sung in the documentaries short 'Daddy Whis'.He commented:'With such a flair for singing ten songs, few singers can have their own records heardand enjoy these songs singing not only over the tapes.'It was inducted furtherDouglas Lancaster, head of the movie ensembleInterestingly, it is revealing that Celebrity Apprentice host Donald Trump told Ball magazine that he doesn't show the characters for the kids aspect of their lives.It's interesting that fans of whether who sees similarities between the movie resorts and the spots themselves are rare.<|endoftext|>This is former Oklahoma Attorney General Swatsina Newman, NAMATP/MCP.\n\nThe 1989 Draft of Hyoichi Young. A member of Wyoming JFF, he was selected 3. Second victim of the mad panda of teams, he was reunited with USHL and provided a PowerPoint presentation featuring 3 more players on the women's game up to the draft.\n\nNobody I knew ever discussed this draft with the world known datalers and it's assumed that he was etcetera and played the game. During the trips I lost playing Young in our office and brought in USHL: Hockey East hockey (CT) and California State League Hockey League (AAA).\n theShowled a fascinating story, analysis of field hockey coverage on USHL since 1987 and told stories as well. To the organization he brings the family debriss from this time. The legacy of his fateful chug cloud; maybe even the future.\n\nWill this draft and a 2016 one come together for any part? I believe so to begin with. There is no way to answer. Questions What the Waitfield Principle and other questions questions which has entered the generated prize can contact me1@gmail.com Staff I also welcome your comments, questions, and let me know.\n\nThe Perspective's Fans or the Waitfield conclusion\n\nIn recent years, Larsen Soll I've jointly came to agreement with the not so named theme for this 3 The Midterm of the study of Social, Policy, and Economics, a paper which he demonstrated:\n\nIt will be difficult for the next several focus topics for the people you will think most to be closely associated with. That thesis was that although there is a demand for feedback, the \"should\" and thus a \" conclusion\" comes from the representative and quasi industry (the CEH then pointed out). Feedback voices are needed only, also the above, to address scholarship and research.\n\nMany of the decisions that we are going to make based on discussion are complicated, three-way decision-forerunners.\n\nThe fact of the two separate winner-to-take loyalties that decide only the golden rules rulemaking bubble are the significance of the issue. Understanding the 'Three First' Problem is\n\nthe best answer. However discussing some requires the narrowest set of decisions.\n\nSteps unlikely, by any definition.\n\nAny winner requires a possibly existing principle such as Rule.\n\nHowever, there is good secrecy.\n\nThe fact is that the majority of the step-by-step discussion is about settled policies and logic. There are huge errors in the middle of the discussion<|endoftext|>
\end{wrappedlisting}

\paragraph{Generation example at 16 sampling steps.}

\begin{wrappedlisting}
<|endoftext|> edited list, I know better than anyone that I\u2019m not going to teach the history class about the realities of the economy and modern business, what we will rather talk about at the first lesson history is be the question of what the road cuts are supposed to do today and about; why we made today\ufffd\u2019s sound model the way it\u2019s built to be; and what our history is about, how we designed the economy and therefore for a country to work and flourish right down.<|endoftext|>Bill: When they walks home behind Etgel \u2013 why he looked at Leah\u2019s-Ghostening cat population? Robot life was saved by theOfSecure? to Mike Davidson 0hLO\n\nGuard or Musk calculated that life has been saved by'secure' on Airbnb after servers have been stepped down.\n\nAirbnb was chastised by Stratus and Google for licensing just from its home world, wondering if things matter and people would rescue it.\n\nNow it\u2019s down to unofficial IoTs-iusmaires as robots check out for vulnerable users and pick their mission's stars and a fleet of their supersonic pros and cons.\n\nTo restore security for this idiotic ecosystem from malware, prevent your exit keys from stealing your eLife sharing, the tech editor Musk tweeted last night.\n\n\"Today on May 31, 2019 the new era of sustainable charmers simply sank to the ground.\" He also responding to the user Duncan reported.\n\nTailingClean in a Blitz: Robots Restore Security of Families\n\nHowever, the manned car security expert believes the deterrents are end\n\ndoe by now.\n\nWhether or not the robots come back are new matters, Musk added today.\n\nNow rather than not, he will ditch the stigmas: Stations.\n\nOver to Martha of Our Fortan:\n\nOren employee Hugh Hallow, Roz and partner Kelly have just taken a major stake in an oilfield crowdfunding platform to build a Battalion Care center for Northern Texas.\n\nThe Musk co-founder and Roz lot manufactured money out of their CEO's round-ups, meaning they raised the question of $15 million to pay for their \"An App I Teutary\u201d platform, now to be called NXhoe. Elon Musk's not too much more reasoned.\n\nBecause the startup is located on\n\nEarth.\n\nPerhaps, \"app\" was not tied to Earth thematically in the first place \u2013 it in the digital universe of discovered parties operating with hopes of finding a person acquainted with men.\n\nBut yes, for App. the helped generate many other thirteen authorities in that regard.\n\nMeanwhile, the app-based Games developer's \"You think is significant\" app hopes to better understand and execute platform efforts potentially to make more people queasy. Franco Lee of the current documentary Citizens for Sociocracy: The Public Record, said:\n\nOver 300 million in 2011 users clicked on its seven platforms \u2013 and 400,000 social media websites \u2013 means \u201ca platform I campaigned on\ufffd, and it raised $1.5M or more to relied on the social hand and knowledge, and, of course, how to \u201ckeep things interesting\u201d\n\nApps, Hallow\u2019s co-founder said, are reagents in the internet that's done no the way he's cousins have taken crude, road time shortcuts. He added:\n\nWith the most of all, everything is handled divulged and registered especially on social media. Does he accountable for actions and are available to post/repost, will your failures destroy the reputation? To see this, hire a willing and trustworthy partner \u2013 \u2018It\u2019 as though the symbiotic model used by Reno & Nevis for both the videos and high school and court to ensure safer and more sensual uses but engaged in the maximizing and supremacy of work ethic.<|endoftext|>Burst entered to win its 150th consecutive race [1]\n\nPraft went on to win 157 of the inaugural knockout-winning 5-0 nations1, and was never unbeaten in singles but first placed class in the 1990s Barry Vale, and Jon, the second, claimed a third place [Most judges would never have known of the fact that this debut class qualified the third World School Championships that season; the first time since Courtney Deletes World Cup in 2005], they formed parties when the other countries competition the World Cup had become. James Jonathan Banks of Pyre who came at the 20 Crickets was one of the most unlikely candidates to make the brackets [The school programme has intermittently heled-grown 40 in the seven year cycle, leaving just the five vars on a track assmall as five as2].\n\nAnother bold entry was in the south, Auckland winner Michael Baker entered on the seven year cycle. As a continuously world class competitor, he has brought vital minds to the high courts whilst delivering a the deep offshore nation the Rugby provides. His return from an amateur break is ab<|endoftext|>
\end{wrappedlisting}

\subsubsection{Generations from \ours distilled on \corpusA}

\paragraph{Generation example at 64 sampling steps.}
\begin{wrappedlisting}
<|endoftext|> Man, Michael knows the look. As soon as I type around it, really I look at the the look in you eyes. I like to film so on and on. There are shorter shadows on the film and then when I go over the camera it kills him the only young person that I have shot that with. But then I'm sitting in one of the Amets and the second steps, and he steps out of his or three shots a bit, and then he points a gun at me, and then BenLE steps down again, and you never know who is shooting but and he still's up there. I show half the size of my screen to him. BenLE notices this idea, then he's hurt everybody else on the floor of the room. It's just really odd, and it happens a few times, not even in our studio.\n\nIt's weird, and really makes it feel like someone getting the best shots or someone puts them over in the EDM will shoot somebody, like Sharon Frye. I've been watching this at least on the TV, and I was seeing it second and third event time. It was my seat of the train.\n\nMike Paul's other material right now is Breaking Bad.\n\n(Source: Getty Images)\n\nDonald Trump is smart and has set some amazing ideas for the future of baseball next season on the success that he brings. Let\u2019s look at the Trump fantasy.\n\n, being a business-shash politician and raising the Little League\u2019s exciting U.S. ideas that he could win or return to the top of the successful Little League.\n\nThere hadn\u2019t been many good news before being let a sitting president take office as a huge executive from here to Trump.\n\n, a 10% banking and 80% investment banker, was absolutely everything we saw for the Orlando Magic, because he also helped secure the Milwaukee Bucks. All-Star league vet Ryan Smith, the Orlando Magic\u2019s manager, left the team after the trade deadline with no need for Tom Donning. Ryan is a surrogate of the fans, family, friends, and fans that should be unfamiliar with politics. It hadn\u2019t been many good news before winning a single game and being let president-elect take office as a huge executive.\n\nTurns out there is now one of the better fan base for any sports team in baseball. Trump clearly knows what to do. That has been a success. For the gaming movement so enthusiastically thrashed for him. It is that many owners are flying around the United States and cities to value their game.\n\nAnd now now there are endless opportunities. There are a person and a company who might be playing or are now sending illegal assets to Ted Cruz stifle their fortunes or open doors. Jos\u00e9 Castro shouldn\u2019t have a gambling ring. Talbot shouldn\u2019t get Super Bowl and be known to be getting ready for a Cup playoff push.\n\nIf we can have Trump and the ownership, we can get a superstar of that caliber as a major league player and give him a first round. And that is nothing to do with Trump, though, as a team, we can see enough of a lot of talent. Drew, a former running back who has thrown 48 games, is now one of the most powerful players that the Little League has ever watched before. His investors have been willing to pay to deal Joe Cannrell on the roster after 5-in-13 losses. Jim Wong is a shrewd baseball scout. The infrastructure for coaching smartness has been built and has been laid up. GM has an asset-management stake of more than $100 million in just two years, with a maver coming to the negotiating table and paying for that. Boston is an organization he needs. A different guy from the old Toronto Toronto. He has been keeping quite close with the club\u2019s incarnations at the time. He was involved in office once again with Mike and Dick Tracy. A part of life. Bryan, big-day work with Matt Kemp like Louis Bette train in the yester baseball movement in his neighborhood. He is ready to take his coaching away and bring in Chris and Davey.\n\nWe can\u2019t find even close to playing out.\n\nTim Juron is the official manager in the division. He left Las Vegas in late\u2019s as manager of Little League Baseball before joining Jeff Landis, the former team\u2019s then-GM with the Cleveland Browns-97. He played overseas for Dallas and became a spotman co-host in the 1993-1996 Stars.\n\nNo guy other than Ryan is able to Tyler Dan Marino. Ryan is going to take over the franchise from Arron McGuire, Alfonso, and Vincent Cox. He could also bring back a long-time NFL star.\n\nShannon graduated from Top 12 Seasonal-owned University of Tennessee and slipped to second pick in 2014. While there, Golden State<|endoftext|>
\end{wrappedlisting}

\paragraph{Generation example at 32 sampling steps.}
\begin{wrappedlisting}
<|endoftext|> at about 200 feet, or almost two miles.\n\nThe company\u2019s superintendent, Michael Accel, said it had been beneficial to the structure for the Packers, but that the fire \u201cshould not have taken a bit in danger.\u201d\n\nIn hours of time before the June 2016 agreement was reached, neither side nor Akug Lee had any fears with the building plan, a representative on the matter could said in details. He told both parties of the decision, but they are refusing to do so. Shannon Lock, who worked with the guards as president of the Flyer Millet Expressaway company a week earlier, saw about least 8 staffers in the building.\n\n\u201cShe and the place said employees are out at comparable sites,\" Lock said Monday, specifically for her personal age range. \"We don\u2019t think it possible to involve any operations, I did. We\u2019ve said the specific gates, some people know exactly what is going on, what\u2019s happening, and it\u2019s true that a lot of goes on, and overhead lighting and other stuff, a lot. But I just think about acknowledging this is an honor, and to me that is the nature of a community putting graphics and place on the floor.\"\n\nIt\u2019s just year, but some folks chose Dec. and people find it a challenging time\u2026 people, who are natives of the site, said Deanne McGarukum\u2019s father, Gene Metz.\n\nHe said they designed and built it at the University of Virginia, recited it, and after some academic somesurgeon deep digging, the first step was the old antiquities.\n\n\u201cWe had a lot of half-finished dealing here,\u201d he said. \u201cWe went there for our residency and torn all the ancient and gravestones, using them as a practice slate.\u201d\n\nStill, McGarukum said it was a problem. A lot of the guys were very socially awkward and very socially society, so we were totally overwhelmed and we sort of started buying and buying stuff.\u201d\n\nWoodes said they had to climb out on their own with the floor-strewn floor. He was opposed to the Wall St wall for several years.\n\n\u201cThe main problem in Yarvi, George, is to solve,\u201d Montee Cromwell said. \u201cWe was considered the best solution in the 90\u2019s \ufffd-82, because there were people out here. Why it it\u2019s a dream endeavor,.\n\n\"When it came to us and they started moving in, co everybody was New Jersey, digging everywhere foot by foot into the River against the big problem. It all couldn\u2019t happen! I understand why it did and how it ended up getting there.\"\n\nThe Flyer Studio owner Jeff Ehbrist, 58, said he\u2019 was visiting a construction company went into the grasses near a bank to move in a building on the bank. \"Maybe I could be out,\" he said. \"I was working really hard and this thing happened here, which was just so hard to keep going like here.\"\n\nLatterstrom Street is Historic belong to the Planet Store, which has a 26-acre feet of floor that flows from several sizeable houses throughout the city. Quarterback Mike Lee said it\u2019s not free space, but he has volunteers all over the world, who have been at Planet Store every four years, making agreements to employees. There\u2019s a decent deal of eight other kinds of big offices they can work on.\n\n\"The main disconnect is that it\u2019s not the idea that looks like the creation of a college were-in, but people working there in a building that has so cleaned it up on a leaking floor and it\u2019s so deep people don\u2019t understand the reason for the use.\"\n\nThe group worked harder to move the space and business into the library, but Stone Long Stine, which just halted fundraising and competing with the city of Dallas and Philadelphia put up extremely higher rents. Rumpbridge said the workers have been here for the last eight years to keep the flat and business open.\n\nAkslyn, Ind.\n\nThe Planet Store is the Church\u2019s fifth-largest site in Indian City. The topography of the building of the Library marked about 1:30 p.m. in mega-housing 2009, as compiled by Census Bureau, and February 25. The primary building to the PlanetStore is the Remnick Building, demolished in 1937 and March 7. The building also now houses rooms, according to CensusC and the Chrysler Building was demolished in big time in 2005.\n\nThe Church of Indian City is an exclusive, sound-tight home maintained by artists from The Flyer (Yber Spring Architecting Group, makers of Old Style Graffiti, and the Windspeaker Artists Companies, Corp.<|endoftext|>
\end{wrappedlisting}

\paragraph{Generation example at 16 sampling steps.}
\enlargethispage{\baselineskip}
\begin{wrappedlisting}
<|endoftext|>ryker\u2019s upcoming Gear VR Game One, which was developed at the PC level, and isn\u2019t likely to be added to with The Gear The Dance of Isrotur right now but to get the game in the studio is, until San Francisco, on The Next VR expansion and Oculus Rift.\n\nWe Web Developers will continue to build a game unlock store, which is several locations in the UK and UK, and Disney releasing Overwatch in South America. This December showcases \u201c Sigourney Weaver\u201d\u2019s soon video debut, and the story of the game warts here in the entire world.\n\nWater City At River City really hasn\u2019t be north of Melbourne for yet, and this is not the time to become the right location. To add to more Fehrally momentum, this is what we\u2019re doing. The company is getting some capital off, and it will be throwing its bid to prize 10 to 40 players for up $1.\n\n- A world filled with monsters that are interconnected, and there have been a number of established monsters, but they will always come up with the impacts of people growing together. This game will bring the magic of the Loop with Super Meat Boy, too.\n\n Episode IV: Infinity Mirror\n\nSocial- Dual street character mode has a powerful effect at River City, including some \u201cnepor\u201d types and detail. To push your traditional content, you can interact with a number of highly detailed events and in the middle of the story, In and immersive, the takes a red card game to see through the battlefield more than gameplay or transcendent details. This experience, for surrealism and graphics behind it, will already become very popular.\n\nbs.games.\n\nA database of the 40 artists who created the game and is collected up to $200, and $320 will be paid by game company Calcalibur City, on the site of \u201c Epic\u201d as on display at Ubisoft San Francisco and will celebrate the game and strive to represent the world with the best in Spaziland.\n\n- The fledgling PCelles would reenact themselves with the AI control puzzle. The game is designed and will end up at $25,000.\n\nThe awarding is now on March 20 for Energy: Envergence.\n\n- Tinker Board, a story and action-packed board game created to the Marvel, the Mark of Encounter, Mastermaster, Hashoric, Legacy, Numeneric and Runes and Chronicles franchise. Ben Ryan will present the designer of the products he made in his first year to Agnesne Gembert, Bible\u2019s Witch of the Year. Ryan can get some special T together for his game that he found so fun to and was thrilled he wants to play this little one.\n\nAnd how the game ends up, as he, his brother, out: \"Canada in the Namps of Madness: Zack\u2019s, game player, have I been overlooked or have been pushed hard by a misconception?\n\n\"Look, the guy is Matt Corcoran and maybe he who wins. He you\u2019d like to receive an awesome shopping future day by meeting Sam, John,zellingi and Mama. Thank you for the many iPhone subscribers.\"\n\nFriday 21, 29: Phil Tayang and his Trejo (Bacon, who previously hasn\u2019t participated in the games, and associate Producer Paul Spragen.\n\nJason, Jan and Jason slipped after being very close at the board with Garrett.\n\nMay 28, 13: From now-9, Danielle, revealed how she was trying to speed up her game video Mark Maiss.\n\nDamia Becker, who was still out of high school, 24, who talks about her passion for game stunts, has at least one game of hers.\n\nRob and Dustin Collins, who was starting how to work around as a senior programmer, told BuzzFeed News that Aaronione was more at odds by\ufffdthe tofuddledness\u201d than working on measure unknown.\n\nHana Sculptee\n\nAlright. And then this is a quote for all Fehusally wrapped fans:<|endoftext|>NOTE: The most competitive staff members with highly mercurile mastery and the gift of emotion, tend less to mean to name it but than those who use conventional history to testify that their test are over problems.\n\nI\u2019m deeply moved by the fact that things get in our heads for life. I try my best not to take notices toward all of us being and when they get so media sabotaged that it\u2019s makes me worse, every day more. I try to grow stronger.\n\nIt is alongside to get a chance just to falsify the facts but to movet people from one over to another for to delegitections better.\n\nFor Hubris, it goes to my heart as I once saw Peter and Caroline homesmaneuver at my own<|endoftext|>
\end{wrappedlisting}

\end{document}